\documentclass[11pt]{article}

\usepackage[colorlinks, linkcolor=blue, citecolor=blue]{hyperref}

\usepackage[utf8]{inputenc} 
\usepackage[T1]{fontenc}    
\usepackage{hyperref}       
\usepackage{url}            
\usepackage{booktabs}       
\usepackage{amsfonts}       
\usepackage{nicefrac}       
\usepackage{microtype}      
\usepackage{xcolor}         

\usepackage{mathrsfs}
\usepackage{amsfonts,amsmath,dsfont,amssymb,amsthm,stmaryrd,bbm}
\usepackage{hyperref}

\usepackage{algorithm}
\usepackage{url}  
\usepackage{verbatim}
\usepackage{color}
\usepackage{graphicx}  
\usepackage{algorithm}
\usepackage{algorithmic}
\newsavebox{\algbox}
\usepackage{mathtools}
\usepackage{caption}
\usepackage{subcaption}
\usepackage{scalerel,stackengine}
\stackMath
\newcommand\reallywidehat[1]{%
\savestack{\tmpbox}{\stretchto{%
  \scaleto{%
    \scalerel*[\widthof{\ensuremath{#1}}]{\kern-.6pt\bigwedge\kern-.6pt}%
    {\rule[-\textheight/2]{1ex}{\textheight}}
  }{\textheight}%
}{0.5ex}}%
\stackon[1pt]{#1}{\tmpbox}%
}

\newcommand {\mub} {\boldsymbol{\mu}}
\newcommand {\xb} {\boldsymbol{x}}
\newcommand{\f}[1]{\boldsymbol{#1}}

\newcommand{\ca}[1]{\mathcal{#1}}

\newcommand{\s}[1]{\mathsf{#1}}

\newcommand{\bb}[1]{\mathbb{#1}}

\newcommand{\p}[1]{\left(#1\right)}
\newcommand{\pp}[1]{\left[#1\right]}

\newcommand{\norm}[1]{\left\|#1\right\|}
\newcommand{\abs}[1]{\left|#1\right|}
\newcommand{\remove}[1]{}

\newcommand{\calA}{{\cal A}}
\newcommand{\calB}{{\cal B}}
\newcommand{\calC}{{\cal C}}

\newcommand{\calO}{{\cal O}}
\newcommand{\calP}{{\cal P}}
\newcommand{\calQ}{{\cal Q}}

\newcommand{\calS}{{\cal S}}
\newcommand{\calT}{{\cal T}}

\newcommand{\calV}{{\cal V}}

\newcommand{\calX}{{\cal X}}
\newcommand{\calY}{{\cal Y}}

\def\bmu{{\mbox{\boldmath $\mu$}}}

\def\bmu{{\mbox{\boldmath $\mu$}}}

\newcommand{\be}{\begin{equation}}
\newcommand{\ee}{\end{equation}}
\newcommand{\beqna}{\begin{eqnarray}}
\newcommand{\eeqna}{\end{eqnarray}}

\newtheorem{definition}{Definition}

\newtheorem{theorem}{Theorem}
\newtheorem{corollary}[theorem]{Corollary}

\newtheorem{lemma}{Lemma}

 \usepackage{float}

\usepackage{multicol}
\usepackage{wrapfig}
\usepackage{appendix}

\let\norm\undefined 
\DeclarePairedDelimiter\norm{\lVert}{\rVert}
\title{Fuzzy Clustering with Similarity Queries}

\author{Wasim~Huleihel\thanks{W. Huleihel is with the Department of Electrical Engineering-Systems at Tel Aviv university, {T}el {A}viv 6997801, Israel (e-mail:  \texttt{wasimh@tauex.tau.ac.il}).} ~~~Arya~Mazumdar\thanks{A.Mazumdar is with the Halıcıoğlu Data Science Institute at University of California, San Diego, USA (email: \texttt{arya@ucsd.edu)}} ~~~Soumyabrata~Pal\thanks{S. Pal is with the Computer Science Department at the University of Massachusetts Amherst, Amherst, MA 01003, USA (email: \texttt{spal@cs.umass.edu}).}}

\begin{document}

\maketitle

\begin{abstract}

The fuzzy or soft $k$-means objective is a popular generalization of the well-known $k$-means problem, extending the clustering capability of the $k$-means to datasets that are uncertain, vague and otherwise hard to cluster. In this paper, we propose a semi-supervised active clustering framework, where the learner is allowed to interact with an oracle (domain expert), asking for the similarity between a certain set of chosen items. We study the query and computational complexities of clustering in this framework. We prove that having a few of such similarity queries enables one to get a polynomial-time approximation algorithm to an otherwise conjecturally NP-hard problem. In particular, we provide algorithms for fuzzy clustering in this setting that ask $O(\mathsf{poly}(k)\log n)$ similarity queries and run with polynomial-time-complexity, where $n$ is the number of items. The fuzzy $k$-means objective is nonconvex, with $k$-means as a special case, and is equivalent to some other generic nonconvex problem such as non-negative matrix factorization. The ubiquitous Lloyd-type algorithms (or alternating minimization algorithm) can get stuck at a local minima. Our results show that by making few similarity queries, the problem becomes easier to solve. Finally, we test our algorithms over real-world datasets, showing their effectiveness in real-world applications. 
\end{abstract}

\section{Introduction}
The $k$-means objective for clustering is perhaps the most ubiquitous of all unsupervised learning paradigms. It is extremely well studied, with Lloyd's algorithm being the most popular solution method, while the problem in general being NP Hard~\cite{blum2020foundations,berkhin2006survey}. A variety of approximate solutions for $k$-means exists~(such as, \cite{de2003approximation,ArthurV07}). 

In recent times, there have been efforts to bring in a flavor of active learning in $k$-means clustering; by allowing the learner to make  a limited number of carefully chosen label/membership queries~\cite{Ashtiani16,ailon2018approximate,chien2018query}. In this setting, the main objective is to show that a polynomial-time solution exists provided that small number of such queries to an oracle is permitted. Note that, an alternative to membership query (i.e., query of the form ``does the $i^{\s{th}}$ element belong to the $j^{\s{th}}$ cluster'') is the same-cluster/similarity query (i.e., ``do elements $i$ and $j$ belong to the same cluster''). Given representatives from each cluster is available, one can simulate any membership query with at most $k$ same cluster queries. Hence whatever is achievable with membership queries, can be also achieved by same cluster queries by asking at most $k$ times as many queries~\cite{Ashtiani16}. 

How realistic is to allow the learner to make such label queries, and how realistic is the oracle assumption? It turns out that many clustering tasks, primarily entity-resolution, are being delegated to crowdsourcing models. One can easily assume that the crowd is playing the role of an oracle here.
It is natural in these models to assign crowd workers with sub-tasks, which can either be answering the membership query or the same cluster query~\cite{balcan2008clustering,vesdapunt2014crowdsourcing,wang2012crowder,mazumdar2017theoretical,firmani2018robust}.

Fuzzy $k$-means or soft $k$-means is a generalized model of the $k$-means objective that covers lot more grounds being applicable to dataset where datapoints show affinity to multiple labels, the clustering criteria are vague and data features are unavailable~\cite{bezdek2013pattern}. In fuzzy clustering, instead of an element being part of one fixed cluster, it is assumed that it has a membership value between $0$ and $1$ to each of the clusters. The objective of the clustering is to recover all of the membership values for each element or datapoint. This relaxation in membership values from $\{0,1\}$ to $[0,1]$ makes the fuzzy $k$-means an attractive model for overlapping clusters~\cite{zhang2007identification}, multilabel classification~\cite{lin2010mr} and numerous other applications such as, pattern recognition~\cite{bezdek2013pattern}, computational biology~\cite{ben1999clustering}, bioinformatics~\cite{valafar2002pattern}, image segmentation~\cite{ahmed2002modified,chuang2006fuzzy} and so on.

What is more interesting is that solving fuzzy clustering is equivalent to approximate solution of  nonnegative symmetric matrix factorization problems~\cite{ding2005equivalence}. Given such applications, it is somewhat surprising that the theoretical results known about fuzzy clustering is quite limited. As is even the computational tractability of fuzzy $k$-means is only conjecturally known. In fact algorithmic study of fuzzy $k$-means algorithm seems to have been only recently started~\cite{blomer2016theoretical,Blomer18,QianLiu20}.

\subsection{Membership vs. Similarity Queries} As mentioned above, in the hard-clustering case, one can simulate any membership query with at most $k$ same-cluster (or, similarity) queries \cite{Ashtiani16}. In fuzzy clustering, we are given a set of $n$ vectors (items) in the $d$-dimensional Euclidean space, and we are tasked the problem of determining the $k$-dimensional {\em membership vector} for each of the items. The clustering algorithm has access to the set of vectors, and
can also make queries about the underlying soft clustering. As in hard-clustering, these queries can take at least two forms. Specifically, membership queries refer to the existence of a \emph{membership-oracle} which whenever queried, replies with the membership value of the $i^{\rm th}$ item to the $j^{\rm th}$ cluster, i.e., the $j^{\rm th}$ entry of the $i^{\rm th}$ membership vector. One could argue that querying the membership values, and specifically, the existence of such a membership-oracle, is not realistic, and often impractical in real-world settings since it requires knowledge of the relevant clusters. Instead, the similarity query model that takes two or three elements as input and ask \emph{``How similar are all these elements?"} is much easier to implement in practice. Accordingly, in a similar vein to hard-clustering,
we show that membership queries can be simulated by similarity queries, measured by the inner-product between pairs (and triplets) of membership vectors. Finally, as will be explained in more detail, the responses of the oracle or the target solution do not need to correspond to the optimum solution of the fuzzy $k$-means objective, but instead need to satisfy a few mild conditions. Accordingly, the target solution can represent noisy/quantized version of the optimum solution of the fuzzy $k$-means objective function making our set-up highly practical in crowdsourcing settings.

A natural question that one can ask about fuzzy clustering following the lead of \cite{Ashtiani16} is how many membership/similarity queries one may ask so that the clustering can be efficiently solved with a good approximation of the target solution? In this paper, our main motivation is to answer this rather basic question. It has to be noted that revealing a fractions of labels during clustering algorithm in fuzzy $k$-means has been studied with some heuristics in~\cite{pedrycz1997fuzzy}, but not only the model was different from ours, a rigorous theoretical study was also lacking in the past literature.

\subsection{Our Contributions and Techniques} 
We design several algorithms and analyze their query and computational complexities. Specifically, the first algorithm we propose works in two phases: in the first phase, we estimate the center of each cluster using \emph{non-adaptive} queries, and in the second phase we recover the membership of each element to each cluster using the estimated centers obtained in the first phase. We show with this algorithm it is possible to get an approximation of the membership vectors by making only $O({\rm poly}(k/\beta)\log n)$ queries, where $\beta$ is ratio of the volume of the smallest cluster to the average volume, informally speaking. Then, we show that by allowing one to ask \emph{adaptive/sequential} queries in the first phase, the polynomial dependence on $\beta$ of the query complexity can be further improved. Despite the strong dependency of the query complexity on $\beta$, for a large regime of parameters, our query complexity is non-trivial and sublinear in $n$. Interestingly, for the special case of $k=2$, we can completely get rid of the dependency of the query complexity on $\beta$, and it suffices to ask $O(\log^2 n)$ queries only.

For our theoretical results to hold, we make an assumption on the dataset that quantifies the clusterability of the data. We would like to emphasize that our assumption is only needed for deriving the theoretical guarantees; our algorithms are demonstrated to perform well on real-world datasets irrespective of whether our assumption holds or not. The computational complexity of all the algorithms scales as $d$ times the number of queries. 

It is to be noted that the analysis of the aforementioned algorithms are technically much more demanding than solving the analogous query complexity problem for hard $k$-means. For the case of hard $k$-means, a general principle would be to sample a small set of items and estimate the centre of the largest cluster from the samples (see, \cite{Ashtiani16}). Subsequently all the items closest to the estimated centre can be assigned to it and the cluster can be removed from consideration. One can iterate over this algorithm to find all the clusters. However the size of a cluster is not well-defined in the fuzzy setting; and also it is not possible to assign items to a cluster and remove them from consideration, since they can very well have nonzero memberships to other clusters. 

To tackle these difficulties, we have to use more challenging technical tools and analysis methods than the hard $k$-means literature. In a nutshell, since at any step of the algorithm we cannot remove items from consideration, we propose sophisticated random sub-sampling algorithms that explore the space of items efficiently. Specifically, having reliable approximations for the centers and memberships of some processed clusters, we partition the space of items into bins that have the same approximated sum of membership weights to unprocessed clusters. Sampling equally from every bin ensures that we obtain enough samples from items which have high membership weight to those clusters which have not yet been approximated. Also, contrary to hard $k$-means where a simple (unbiased) empirical mean estimator is used to approximate the means, in the fuzzy setting, we end up with a form of self-normalized importance sampling estimators. While these estimators are biased, we managed to bound their deviations in an elegant way using concentration inequalities. We tested our algorithms over both synthetic and real-world datasets, and illustrate the trade off and effect of different values of $\beta$ and $k$ on the estimation error and query complexity.

\subsection{Related Work} Clustering with limited label-queries have mostly followed two different lines of works. In the first line of work, queries were used to make the objective based clustering, such as $k$-means, polynomial time solvable~\cite{Ashtiani16,ailon2018approximate,chien2018query,bressan2020exact,bressan2021exact,bressan2021margin}. 
The problem being considered here is most closely related to this line of work, albeit, instead of hard $k$-means, we focus on a soft objective.

In the second line of work, instead of an objective based clustering, it is assumed that there exists a ground-truth clustering, and by label-querying, one gets a (noisy) answer from an oracle about this ground-truth. The aim of this line of work has been to give statistical-computational guarantees on recovery of the ground truth~\cite{mazumdar2017clustering,mazumdar2017query,mazumdar2017semisupervised,ailon2018approximate,bressan2019correlation,tsourakakis2017predicting,saha2019correlation}. The work in this canon that is closest to our problem is \cite{huleihel2019same}, where an overlapping clustering groundtruth was considered. At a high level, we also consider overlapping clusters, however, focus on objective-based clustering therein.

It is worth noting that clustering with same cluster, or other label queries, has found connections with and application to correlation clustering~\cite{bressan2019correlation,saha2019correlation}, community detection~\cite{mazumdar2017query,chien2018community}, heavy-hitters~\cite{sarmasarkar2020query} and possibly other areas, that we might have missed. Our problem can be seen as approximately equivalent to non-negative symmetric matrix factorization \cite{Paatero94} with limited access to the entries of the factors (via queries) \cite{ding2005equivalence}.

\subsection{Organization} The rest of the paper is organized as follows. In Section~\ref{sec:model} we present our model and the learning problem. Then, in Section~\ref{sec:main_results} we present our algorithms and their theoretical guarantees. The main theoretical results can be found in Theorems~\ref{thm:1},\ref{thm:2},\ref{thm:3state}, followed by several experimental results over both synthetic and real-world datasets. Finally, our conclusion an outlook appear in Section~\ref{sec:conc}.  Some technical proofs and discussions are relegated to the appendices.

\section{Model Formulation and Background}\label{sec:model}

\subsection{The Optimization Problem}\label{subsec:optimiz}

We start by describing the hard $k$-means problem, and then move forward to the fuzzy $k$-means problem (soft clustering). Let $\calX\subset\mathbb{R}^d$ be a set of $d$-dimensional points, with $|\calX|=n$. We denote this set of vectors by $\{\xb_i\}_{i=1}^n$, and assume without loss of generality that $\xb_i\in\calB(\f{0},\s{R})$, for any $i\in[n]$, where $\calB(\f{a},\s{R})$ is the $L_2$-Euclidean ball of radius $\s{R}\in\mathbb{R}_+$ centered around $\f{a}\in\mathbb{R}^d$. Let $\calC_{\calX}\triangleq\{\calC_1,\calC_2,\ldots,\calC_k\}$ be a clustering (or, partition) of $\calX$, where $k\in\mathbb{N}$ is the number of clusters. We say that a clustering $\calC_{\calX}$ is center-based if there exists a set of centers $\mub\triangleq\{\mub_1,\mub_2,\ldots,\mub_k\}\subset\mathbb{R}^d$ such that the clustering is induced by the Voronoi diagram over those center points. The objective function of hard $k$-means is
\begin{align}
    \mathsf{J_{km}}(\calX,\bmu,\s{U}) = \sum_{i=1}^n\sum_{j=1}^k\s{U}_{ij}\norm{\xb_i-\mub_j}_2^2,\label{eqn:hard}
\end{align}
where $\s{U}\in\{0,1\}^{n\times k}$ is the partition matrix, i.e., the membership weight $\s{U}_{ij}$ is 1 if $\xb_i\in\calC_j$, and 0 otherwise. In hard clustering each data point is assigned exactly to one cluster and every cluster must contain at least one data point. The goal of $k$-means is to minimize \eqref{eqn:hard} w.r.t. $(\bmu,\s{U})$.

In soft-clustering there is no partitioning of $\calX$. Instead, we describe the $k^{\s{th}}$ cluster of a fuzzy clustering as a vector of the fractions of points assigned to it by the membership function. In particular, the memberships weights take on real-valued numbers in $[0,1]$ rather than $\{0,1\}$, as for hard clustering. The fuzzy $k$-means problem is defined as follows.
\begin{definition}[Fuzzy $k$-means]\label{def:fuzzy_prim}
Let $\alpha\geq1$ and $k\in\mathbb{N}$. The fuzzy $k$-means problem is to find a set of means $\bmu = \{\bmu_\ell\}_{\ell\in[k]}\subset\mathbb{R}^d$ and a membership matrix $\mathsf{U}\in[0,1]^{n\times k}$ 
minimizing
\begin{align}
\mathsf{J_{fm}}(\calX,\bmu,\s{U})\triangleq\sum_{i=1}^n\sum_{j=1}^k\mathsf{U}^\alpha_{ij}\norm{\xb_i-\bmu_j}^2_2,\label{eqn:FuzzyObjective}
\end{align}
subject to 
$\sum_{j=1}^k\s{U}_{ij}=1$ and $0<\sum_{i=1}^n\s{U}_{ij}<n$, for all $i\in[n]$ and $j\in[k]$.
\end{definition}
The parameter $\alpha$ is called the \emph{fuzzifier}, and is
not subject to optimization. It can be shown that when $\alpha\to1$, the fuzzy $k$-means solution coincides with that of the hard $k$-means problem.
Similarly to the classical $k$-means problem, it is easy to optimize the means or memberships of fuzzy $k$-means, assuming the other part of the solution is fixed. Specifically, given a set of means $\bmu$, it can be shown that the optimal membership weights minimizing the cost in \eqref{eqn:FuzzyObjective} are (see, e.g., \cite[Theorem 11.1]{bezdek2013pattern})
\begin{align}
    \s{U}_{ij} = \pp{\sum_{\ell=1}^k\p{\frac{\norm{\xb_i-\mub_j}}{\norm{\xb_i-\mub_\ell}}}^{2/(\alpha-1)}}^{-1},\label{eqn:U_solve}
\end{align}
for all $i\in[n],\;j\in[k]$. On the other hand, given a set of membership weights $\s{U}$, the optimal centers are given by
\begin{align}
\mub_j = \frac{\sum_{i=1}^n\s{U}_{ij}^\alpha\xb_i}{\sum_{i=1}^n\s{U}_{ij}^\alpha},\label{eqn:mu_solve}
\end{align}
for all $i\in[n],\;j\in[k]$. Iterating these solutions is known as the Lloyd's algorithm \cite{Lloyd82}, which provably lands on a local optima. It is well-known, however, that finding the optimal solution (i.e., global minimia) to the $k$-means problem is is NP-hard in general \cite{Mahajan09,Vattani_thehardness}. While the same result is unknown for the fuzzy formulation, it is strongly believed to be the case (e.g., \cite{blomer2016theoretical,Kathrin17}).

Finally, although several validity indexes for the performance of fuzzy $k$-means solutions have been proposed in the literature (see, e.g., \cite{Fukuyama1989ANM,Gath89}), the Xie–Beni \cite{Xie91} is the most widely used in the literature. This measure is defined as follows\footnote{In \eqref{eqn:XBmeasure} we divide by $nk$ and not just by $n$, as in \cite{Xie91}, where $k$ was consider a fixed parameter. If $k$ is allowed to grow (e.g., with $n$), then our definition makes more sense, as in general the numerator in \eqref{eqn:XBmeasure} can grow with $k$.}
\begin{align}
    \mathsf{XB}(\calX,\bmu,\s{U}) \triangleq \frac{\mathsf{J_{fm}}(\calX,\bmu,\s{U})}{nk\cdot\min_{i\neq j}\norm{\bmu_i-\bmu_j}_2^2}.\label{eqn:XBmeasure}
\end{align}
For the rest of this paper, we denote a clustering by $\calP\triangleq(\f{\mu},\s{U})$, and we replace $\mathsf{XB}(\calX,\bmu,\s{U})$ by $\mathsf{XB}(\calX,\calP)$. Accordingly, the optimal solution to the fuzzy $k$-means problem (Defintion \ref{def:fuzzy_prim}) is denoted by $\calP^\star$. 
Finally, we say that a clustering $\calP$ is \emph{consistent center-based} if:
\begin{itemize}
    \item Given membership weights $\s{U}$, the centers $\f{\mu}$ are given by \eqref{eqn:mu_solve}.
    \item The membership weights $\s{U}$ are monotone, i.e., if $\norm{\f{x}_i-\f{\mu}_j}_2\leq\norm{\f{x}_\ell-\f{\mu}_j}_2$ then $\s{U}_{ij}\geq\s{U}_{\ell j}$, for any $i,\ell\in[n]$ and $j\in[k]$. 
\end{itemize}
The first condition appears also in \cite{Ashtiani16}. The second condition excludes inconsistent membership weights $\s{U}$. In particular, as is shown in \cite{Kathrin17}, and in Lemma~\ref{lem:fuzyy_prop} in the appendix, both conditions are satisfied by the optimal solution $\calP^\star$, as well as by \eqref{eqn:mu_solve}. We would like to clarify that the above assumptions are \emph{only} required for the theoretical analysis; Our proposed algorithms can be implemented regardless of these assumptions. As is shown in Appendix~\ref{app:exper}, in real-world applications, our algorithm works well even if these assumptions do not hold.

\subsection{Domain Expert/Oracle}
Our goal is to construct \emph{efficient} algorithms with the aid of a domain expert who can give an approximate solution to the fuzzy $k$-means problem. Specifically, given a clustering $\calP$, a $\calP$-oracle is a function $\calO_{\calP}$ that answers queries consistent with $\calP$. One can think of such an oracle as a user that has some idea about its desired solution, enough to answer the algorithm's queries. Importantly, note that as long as $\calP$ is a consistent center-based clustering, it  need not be the optimal solution $\calP^\star$ minimizing \eqref{eqn:FuzzyObjective}. The algorithm then tries to recover (or, approximate) $\calP$ by querying the oracle a small number of times. For hard clustering, the same-cluster query oracle model (i.e., ``do elements $i$ and $j$ belong to the same cluster") was considered in \cite{Ashtiani16}. 
Using a certain number of such same-cluster queries and the data points $\calX$, the clustering $\calP$ can be found efficiently \cite{Ashtiani16}, while finding the optimal solution without the aid of an oracle is NP-hard. For the fuzzy formulation, one can think of analogous oracle models, and in this paper, we focus on the following.
\begin{definition}[Membership-Oracle]\label{def:oracle_fuzzy}
A membership query asks the membership weight of an instance $\xb_i$ to a cluster $j$, i.e., $\calO_{\mathrm{fuzzy}}(i,j) = \s{U}_{ij}$. 
\end{definition}
Basically we ask the oracle for the membership weight of $\xb_i$ to cluster $\calC_j$. Using a certain amount of such queries and the data set $\calX$ we would like to approximate the solution for fuzzy $k$-means problem. As discussed in the Introduction, one could argue that querying the membership values, and specifically, the existence of such a membership-oracle, is not realistic, and often impractical in real-world settings since it requires knowledge of the relevant clusters. Instead, the following similarity query models are easy to implement in practice. 
\begin{definition}[Similarity-Oracle]\label{def:oracle_similarity_body}
A fuzzy pairwise similarity query asks the similarity of two distinct instances $\f{x}_p$ and $\f{x}_q$ i.e., $\calO_{\mathrm{sim}}(p,q) = \langle \s{U}_p,\s{U}_q \rangle$. A fuzzy triplet similarity query asks the similarity of three distinct instances $\f{x}_p,\f{x}_q,\f{x}_{r}$, i.e., $\calO_{\mathrm{triplet}}(p,q,r) = \sum_{t \in [k]}\s{U}_{pt}\s{U}_{qt}\s{U}_{rt}$.
\end{definition}
We show in Appendix \ref{app:equivalence2}, however, that each membership query can be simulated with a few similarity queries. From a theoretical perspective, membership queries can be simulated with only fuzzy pairwise similarity queries if there exists many \emph{pure} instances (instances that have a membership weight of \texttt{1} to a particular cluster in the dataset), see Lemma~\ref{lem:pairw} in appendix. Unfortunately this strong condition is necessary as well. But it turns out that by using a few triplet similarity queries at the beginning, we can resolve this issue.
Specifically, we prove the following  result, informally stated, that is interesting in its own right.
\begin{lemma}[\emph{Membership to Similarity Queries Reduction, Informal}]
Suppose there exists an algorithm $\ca{A}_n$ with time complexity $\s{T}_n$ that uses $m$ membership queries to provide some guaranteed solution of the fuzzy $k$-means problem. In that case, under mild conditions on the dataset, the same guarantee can be obtained with $O(k^3)$ fuzzy triplet similarity queries and $O(km)$ fuzzy pairwise similarity queries, with a time complexity of $\s{T}_n+O(mk+k^3)$. 
\end{lemma}
Therefore, we can express everything in terms of the more natural notion of similarity queries and hence, we use membership queries for the rest of the paper to make the representation of our algorithms simpler.

\paragraph{Noisy/Inaccurate Oracle.} Our algorithms work even with inaccuracies in the oracle answer, provided the deviation is small. Under the assumption that such inaccuracies are random, in particular, 
a repetition algorithm can remedy such errors. Specifically, consider the case where $\calO_{\mathrm{fuzzy}}(i,j) = \s{U}_{ij}+\zeta_{ij}$, where $\zeta_{ij}$ is a zero-mean random variable, and assume the errors to be independent, i.e., the answers are collected from independent crowd-workers. Then, in Appendix~\ref{app:NoisyOracle} we show how one can convert any ``successful" solver for the noiseless problem into a solver for the noisy problem; generally speaking, by repeating the noiseless solver for $\s{T}$ consecutive steps using the noisy responses, and then averaging the results to obtain the final estimate. Note that, the technique used to handle the subsequent steps are also applicable to {\em any bounded error}.  For simplicity of exposition and clarity,  we opted to focus on the noiseless setting given in Definition~\ref{def:oracle_fuzzy}.


\subsection{Objective and Learning Problem}
The goal of the algorithm is to find an approximated clustering that is consistent with the answers given to its queries. We have the following definition.
\begin{definition}[Oracle-Based Solver]\label{def:query}
Let $(\calX,\calP)$ be a clustering instance. An algorithm $\calA$ is called a $(\epsilon_1,\epsilon_2,\s{Q})$-solver if it can \emph{approximate} $\calP$ by having access to $\calX$ and making at most $\s{Q}$ membership queries to a $\calP$-oracle. Specifically, for any $\epsilon_1,\epsilon_2>0$, the algorithm outputs $\widehat{\calP}$ such that $\max_{j\in[k]}\norm{\bmu_j-\widehat{\bmu}_j}_2\leq\epsilon_1$ and $\max_{i\in[n],j\in[k]}|\s{U}_{ij}-\widehat{\s{U}}_{ij}|\leq\epsilon_2$. Such an algorithm is a polynomial $(\epsilon_1,\epsilon_2,\s{Q})$-solver if its time-complexity is polynomial in $(n,k,\epsilon_1^{-1},\epsilon_2^{-1})$.
\end{definition}
We would like to emphasize that we seek an approximate solution rather than exact one. It is important to note $\s{Q}=O(nk)$ membership queries suffice to recover the $\calP$-oracle clustering. Also, with unbounded computational resources one can find the optimal clustering $\calP^\star$ without any queries. Nonetheless, in this paper, we would like to have a \emph{sub-linear} dependency of the query complexity on $n$ and a polynomial dependency of the time complexity on $n,k$. Having an algorithm such as the one in Definition~\ref{def:query}, implies the following guarantee, proved in Appendix~\ref{app:XBmeasure}.
\begin{lemma}\label{lem:1}
Let $(\calX,\calP)$ be a clustering instance, and $\calA$ be a $(\epsilon_1,\epsilon_2,\s{Q})$-solver. Then,
\begin{align}
&\abs{\mathsf{XB}(\calX,\calP)-\mathsf{XB}(\calX,\widehat{\calP})}\leq \frac{\mathsf{XB}(\calX,\calP)\cdot O(\epsilon_1)+O(\epsilon_2)}{\min_{i\neq j}\norm{\f{\mu}_i-\f{\mu}_j}_2^2}+o\p{\frac{\epsilon_1^2+\epsilon_2^2}{\min_{i\neq j}\norm{\f{\mu}_i-\f{\mu}_j}_2^2}}.\label{eqn:XBAlg}
\end{align}
\end{lemma}
The lemma above shows that if $\epsilon_1$ and $\epsilon_2$ are small enough, then the $\s{XB}$ measure associated with an  $(\epsilon_1,\epsilon_2,\s{Q})$-solver is ``close" to the $\s{XB}$ measure associated with $\calP$. Note, however, that the r.h.s. of \eqref{eqn:XBAlg} depends inversely on $\min_{i\neq j}\norm{\f{\mu}_i-\f{\mu}_j}_2^2$. Accordingly, if for some clustering instance $(\calX,\calP)$ this quantity is too small compared to the estimation error, then the difference $|\mathsf{XB}(\calX,\calP)-\mathsf{XB}(\calX,\widehat{\calP})|$ is uncontrolled. This is reasonable since when clusters are ``close", the estimation task gets harder. 

Finally, we introduce certain notations that will be needed for the presentation of our main results. Given a set of points $\calX$ and $\f{v}\in \bb{R}^d$, let $\pi_{\f{v}}:[n] \rightarrow [n]$ be the permutation defined on $[n]$ such that
\begin{align}
\norm{\f{v}-\f{x}_{\pi_{\f{v}}(i)}}_2 \le \norm{\f{v}-\f{x}_{\pi_{\f{v}}(j)}}_2,
\end{align}
for all $i<j\in[n]$. To wit, $\pi_{\f{v}}$ is the ranking of the elements in $\ca{X}$ when they are sorted in ascending order w.r.t. to their distance from $\f{v}$. Then, we let $\gamma\in\mathbb{R}_+$ be the maximum positive number such that for all $j\in[k]$,
\begin{align}
\pi_{\widehat{\f{\mu}}} = \pi_{\f{\mu}_j}, \quad\forall \widehat{\f{\mu}} \in \ca{B}(\f{\mu}_j,\gamma).\label{eqn:equiv}
\end{align}
Intuitively, if one has reliable estimates for the centers, then the ordering of $\calX$ w.r.t. these estimated centers is the same as the ordering w.r.t. the target centers. Next, we define the \textit{cluster membership size} associated with $\calP$ by $\sum_{i \in [n]} \s{U}_{ij}$, for all $j\in [k]$, and a parameter $\beta\in\mathbb{R}_+$ by
\begin{align}
\beta\triangleq\min_{j \in [k]}\;\frac{k}{n}\cdot\sum_{i \in [n]} \s{U}_{ij}.\label{eqn:beta}
\end{align} 
Accordingly, in terms of $\beta$, the minimum cluster membership size is at least $\beta n/k$. Notice that $\beta$ must be less than unity since $\sum_{i,j} \s{U}_{ij} = n$, and there must exist at least one cluster whose membership size is at most $n/k$. We mention that $\beta$ was introduced in \cite{chien2018query} as well. 

\section{Algorithms and Theoretical Guarantees}\label{sec:main_results}

In this section we provide our algorithms along with their theoretical guarantees, for the setting described in the previous section. Specifically, we assume that the oracle has a consistent center-based clustering in mind. For this model, we provide two different algorithms; the first is a two-phase algorithm, and the second is a sequential one. Finally, for the special case of $k=2$, we show an algorithm with significantly improved guarantees.

\subsection{Two-Phase Algorithm}

We start by providing a non-formal description of our algorithm, whose pseudocode is given in Algorithms~\ref{algo:estimate_together}--\ref{algo:binary_search}. Specifically, in the first three steps of  Algorithm~\ref{algo:estimate_together} we approximate the centers of all clusters. To that end, we start by sampling at random $m$ elements from $[n]$ with replacement, and denote the obtained samples by $\calS$. The value of $m$ is specified in Theorem~\ref{thm:1} below. Subsequently, we query the membership weights $\s{U}_{ij}$, for all $i\in\calS$ and $j\in [k]$. Using these queries, we approximate the center of every cluster using the estimator defined in Step 3 of Algorithm~\ref{algo:estimate_together}. Note that this estimator is the most natural candidate when one knows the membership weights only for a set of sub-sampled elements. Although this Algorithm is quite intuitive, it suffers from the natural drawback that if the size of a particular cluster is small, many samples are needed for a reliable estimate of its center. 
\setlength{\intextsep}{0.5\baselineskip}
\begin{minipage}{0.49\textwidth}
\begin{algorithm}[H]
\caption{Parallel algorithm for approximating $\calP$. \label{algo:estimate_together}} 
\footnotesize
\begin{algorithmic}[1]
\REQUIRE $\ca{X}$, $\calO_{\s{fuzzy}}$, $\alpha$, $m$, and $\eta$. 
\ENSURE An estimate $\widehat{\calP}$.
\STATE Let $\calS$ be a set of $m$ elements sampled at random from $[n]$ with replacement.
\STATE Query $\ca{O}_{\s{fuzzy}}(i,j)$, $\forall i\in\calS, j\in[k]$.
\STATE Compute $\widehat{\f{\mu}}_j = \frac{\sum_{i \in \ca{S}} \s{U}_{ij}^{\alpha} \f{x}_i}{\sum_{i \in \ca{S}} \s{U}_{ij}^{\alpha}},$ $\;\forall j \in [k]$.
\FOR{$j=1,2,\dots,k$}
\STATE $\{\widehat{\s{U}}_{ij}\}_{i=1}^n\leftarrow$\textsc{Membership($\ca{X}, \widehat{\f{\mu}}_j,\alpha,\eta$)}.
\STATE $\widehat{\s{U}}_{ij}\leftarrow\widehat{\s{U}}_{ij}+\frac{1-\sum_{j=1}^k\widehat{\s{U}}_{ij}}{k},\;\forall i\in[n]$.
\ENDFOR
\end{algorithmic}
\end{algorithm}
\end{minipage}
\hfill
\begin{minipage}{0.46\textwidth}
\begin{algorithm}[H]
\caption{ \textsc{Membership($\ca{X}, \widehat{\f{\mu}}_j, \alpha, \eta$)} \label{algo:membership}} 
\footnotesize
\begin{algorithmic}[1]
\REQUIRE $\calX$, and $\calO_{\s{fuzzy}}$.
\ENSURE An estimate  $\{\widehat{\s{U}}_{ij}\}_{i=1}^n$.
\STATE Find $\pi_{\widehat{\f{\mu}}_j}$ w.r.t. $\calX$.
\FOR{$s=0,1,\dots,1/\eta$}
\STATE Find $\ell_s = \s{argmax}_{i \in [n]} \s{U}_{\pi_{\widehat{\f{\mu}}_j}(i)j} \ge s\eta$ using \textsc{BinarySearch}($\ca{X},\pi_{\widehat{\f{\mu}}_j}, s\eta$).
\ENDFOR
\FOR{$s= 0,1,2,\dots,1/\eta$}
\FOR{$i=\ell_s,\ell_s-1,\dots,,\ell_{(s+1)}+1$}
\STATE  Set $\widehat{\s{U}}_{\pi_{\widehat{\f{\mu}}^j}(i)j} = s\eta$.
\ENDFOR
\ENDFOR
\end{algorithmic}
\end{algorithm}
\end{minipage}

\begin{wrapfigure}{R}{0.43\textwidth}
\begin{minipage}{0.43\textwidth}
\begin{algorithm}[H]
\caption{ \textsc{BinarySearch($\ca{X}, \pi, x$)}
\label{algo:binary_search}} 
\footnotesize
\begin{algorithmic}[1]
\REQUIRE $\calO_{\s{fuzzy}}$.
\STATE Set $\s{low}=1$ and $\s{high}=n$. 
\WHILE{$\s{low} \neq \s{high}$} 
\STATE Set $\s{mid}= \lceil (\s{low}+\s{high})/2 \rceil$.
\STATE Query $\ca{O}_{\s{fuzzy}}(\pi(\s{mid}),j)$.
\IF{$\s{U}_{\pi(\s{mid})j} \ge x$}
\STATE Set $\s{low}=\s{mid}$.
\ELSE
\STATE Set $\s{high} =\s{mid}-1$.
\ENDIF
\ENDWHILE
\end{algorithmic}
\end{algorithm}
\end{minipage}
\end{wrapfigure}

Next, using the approximated centers, we approximate the membership weights of every element up to a precision of $\eta$ by using Algorithm~\ref{algo:membership}. To that end, in the first step of this algorithm, we sort the elements of $\calX$ in an ascending order according to their distance from the estimated center. Accordingly, if the estimation error of the centers $\epsilon$ satisfies $\epsilon\leq\gamma$, then \eqref{eqn:equiv} implies that the obtained ordering of elements after sorting from the estimated centers is same as the true ordering of the elements obtained by sorting according to distance (in ascending order) from the exact centers (instead of approximated centers). Furthermore, the fact that $\calP$ is a consistent center-based clustering implies that for any $j\in[k]$ the membership weights $\s{U}_{\pi_{\widehat{\f{\mu}}_j(i)}j}$ are non-increasing in $i$. Therefore, for any $j\in[k]$ using the binary search in Algorithm~\ref{algo:binary_search}, we find the maximum index $\ell_s$ such that $\s{U}_{\pi_{\widehat{\f{\mu}}_j(\ell_s)}j}$ is larger than $s\eta$, for $s\in\{0,1,\dots,1/\eta\}$. This will create an $\eta$-spaced grid, and we approximate each membership weight in accordance to the cell it falls into. Specifically, we assign $\widehat{\s{U}}_{\pi_{\widehat{\f{\mu}}_j(t)}j}=s\eta$, for all $t\in[\ell_s,\ell_s-1,\ldots,\ell_{s+1}+1]$. Clearly, the estimation error is at most $\eta$. Overall, each binary search step takes $O(\log n)$ queries, and therefore the total query complexity of this step is $O(\log n/\eta)$. Finally, Step 6 transforms $\widehat{\s{U}}_{ij}$ to ensure that $\sum_{i=1}^n\widehat{\s{U}}_{ij}=1$, for $j\in[k]$. We have the following result, the detailed proof of which is delegated to Appendix \ref{sec:theorem1}.

\begin{theorem}\label{thm:1}
Let $(\calX,\calP)$ be a consistent center-based clustering instance, and recall the definitions of $\gamma\in\mathbb{R}_+$ and $\beta\in(0,1)$ in \eqref{eqn:equiv}-\eqref{eqn:beta}. Pick $\epsilon\leq\gamma$, $\eta\in(0,1)$, and $\delta\in(0,1)$. Take as input to Algorithm~\ref{algo:estimate_together} $m\ge \p{\frac{\s{R}k^{\alpha}}{\epsilon \beta^{\alpha}}}^4 c \log\frac{k}{\delta}$, for some $c>0$. Then, with probability at least $1-\delta$, Algorithm~\ref{algo:estimate_together} is $(\epsilon,\eta,\s{Q})$-solver using $\s{Q}=O\p{\frac{\s{R}^4k^{4\alpha+1}}{(\epsilon \beta^{\alpha})^4}\log\frac{k}{\delta}+\frac{k \log n}{\eta}}$ membership queries. Furthermore, Algorithm~\ref{algo:estimate_together} time-complexity is of order $O\p{kn\log n+\frac{d\s{R}^4k^{4\alpha+1}}{(\epsilon \beta^{\alpha})^4}\log\frac{k}{\delta} +\frac{k \log n}{\eta}}$.
\end{theorem}

\begin{proof}[Proof Sketch of Theorem~\ref{thm:1}]
To prove Theorem~\ref{thm:1} we first show that with high probability, Steps 1--4 in Algorithm~1 output an $\epsilon$-approximation for all centers. To that end, we first show in Lemma~\ref{lem:mean1} that for a given cluster $j\in[k]$, with probability $1-\delta$, $\norm{\f{\mu}_j-\widehat{\f{\mu}}_j}^2_2\leq\frac{4\s{R}^2}{\s{Y}^2}\sqrt{\frac{c}{m}\log\frac{1}{\delta}}$, where $\s{Y}\triangleq \min_{j \in [k]} \frac{1}{n}\sum_{i \in [n]} \s{U}_{ij}^{\alpha}$. This result follows by noticing that both the numerator and denominator of $\widehat{\f{\mu}}_j$ are unbiased estimates of the corresponding numerator and denominator of $\f{\mu}_j$ in \eqref{eqn:mu_solve}. Representing $\widehat{\f{\mu}}_j$ as a sum of the true underlying center $\f{\mu}_j$ and a certain error term, the generalized Hoeffding's inequality in Lemma~\ref{lem:gen_hoeffding} implies the above estimation error guarantee. In light of this result, using H\"{o}lder's inequality and the union bound we show in Lemma~\ref{lem:mean_together} that $\norm{\widehat{\f{\mu}}_j-\f{\mu}_j}_2 \le \epsilon$, \emph{for all} $j\in[k]$, if $m$ is as given in Theorem~\ref{thm:1}. Note that this centers estimation algorithm requires a total of $km$ membership-oracle queries. Next, conditioned on the above proxy of any given center, we show in Lemma~\ref{lem:membership} that Algorithm~\ref{algo:membership} estimates the corresponding membership weights within a certain range of error with high probability. Specifically, for a given cluster $j$, we prove that Algorithm~\ref{algo:membership} outputs $\widehat{\s{U}}_{ij}$, for $i\in[n]$, such that $0 \le \s{U}_{ij}-\widehat{\s{U}}_{ij} \le \eta$ and $\widehat{\s{U}}_{ij} \in \{0,\eta,2\eta,\dots,1\}$, $\forall i \in [n]$, using $O\p{\log n/\eta}$ queries to the membership-oracle. Roughly speaking, this result follows from the arguments presented in the paragraph preceding the statement of Theorem~\ref{thm:1}.
Finally, it is evident from the above that the total number of membership-oracle queries is $O(km+k\log n/\eta)$.
\end{proof}

Note that both the query and computational complexities in Theorem \ref{thm:1} exhibit an exponential dependency on the ``fuzzifier" $\alpha$. As mentioned in Subsection~\ref{subsec:optimiz}, this parameter is not subject to optimization and, usually, practitioners pick $\alpha=2$. Nonetheless, there has been a lot of work on this topic; one notable example is \cite{HUANG20122280} which presents an initial theoretical approach towards the investigation of this question. It seems that the value of $\alpha$ is influenced by the structure of the dataset, i.e., the number of clusters and the distribution of each cluster. Currently, a definite answer on how to pick $\alpha$ is unknown. 
Indeed, since there is no baseline, one cannot decide which value of $\alpha$ is preferable. Note also that despite the fact that our query complexity naturally depend on $\alpha$, it cannot be used to determine the ``best" value of $\alpha$. Assuming the existence of a ground truth, in Appendix~\ref{app:exper} we compare the clustering accuracy for different values of $\alpha$ on a real-world dataset.

\subsection{Sequential Algorithm}

As mentioned in the previous subsection, the main drawback faced by Algorithm~\ref{algo:estimate_together} is its undesirable strong dependency on $\beta$. The algorithm proposed in this section remedies this drawback to some extent. The algorithm works as follows. In the first four steps of Algorithm~\ref{algo:estimate_together}, we randomly draw $m$ samples from $[n]$, and query the membership weights that correspond to all these samples. Using these samples we estimate the center of \emph{largest} cluster $t_1$, having the maximal total membership weight $\sum_{i\in\calS}\s{U}_{ij}^\alpha$, among the $k$ clusters. Corollary~\ref{coro:base1} in Section \ref{sec:theorem2} specifies the number of such elements needed to have a reliable estimate. Then, using this approximated center, we estimate the membership of every element in cluster $t_1$ using Algorithm~\ref{algo:membership} such that the estimate is a multiple of $\eta_1$ and further, the estimate has an additive error of at most $\eta_1$, just as we did in the two-phase algorithm. For the rest of the clusters, we follow the sequential algorithm in Steps 5--13 of Algorithm~\ref{algo:estimate_sequentially}, the guarantees of which are characterized in the following result:
\begin{theorem}\label{thm:2}
Let $(\calX,\calP)$ be a consistent center-based clustering instance, and recall the definitions of $\gamma\in\mathbb{R}_+$ and $\beta\in(0,1)$ in \eqref{eqn:equiv}-\eqref{eqn:beta}. Take $m\ge \p{\frac{\s{R}k^{\alpha}}{\epsilon}}^4 c \log\frac{2k}{\delta}$, $r\ge\frac{c\s{R}^4k^{4\alpha}}{\epsilon^4\beta^{4\alpha-4}}\log\frac{4k}{\eta_1\delta}$, for some $c>0$, and any $\delta\in(0,1)$. Then, with probability at least $1-\delta$, Algorithm~\ref{algo:estimate_sequentially} is $(\epsilon,\eta_2,\s{Q})$-solver using $\s{Q}=O\p{\frac{\s{R}^4k^{4\alpha+1}}{\eta_1\epsilon^4 \beta^{4\alpha-4}}\log\frac{4k}{\eta_1 \delta}+\frac{k \log n}{\eta_1}+\frac{k \log n}{\eta_2}}$ membership queries, for any chosen $\epsilon\leq\gamma$ and $\eta_2 \le \eta_1 \le \frac{1}{k}\p{1-\frac{\beta}{k}}$. Furthermore, Algorithm~\ref{algo:estimate_sequentially} time-complexity is of order $O\p{\frac{d\s{R}^4k^{4\alpha+1}}{\eta_1\epsilon^4 \beta^{4\alpha-4}}\log\frac{4k}{\eta_1 \delta} +\frac{k \log n}{\eta_1}+\frac{k \log n}{\eta_2}+kn\log n}$.
\end{theorem}

The proof of Theorem \ref{thm:2} is given in Subsection~\ref{sec:theorem2}, and is based on a technically complicated induction arguments. Below, we provide a short sketch.

\begin{proof}[Proof Sketch of Theorem \ref{thm:2}]

Suppose that we have approximated the means of $\ell\leq k$ clusters, indexed by $\calT_{\ell}\triangleq\{t_1,t_2,\dots,t_{\ell}\}$, with an estimation error at most $\epsilon \le \gamma$, and further approximated the membership weights of every element in $\calX$ for all clusters indexed by $\calT_{\ell}$, up to an additive error $\eta_1$. The leftover clusters, indexed by $[k]\setminus\calT_\ell$, are the \textit{unprocessed clusters}. The above implies that we can also approximate the sum $\sum_{j \in [\ell]}\s{U}_{it_{j}}$, for each $i\in\calX$, up to an additive error $\ell\eta_1$. Note that, since $\widehat{\s{U}}_{it_j}$ is a multiple of $\eta_1$ for all $j \in [\ell]$, $\sum_{j \in [\ell]}\widehat{\s{U}}_{it_j}$ must also be a multiple of $\eta_1$. Subsequently, we partition $\calX$ into bins, where the elements in a particular bin have the same approximated sum of membership weights. Formally, let $\Sigma \triangleq \{0,1,\dots,1/\eta_1\}$. For each $s \in \Sigma$, we define $\ca{X}_s \triangleq \{i \in [n]: \sum_{j \in [\ell]}\widehat{\s{U}}_{it_j}=s\eta_1\}$ to be the $s^{\s{th}}$ bin. Then, in Steps~9-10 of Algorithm~\ref{algo:estimate_sequentially}, we randomly sample $\min(r,|\calX_s|)$ elements from each bin with replacement, and denote the resulted sub-sampled set by $\ca{Y}_s$. The value of $r$ is specified in Theorem~\ref{thm:2}. Sampling equally from every bin ensures that we obtain enough samples from elements which have high membership weight to the unprocessed clusters, i.e., $\sum_{j \in [k]\setminus[\ell]} \s{U}_{it_j}$ is large. Given these sub-sampled sets, we query the membership weights $\s{U}_{ij}$, for all $i\in\cup_{s\in\Sigma}\calY_s$ and $j\in[k]\setminus\calT_{\ell}$. Using these queries we would like to approximate the center of some unprocessed cluster. We choose the unprocessed cluster having the largest total membership size, weighted by the size of each bin. Mathematically, the index of the $(\ell+1)^{\s{th}}$ cluster, and its approximated center are given as follows,
\begin{align}
t_{\ell+1} &\triangleq \underset{j \in [k]\setminus \calT_{\ell}}{\s{argmax}}\;\sum_{s\in\Sigma}\frac{|\ca{X}_s|}{r}\sum_{i \in \ca{Y}_s}\s{U}_{ij}^{\alpha}\quad\s{and}\quad\widehat{\f{\mu}}_{t_{\ell+1}} \triangleq \frac{\sum_{s\in\Sigma}\frac{|\ca{X}_s|}{r}\sum_{i \in \ca{Y}_s} \s{U}_{it_{\ell+1}}^{\alpha} \f{x}_i}{\sum_{s\in\Sigma} \frac{|\ca{X}_s|}{r}\sum_{i \in \ca{Y}_s} \s{U}_{it_{\ell+1}}^{\alpha}}.\label{eqn:tlplus1}
\end{align}
We choose the number of samples $r$ large enough so that the estimation error is less than $\epsilon$. The above steps are repeated until all unprocessed clusters are exhausted. Finally, once we have approximated all the centers, we apply Algorithm~\ref{algo:membership} to approximate the membership weight of every cluster up to an additive error of $\eta_2\leq\eta_1$. 
\end{proof}
Comparing Theorems~\ref{thm:1} and \ref{thm:2} we see that the sequential algorithm has a more graceful dependency on $\beta$, the size of the smallest cluster. On the other hand, the query complexity of the two-phase algorithm is better in terms of $k$ (recall that $\eta_1\leq1/k$ for the sequential algorithm). Note that the query complexity of the sequential algorithm still has some dependency on $\beta$. This stems from the fact that we estimate the membership weights up to a precision of $\eta_1$, which might be too large if there exists a cluster $j$ such that $\max_{i\in[n]}\s{U}_{ij} \le \eta_1$. Accordingly, in this case, all elements will fall into a single bin in the partition, leading to the same drawback the two-phase algorithm suffers from. 

\begin{algorithm}[t!]
\caption{Sequential algorithm for approximating $\calP$. \label{algo:estimate_sequentially}} 
\footnotesize
\begin{multicols}{2}
\begin{algorithmic}[1]
\REQUIRE $\ca{X}$, $\calO_{\s{fuzzy}}$, $\alpha$, $m$, $r$, $\eta_1$, and $\eta_2$. 
\ENSURE An estimate $\widehat{\calP}$.
\STATE Let $\calS$ be a set of $m$ elements sampled at random from $[n]$ with replacement.
\STATE Query $\ca{O}_{\s{fuzzy}}(i,j),\;\forall i\in\calS,j \in [k]$.
\STATE $t_1 = \s{argmax}_{j \in [k]}\sum_{i \in \ca{S}} \s{U}_{ij}^{\alpha}.$
\STATE Compute $\widehat{\f{\mu}}_{t_1} = \frac{\sum_{i \in \ca{S}} \s{U}_{it_1}^{\alpha} \f{x}_i}{\sum_{i \in \ca{S}} \s{U}_{it_1}^{\alpha}}$.
\FOR{$\ell=1,2,\dots,k-1$}
\STATE $\{\widehat{\s{U}}_{i,\ell}\}_{i=1}^n\leftarrow$\textsc{Membership($\ca{X}, \widehat{\f{\mu}}_{t_{\ell}}, \alpha, \eta_1$)}.
\STATE Set  $\ca{X}_s \triangleq \{i \in [n]: \sum_{j \in [\ell]} \widehat{\s{U}}_{it_j} = s\eta_1 \}$ for $s \in \{0,1,\dots,1/\eta_1\}$.
\FOR{$s=0,1,\dots,1/\eta_1$}
\STATE Let $\calY_s$ be a set of $\min(r,|\calX_s|)$ elements sampled at random from $\ca{X}_s$ with replacement.
\ENDFOR
\STATE Compute $t_{\ell+1}$ using \eqref{eqn:tlplus1}. 
\STATE Compute $\widehat{\f{\mu}}_{t_{\ell+1}}$ using \eqref{eqn:tlplus1}.
\ENDFOR
\FOR{$j=1,2,\dots,k$}
\STATE $\{\widehat{\s{U}}_{ij}\}_{i=1}^n\leftarrow$\textsc{Membership($\ca{X}, \widehat{\f{\mu}}_j,\alpha,\eta_2$)}.
\STATE $\widehat{\s{U}}_{ij}\leftarrow\widehat{\s{U}}_{ij}+\frac{1-\sum_{j=1}^k\widehat{\s{U}}_{ij}}{k},\;\forall i\in[n]$.
\ENDFOR
\end{algorithmic}
\end{multicols}
\end{algorithm}

\subsection{The Special Case of Two Clusters}

While the sequential algorithm proposed in the previous subsection exhibits graceful dependency of the query complexity on $\beta$, it is still unclear to what extent this dependency can be reduced. Interestingly, it turns out that for the special case of two clusters $k=2$, we can devise an algorithm whose query complexity is completely independent of $\beta$. The proof of the following result is given in Subsection~\ref{app:algoTh3}.

\begin{theorem}\label{thm:3state}
Let $(\calX,\calP)$ be a consistent center-based clustering instance, recall the definition of $\gamma\in\mathbb{R}_+$ in \eqref{eqn:equiv}, and let $k=2$. Then, with probability at least $1-\delta$, Algorithm~\ref{algo:estimate_sequentially2} with $m=O\p{\Big(\frac{\s{R}}{\epsilon}\Big)^4\log \frac{4}{\delta}}$ and $r= O\p{\frac{\s{R}^4}{\epsilon^4}\log\frac{2}{\eta\delta}}$ is a $(\epsilon,\eta,\s{Q})$-solver using $$\s{Q}=O\p{\frac{\s{R}^4 \log n \log (1/\eta\delta)}{\epsilon^4}+\log^{2}n +\frac{\log n}{\eta} }$$ membership queries, for any chosen $\epsilon\leq\gamma, \; \delta \in (0,1)$ and $\eta>0$. Furthermore, Algorithm~\ref{algo:estimate_sequentially2} time-complexity is of order $O\p{n\log n+\frac{d\s{R}^4 \log n \log (1/\eta\delta)}{\epsilon^4}+\log^{2}n +\frac{\log n}{\eta}}$. 
\end{theorem}

\begin{algorithm}
\caption{Sequential algorithm for approximating $\calP$ with two clusters \label{algo:estimate_sequentially2}} 
\begin{algorithmic}[1]
\REQUIRE $\calX$, $\ca{O}_{\s{fuzzy}}$, $\alpha$, $m$, $r$, and $\eta$.
\STATE Let $\calS$ be a set of $m$ elements sampled at random from $[n]$ with replacement.
\STATE Query $\ca{O}_{\s{fuzzy}}(i,j)$ and store $\s{U}_{ij}$ for all $j \in \{1,2\}$.
\STATE Choose $t_1 = \s{argmax}_{j \in \{1,2\}}\sum_{i \in \ca{S}} \s{U}_{ij}^{\alpha}.$
\STATE Compute estimator $\widehat{\f{\mu}}_{t_1} = \frac{\sum_{i \in \ca{S}} \s{U}_{it_1}^{\alpha} \f{x}_i}{\sum_{i \in \ca{S}} \s{U}_{it_1}^{\alpha}}$.
\STATE Obtain $\ca{P}_1,\ca{P}_2,\ca{X}_1,\ca{X}_2,\dots,\ca{X}_q$ and $\widehat{\s{U}}_{it_2}$ for all $i \in [n]$ using \textsc{Membership2($\ca{X}, \widehat{\f{\mu}}_{t_1}, \alpha$)}
\FOR{$q=1,2,3,\dots,3\log n$}
\STATE Let $\ca{Y}_q$ be a set of $r$ elements sampled at random from $\ca{X}_q$ with replacement.
\STATE Query $\ca{O}_{\s{fuzzy}}(i,t_2)$ to obtain $\s{U}_{it_2}$ for all $i \in \ca{Y}_q$
\ENDFOR
\STATE Let $\ca{Q}$ be a set of $r$ elements sampled at random from $\ca{P}_1$ with replacement.
\STATE Query $\ca{O}_{\s{fuzzy}}(i,t_2)$ to obtain $\s{U}_{it_2}$ for all $i \in \ca{Q}$.
\STATE Compute $\widehat{\f{\mu}}_{t_{2}}$ in \eqref{eq:secondmean}.
\FOR{$j=1,2$}
\STATE Run Algorithm \textsc{Membership($\ca{X}, \widehat{\f{\mu}}_j, \alpha, \eta, j$)} to update $\widehat{\s{U}}_{ij}$ for all $i \in [n]$.
\ENDFOR
\STATE Return $\widehat{\f{\mu}}_j$ for all $j \in \{1,2\}$ and $\widehat{\s{U}}_{ij}$ for all $i \in [n], j\in \{1,2\}$.
\end{algorithmic}
\end{algorithm}

We now explain the main ideas behind Algorithm~\ref{algo:estimate_sequentially2}. First, let $\calS$ be a set of $m$ elements sampled from $[n]$ with replacement. Denote by $t_1\in\{1,2\}$ the index of the larger cluster and $t_2\in\{1,2\}$ the index of the smaller cluster, i.e., $\sum_{i\in\calS}\s{U}_{it_1}^\alpha>\sum_{i\in\calS}\s{U}_{it_2}^\alpha$. The algorithm works as follows. We start by computing an estimate $\widehat{\f{\mu}}_{t_1}$ for the center of the largest cluster $t_1$. Corollary~\ref{coro:base1} shows that  $O\p{(\s{R}/\epsilon)^4\log(4/\delta)}$ membership queries are needed to guarantee that $\norm{\widehat{\f{\mu}}_{t_1}-\f{\mu}_{t_1}}_2\le \epsilon$. Now, the main novelty in Algorithm~\ref{algo:estimate_sequentially2} is a non-trivial querying scheme used to estimate the memberships of the second cluster, $\s{U}_{it_2}$, for all $i\in[n]$, using the center estimated for the first cluster $t_1$. Note that for the special case of two clusters, querying $\s{U}_{it_1}$ reveals $\s{U}_{it_2}$ since $\s{U}_{it_1}+\s{U}_{it_2}=1$. While, in the sequential algorithm, we always approximated the membership weights using bins of fixed size $\eta_1$, here, we will approximate $\s{U}_{it_2}$ with bins whose size is \emph{data-dependent}. The reason for using variable bin sizes rather than fixed size bins is that with the latter we might ask for many redundant membership queries simply because many bins might be empty. Accordingly, by adaptively choosing the bin size for approximating the membership weights of the smaller cluster $t_2$, while querying membership weights to the larger cluster $t_1$, allows us to get around this issue completely; this results in a theoretical guarantee that is completely independent of $\beta$. 

We now describe Algorithm~\ref{algo:membership2} used to estimate the membership weights of cluster $t_2$. We start by sorting $\calX$ in ascending order according to their distance from $\widehat{\f{\mu}}_{t_1}$, and initialize $\eta_1= \s{U}_{\pi_{\widehat{\f{\mu}}_{t_1}}(n)t_2}$ and 
$p_{\eta_1} \triangleq n$. 
To wit, $\eta_1$ is the membership of the element farthest from $\widehat{\f{\mu}}_{t_1}$ to the smaller cluster $t_2$. 
For each $q \in =\{2,3\dots,3\log n \}$, we set $\eta_q = \eta_{q-1}/2$, and then perform binary search to find $p'_{\eta_q}=\s{argmin}_{j \in [n]} \mathds{1}[\s{U}_{\pi_{\widehat{\f{\mu}}_{t_1}}(j)t_2}\ge \eta_q]$. Here, we should think of $p'_{\eta_q}$ to be the index for which $\f{x}_{\pi_{\widehat{\f{\mu}}_{t_1}}(p'_{\eta_q})}$ is the element closest from $\widehat{\f{\mu}}_{t_1}$, such that its membership to the smaller cluster $t_2$ is at least $\eta_q$. Now, if $|p'_{\eta_q}-p_{\eta_{q-1}}| \ge \log n$, we keep $\eta_q$ unchanged, set $p_{\eta_q}=p'_{\eta_q}$, and put all elements which are closer to $\widehat{\f{\mu}}_{t_1}$ than $\f{x}_{\pi_{\widehat{\f{\mu}}_{t_1}}(p_{\eta_{q-1}})}$, but farther than $\f{x}_{\pi_{\widehat{\f{\mu}}_{t_1}}(p_{\eta_{q}})}$ into a single bin. Notice that each such bin contains at least $\log n$ elements by definition. However, if $p'_{\eta_q}$ is the same as $p_{\eta_{q-1}}$, then the membership queries in the latter binary search to compute $p'_{\eta_q}$ is wasteful. In order to fix this issue, we check if there are a sufficient number (at least $\log n$) of elements between $p'_{\eta_q}$ and $p_{\eta_{q-1}}$. If there are not, then we set $p_{\eta_q}$ such that the above condition is satisfied. More formally, if $|p'_{\eta_q}-p_{\eta_{q-1}}| \le \log n$ 
we update $p_{\eta_q} = \min(0,p_{\eta_{q-1}}-1-\log n)$ and $\eta_q =  \s{U}_{\pi_{\widehat{\f{\mu}}_{t_1}}(p_{\eta_q})t_2}$. In other words, we set $p_{\eta_q}$ such that there at least $\log n$ elements between $p_{\eta_q}$ and $p_{\eta_{q-1}}$, and accordingly, we set $\eta_q$ to be the membership of $\pi_{\widehat{\f{\mu}}_{t_1}}(p_{\eta_q})$ to the cluster $t_2$. Subsequently, we query $\s{U}_{it_2}$ for every element that is closer to $\widehat{\f{\mu}}_{t_1}$ than $\f{x}_{\pi_{\widehat{\f{\mu}}_{t_1}}(p_{\eta_{q-1}})}$ but farther than $\f{x}_{\pi_{\widehat{\f{\mu}}_{t_1}}(p_{\eta_{q}})}$. In this case, we call these elements \textit{special elements}, since we query every one of them. We will assume that all these special elements between $\f{x}_{\pi_{\widehat{\f{\mu}}_{t_1}}(p_{\eta_{q-1}})}$ and $\f{x}_{\pi_{\widehat{\f{\mu}}_{t_1}}(p_{\eta_{q}})}$ are assigned a single bin.

\begin{algorithm}[t]
\caption{ \textsc{Membership2($\ca{X},\widehat{\f{\mu}}_{t_1}, \alpha$)} Estimate the memberships $\s{U}_{it_2}$ for all $i \in [n]$ given an estimated mean $\widehat{\f{\mu}}_{t_1}$ \label{algo:membership2}}
\begin{multicols}{2}
\begin{algorithmic}[1]
\REQUIRE $\ca{O}_{\s{fuzzy}}$.
\STATE Sort the elements $\f{x}_1,\f{x}_2,\dots,\f{x}_n$ in ascending order according to $\norm{ \f{x}_i- \widehat{\f{\mu}}_{t_1} }_2$. Denote the resultant permutation of $[n]$ corresponding to the sorted elements by $\pi_{\widehat{\f{\mu}}_{t_1}}$. 
\STATE Set $\eta_1  \triangleq  1-\s{U}_{\pi_{\widehat{\f{\mu}}_{t_1}}(n)t_1} \quad \text{and} \quad p_{\eta_1} \triangleq n$.
\STATE Query $\ca{O}_{\s{fuzzy}}(\pi_{\widehat{\f{\mu}}_{t_1}}(n),t_1)$ to obtain $\s{U}_{\pi_{\widehat{\f{\mu}}_{t_1}}(n)t_1}$. Set $\s{U}_{\pi_{\widehat{\f{\mu}}_{t_1}}(n),t_2}=1-\s{U}_{\pi_{\widehat{\f{\mu}}_{t_1}}(n)t_1}$.
\STATE Initialize $\ca{P}_1, \ca{P}_2=\phi$
\FOR{$q=2,3,\dots,3 \log n$}
         \STATE Initialize $\ca{X}_q = \phi$ and set $\eta_q=\frac{\eta_{q-1}}{2}$
          \STATE Find $p'_{\eta_q} = \s{argmin}_{i \in [n]} \s{U}_{\pi_{\widehat{\f{\mu}}_{t_1}}(i),t_2} \ge \eta_q$ using \textsc{BinarySearch2}($\ca{X}, \pi_{\widehat{\f{\mu}}_{t_1}}, \eta_q$).
\IF{$|p'_{\eta_q}-p_{\eta_{q-1}}| \ge \log n$}
         \STATE Set $p_{\eta_q} = p'_{\eta_q}$
         \FOR{$i=p_{\eta_q},p_{\eta_q}+1,\dots,p_{\eta_{q-1}}-1$}
                  \STATE Set $\ca{X}_q = \ca{X}_q \cup \pi_{\widehat{\f{\mu}}_{t_1}}(i)$ and            set $\widehat{\s{U}}_{\pi_{\widehat{\f{\mu}}_{t_1}}(i),t_2} = 1-\s{U}_{\pi_{\widehat{\f{\mu}}_{t_1}}(p_{\eta_q})t_1}$
         \ENDFOR
\ELSE
         \STATE Query $\ca{O}_{\s{fuzzy}}(\pi_{\widehat{\f{\mu}}_{t_1}}(\min(0,p_{\eta_{q-1}}-\log n-1)),t_1)$ and obtain the membership $\s{U}_{\pi_{\widehat{\f{\mu}}_{t_1}}                                        (\min(0,p_{\eta_{q-1}}-\log n-1)), t_1}$.
         \STATE Set $\eta_q = \s{U}_{\pi_{\widehat{\f{\mu}}_{t_1}}(\min(0,p_{\eta_{q-1}}-\log n-1)), t_2}$ and $p_{\eta_q} = \min(0,p_{\eta_{q-1}}-1-\log n)$.
         \FOR{$i=p_{\eta_q},p_{\eta_q}+1,\dots,p_{\eta_{q-1}}-1$}
                 \STATE Query $\ca{O}_{\s{fuzzy}}(\pi_{\widehat{\f{\mu}}_{t_1}}(i),t_1)$ to obtain $\s{U}_{\pi_{\widehat{\f{\mu}}_{t_1}}(i)t_1}$.
                 \STATE Set $\widehat{\s{U}}_{\pi_{\widehat{\f{\mu}}_{t_1}}(i),t_2} = 1-\s{U}_{\pi_{\widehat{\f{\mu}}_{t_1}}(i)t_1}$
                 and set $\ca{X}'_q = \ca{X}'_q \cup \pi_{\widehat{\f{\mu}}_{t_1}}(i)$
         \ENDFOR
\ENDIF
\ENDFOR
\FOR{$i=0,1,\dots, \eta_{3\log n}-1$}
         \STATE Set $\widehat{\s{U}}_{\pi_{\widehat{\f{\mu}}_{t_1}}(i),t_2} = 0$
         and set $\ca{P}_1= \ca{P}_1 \cup \pi_{\widehat{\f{\mu}}_{t_1}}(i)$
\ENDFOR
\STATE Set $\ca{P}_2 = \cup_q \ca{X}'_q$.
\STATE Return $\ca{P}_1,\ca{P}_2,\ca{X}_1,\ca{X}_2,\dots,\ca{X}_q$ and $\widehat{\s{U}}_{it_2}$ for all $i \in [n]$. 
\end{algorithmic}
\end{multicols}
\end{algorithm}

\begin{algorithm}
\caption{ \textsc{BinarySearch2($\ca{X}, \pi, x$)}: Search for the minimum index $i$ such that $1-\s{U}_{\pi(i)t_1} \ge x$ 
 \label{algo:binary_search2} }
\begin{algorithmic}[1]
\REQUIRE $\ca{O}_{\s{fuzzy}}$.
\STATE Set $\s{low}=1$ and $\s{high}=n$. 
\WHILE{$\s{low} \neq \s{high}$} 
\STATE Set $\s{mid}= \lfloor (\s{low}+\s{high})/2 \rfloor$.
\STATE Query $\ca{O}_{\s{fuzzy}}(\pi(\s{mid})t_1)$ to obtain $\s{U}_{\pi(\s{mid})t_1}$. 
\IF{$\s{U}_{\pi(\s{mid})t_1} \ge 1-x$}
\STATE Set $\s{low}=\s{mid}+1$
\ELSE
\STATE Set $\s{high} =\s{mid}$
\ENDIF
\ENDWHILE
\STATE Return $\s{low}$.
\end{algorithmic}
\end{algorithm}

To resolve the edge case, we put all those elements which are closer to $\widehat{\f{\mu}}_{t_1}$ than $\f{x}_{\pi_{\widehat{\f{\mu}}_{t_1}}(p_{\eta_{3\log n}})}$ into a separate bin. Let $g$ be the total number of bins formed by this algorithm, and say $\ca{X}_j$ is the $j^{\s{th}}$ bin formed by this method that does not contain any special elements. Let $\ca{P}$ be the set of all special elements. Notice that the number of bins $g$ is at most $1+3\log n$, and the number of special elements is at most $3\log^{2}n$ (since each bin can contain at most $\log n$ special elements). Moreover, for any bin $\ca{X}_j$ not containing any special elements, we can show that $\max_{i \in \ca{X}_j}\s{U}_{it_2} \le 2\min_{i \in \ca{X}_j}\s{U}_{it_2}$, which should be contrasted with the guarantee we have for the sequential algorithm, i.e., for a particular bin $\ca{X}_j$, $|\max_{i \in \ca{X}_j}\s{U}_{it_2} -\min_{i \in \ca{X}_j} \s{U}_{it_2}| \le \eta_1$. Finally, we again sub-sample $r$ elements from each such bin $\ca{X}_j$ with replacement, and denote the resulting subset by $\ca{Y}_j$. We approximate the center of the second cluster by
\begin{align}
     \widehat{\f{\mu}}_{t_{2}} = \frac{ \sum_{j} \frac{|\ca{X}_j|}{r}\sum_{i \in \ca{Y}_j} \s{U}_{it_{2}}^{\alpha} \f{x}_i+\sum_{i \in \ca{P}} \s{U}_{it_{2}}^{\alpha} \f{x}_i}{ \sum_{j} \frac{|\ca{X}_j|}{r}\sum_{i \in \ca{Y}_j} \s{U}_{it_{2}}^{\alpha}+\sum_{i \in \ca{P}} \s{U}_{it_{2}}^{\alpha}},
\end{align}
and show that the required $r$ is independent of the cluster size.

\section{Proofs}

\subsection{Proof of Theorem~\ref{thm:1}}\label{sec:theorem1}

To prove Theorem~\ref{thm:1} we will establish some preliminary results. 

\subsubsection{Auxiliary Lemmata}

We start by analyzing the performance of Algorithm~\ref{algo:uniform}, which estimates the center of a given cluster using a set of randomly sampled elements. Note that this algorithm is used as a sub-routine of Algorithm~\ref{algo:estimate_together}.
\begin{algorithm}[htbp]
\caption{Algorithm for estimating the mean $\f{\mu}_j$ for any $j \in [k]$. \label{algo:uniform}} 
\begin{algorithmic}[1]
\REQUIRE $\calX$, $\calO_{\s{fuzzy}}$, $\alpha$, and $m$.
\ENSURE $\widehat{\f{\mu}}_j$
\STATE Initialize $\ca{S}\leftarrow\phi$.
\FOR{$s=1,2,\dots,m$}
\STATE Sample $i$ uniformly at random from $[n]$ and update $\ca{S}\leftarrow\ca{S} \cup \{i\}$.
\STATE Query $\ca{O}_{\s{fuzzy}}(i,j)$.
\ENDFOR
\STATE Compute $\widehat{\f{\mu}}_j= \frac{\sum_{i \in \ca{S}} \s{U}_{ij}^{\alpha} \f{x}_i}{\sum_{i \in \ca{S}} \s{U}_{ij}^{\alpha}}$.
\end{algorithmic}
\end{algorithm}
\begin{lemma}[Estimate of mean using uniform sampling]{\label{lem:mean1}}
Let $(\calX,\calP)$ be a consistent center-based clustering instance, and let $\delta\in(0,1)$. With probability at least $1-\delta$, Algorithm~\ref{algo:uniform} outputs an estimate $\widehat{\f{\mu}}_j$ such that
\begin{align}
\norm{\f{\mu}_j-\widehat{\f{\mu}}_j}^2_2\leq\frac{4\s{R}^2}{\s{Y}^2}\sqrt{\frac{c}{m}\log\frac{1}{\delta}},
\end{align}
where $\s{Y}\triangleq \min_{j \in [k]} \frac{1}{n}\sum_{i \in [n]} \s{U}_{ij}^{\alpha}$, and $c>0$ is some absolute constant.
\end{lemma}

\begin{proof}[Proof of Lemma~\ref{lem:mean1}]
First, note that
\begin{align}
\widehat{\f{\mu}}_j = \frac{\sum_{i \in \ca{S}} \s{U}_{ij}^{\alpha} \f{x}_i}{\sum_{i \in \ca{S}} \s{U}_{ij}^{\alpha}} = \frac{(1/m)\sum_{i \in \ca{S}} \s{U}_{ij}^{\alpha} \f{x}_i}{(1/m)\sum_{i \in \ca{S}} \s{U}_{ij}^{\alpha}}\triangleq \frac{\bar{\f{\lambda}}_{\f{x}}}{\bar{\s{Y}}}.
\end{align}
Recall that the true mean of the $j^{\s{th}}$ cluster is  
\begin{align}
\f{\mu}_{j} = \frac{\sum_{i \in [n]} \s{U}_{ij}^{\alpha} \f{x}_i}{\sum_{i \in [n]} \s{U}_{ij}^{\alpha}} = \frac{(1/n)\sum_{i \in [n]} \s{U}_{ij}^{\alpha} \f{x}_i}{(1/n)\sum_{i \in [n]} \s{U}_{ij}^{\alpha}}\triangleq\frac{\f{\lambda}_{\f{x}}}{\s{Y}}.
\end{align}
It is clear that $\bar{\f{\lambda}}_{\f{x}}$ and $\bar{\s{Y}}$ are an unbiased estimators of $\f{\lambda}_{\f{x}}$ and $\s{Y}$, respectively. Now, note that we can write $\bar{\f{\lambda}}_{\f{x}}$ as an average of $m$ i.i.d random variables $\bar{\f{x}}_{i_p} =\s{U}_{i_pj}^{\alpha}\f{x}_{i_p}$, where $i_p$ is sampled uniformly at random from $[n]$, and included to the set $\ca{S}$ in the third step of Algorithm~\ref{algo:uniform} as the $p^{\s{th}}$ sample. Similarly, $\bar{\s{Y}}$ can also be written as the average of $m$ i.i.d random variables $\bar{\s{Y}}_{i_{p}}=\s{U}_{i_pj}^{\alpha}$. Further, notice that $\bb{E}\bar{\f{x}}_{i_p} = \f{\lambda}_{\f{x}}$, $\norm{\bar{\f{x}}_{i_p}}_2\le \s{R}$, and similarly, $\bb{E}\bar{\s{Y}}_{i_p} = \s{Y}$, and $\abs{\bar{\s{Y}}_{i,p}}\le 1$, for all $p \in \{1,2,\dots,m\}$. 
Next, note that
\begin{align}
\widehat{\f{\mu}}_{j}&=\frac{\bar{\f{\lambda}}_{\f{x}}}{\bar{\s{Y}}}= \frac{{\f{\mu}}_{\f{x}}}{{\s{Y}}}+\frac{\bar{\f{\lambda}}_{\f{x}}-\f{\lambda}_{\f{x}}}{\s{Y}}+\frac{\bar{\f{\lambda}}_{\f{x}}}{\bar{\s{Y}}}\frac{\s{Y}-\bar{\s{Y}}}{\s{Y}}.
\end{align}
Thus, using the triangle inequality we get
\begin{align}
\norm{\widehat{\f{\mu}}_{j} - \f{\mu}_{j}}_2 \le \norm{\frac{\bar{\f{\lambda}}_{\f{x}}-\f{\lambda}_{\f{x}}}{\s{Y}}}_2+\norm{\frac{\bar{\f{\lambda}}_{\f{x}}}{\bar{\s{Y}}}\frac{\s{Y}-\bar{\s{Y}}}{\s{Y}}}_2 \le \norm{\frac{\bar{\f{\lambda}}_{\f{x}}-\f{\lambda}_{\f{x}}}{\s{Y}}}_2+\s{R}\norm{\frac{\bar{\s{Y}}-\s{Y}}{\s{Y}}}_2,
\end{align}
where in the last inequality we have used the fact that $\norm{\widehat{\f{\mu}}_j}_2 = \norm{\bar{\f{\lambda}}_{\f{x}}/\bar{\s{Y}}}_2\leq\s{R}$. Then, the generalized Hoeffding's inequality in Lemma \ref{lem:gen_hoeffding}, implies that with probability at least $1-\delta$,
\begin{align}
\norm{\bar{\f{\lambda}}_{\f{x}}-\f{\lambda}_{\f{x}}}_2^2 \le \s{R}^2\sqrt{\frac{c}{m}\log\frac{1}{\delta}},
\end{align}
for some $c>0$, and thus,
\begin{align}
\norm{\widehat{\f{\mu}}_{j} - \f{\mu}_{j}}_2^2 \le \frac{4\s{R}^2}{\s{Y}^2}\sqrt{\frac{c}{m}\log\frac{1}{\delta}}.
\end{align}  
which concludes the proof.
\end{proof}
The following lemma shows that if for a given cluster $j$ we have been able to approximate its center well enough, then Algorithm~\ref{algo:membership} computes good estimates of the corresponding membership weights with high probability. 
\begin{lemma}[Estimate of membership given estimated center]{\label{lem:membership}}
Let $(\calX,\calP)$ be a consistent center-based clustering instance, and recall the definition of $\gamma\in\mathbb{R}_+$ in \eqref{eqn:equiv}. Assume that for any $j\in[k]$, there exists an estimator $\widehat{\f{\mu}}_j$ such that $\norm{\f{\mu}_j-\widehat{\f{\mu}}_j}_2 \le \epsilon$ with $\epsilon\leq\gamma$. Then, Algorithm~\ref{algo:membership} outputs $\widehat{\s{U}}_{ij}$, for $i\in[n]$, such that
\begin{align}
&0 \le \s{U}_{ij}-\widehat{\s{U}}_{ij} \le \eta,\quad \widehat{\s{U}}_{ij} \in \{0,\eta,2\eta,\dots,1\}, \quad \forall i \in [n],\label{eqn:membershipEstLemma}
\end{align}
for some $\eta\in\mathbb{R}_+$, using $\s{Q} = O\p{\log n/\eta}$ queries to the membership-oracle. 
\end{lemma}

\begin{proof}[Proof of Lemma~\ref{lem:membership}]
First, note that since $\calP$ is a consistent center-based clustering, we have
\begin{align}
\s{U}_{\pi_{\f{\mu}_j}(i_1)j} \ge \s{U}_{\pi_{\f{\mu}_j}(i_2)j} \quad \text{if}\quad i_1<i_2, i_1,i_2\in[n].
\end{align}
Indeed, when the elements of $\calX$ are sorted in ascending order according to their distance from $\f{\mu}_j$, if $\f{x}_{u_1}$ is closer to $\f{\mu}_j$ than it is to $\f{x}_{u_2}$, then $\s{U}_{u_1j}>\s{U}_{u_2j}$. Also, since $\norm{\f{\mu}_j- \widehat{\f{\mu}}_j}_2 \le \epsilon\le\gamma$, using \eqref{eqn:equiv}, this ordering remains the same. Therefore, sorting the elements in $\ca{X}$ in ascending order from $\widehat{\f{\mu}}_j$ as in the first step of Algorithm~\ref{algo:membership} gives the same ordering with respect to the true mean. Now, given $\eta\in\mathbb{R}_+$, for each $s\in\{0,1,2,\dots,1/\eta\}$, in the third step of Algorithm~\ref{algo:membership} we binary search to find an index $\ell_s$ such that
\begin{align}
\ell_s = \underset{i \in [n]}{\s{argmax}}\; \s{U}_{\pi_{\widehat{\f{\mu}}_j}(i)j} \ge s\eta.
\end{align} 
This is done by using $O(\log n/\eta)$ membership-oracle queries. Finally, in the last three steps of Algorithm~\ref{algo:membership}, for each $s\in\{0,1,2,\dots,1/\eta\}$, and for $i\in \{\ell_s,\ell_s-1,\dots,\ell_{(s+1)}+1\}$, we assign $\widehat{\s{U}}_{\pi_{\widehat{\f{\mu}}_j}(i)j} = s\eta$. It is then clear that the estimated memberships satisfy \eqref{eqn:membershipEstLemma}, which concludes the proof.
\end{proof}

Using the above lemmas we obtain the following.
\begin{lemma}[Estimate all means]{\label{lem:mean_together}}
Let $(\calX,\calP)$ be a consistent center-based clustering instance, recall the definition of $\beta\in(0,1)$ in \eqref{eqn:beta}, and let $\delta\in(0,1)$. Then, with probability at least $1-\delta$, Algorithm~\ref{algo:uniform} outputs $\widehat{\f{\mu}}_j$ such that $\norm{\widehat{\f{\mu}}_j-\f{\mu}_j}_2 \le \epsilon$, for all $j\in[k]$, if $m \ge \p{\frac{\s{R}k^{\alpha}}{\epsilon \beta^{\alpha}}}^4 c \log \frac{k}{\delta}$, for some $c>0$.
\end{lemma}
\begin{proof}
Using \eqref{eqn:beta} and H\"{o}lder's inequality, for any $j \in [k]$ we have,
\begin{align}
&\p{\sum_{i \in [n]} \s{U}_{ij}^{\alpha}}^{1/\alpha}\p{\sum_{i \in [n]}  1^{\alpha/(\alpha-1)}}^{(\alpha-1)/\alpha} \ge \sum_{i \in [n]} \s{U}_{ij} \ge \frac{\beta n}{k},
\end{align}
which implies that
\begin{align}
\p{\sum_{i \in [n]} \s{U}_{ij}^{\alpha}}^{1/\alpha} \ge \frac{\beta n}{kn^{(\alpha-1)/\alpha}}=\frac{\beta n^{1/\alpha}}{k},
\end{align}
and thus
\begin{align}
\sum_{i \in [n]}  \s{U}_{ij}^{\alpha} \ge \frac{n\beta^{\alpha}}{k^{\alpha}}.    
\end{align}
Therefore,
\begin{align*}
\s{Y} = \min_{j \in [k]} \frac{\sum_{i \in [n]} \s{U}_{ij}^{\alpha}}{n} \ge \Big(\frac{\beta}{k}\Big)^{\alpha}
\end{align*}
Now, using Lemma~\ref{lem:mean1} and the last result, taking a union bound over all $j \in [k]$, we get 
\begin{align}
\norm{\widehat{\f{\mu}}_j-\f{\mu}_j}_2 \le \frac{2\s{R}k^{\alpha}}{\beta^{\alpha} }\Big(\frac{c}{m}\log\frac{1}{\delta'}\Big)^{1/4} \le \epsilon.
\end{align}
with probability $1-k\delta'$. Rearranging terms and substituting $\delta=k\delta'$, we get the proof of the lemma.
\end{proof}

\subsubsection{Proof of Theorem~\ref{thm:1}}

We are now in a position to prove Theorem~\ref{thm:1}. Using Lemma~\ref{lem:mean_together}, we can conclude that by taking $m \ge \p{\frac{\s{R}k^{\alpha}}{\epsilon \beta^{\alpha}}}^4c\log \frac{k}{\delta}$ in Algorithm~\ref{algo:estimate_together}, which would require $km$ membership-oracle queries, we get $\norm{\widehat{\f{\mu}}_j - \f{\mu}_j}_2\le\epsilon$, for all $j\in[k]$. The time-complexity required to estimate all these means is of order $O(kdm)$. Furthermore, using Lemma~\ref{lem:membership}, using $O(\log n/\eta)$ membership-oracle queries, Algorithm~\ref{algo:estimate_together} outputs membership estimates such that \eqref{eqn:membershipEstLemma} holds. This requires a time-complexity of order $O(\log n/\eta)$. We note, however, that the membership $\{\widehat{\s{U}}_{ij}\}_{j=1}^k$, for any $i\in[n]$, may not sum up to unity, which is an invalid solution. To fix that in step 7 of Algorithm~\ref{algo:estimate_together} we add to each $\widehat{\s{U}}_{ij}$ a factor of $\frac{1-\sum_{j=1}^k\widehat{\s{U}}_{ij}}{k}$, and then it is clear that the new estimated membership weights sum up to unity. Furthermore, these updated membership weights satisfy
$|\widehat{\s{U}}_{ij}-\s{U}_{ij}|\le\eta$, for all $i\in[n]$ and $j\in [k]$. Therefore, we have shown that Algorithm~\ref{algo:estimate_together} is $(\epsilon,\eta,\s{Q})$-solver with probability at least $1-\delta$, which concludes the proof.

\subsection{Proof of Theorem~\ref{thm:2}}\label{sec:theorem2}

In this subsection, we prove Theorem~\ref{thm:2} using induction. 

\subsubsection{Base Case} As can be seen from Algorithm~\ref{algo:estimate_sequentially}, in the first step of this algorithm we sample $m$ indices uniformly at random and obtain the multiset $\ca{S} \subseteq \ca{X}$. Subsequently, we query $\s{U}_{ij}$ for all $i\in\ca{S}$ and $j\in[k]$, and then, in the third step of the algorithm we choose the cluster $t_1$ with the highest \textit{membership} value, namely, 
\begin{align}
t_1 = \underset{j \in [k]}{\s{argmax}}\;\sum_{i \in \ca{S}} \s{U}_{ij}^{\alpha}.
\end{align}
Then, in the fourth step of this algorithm we estimate the mean of this cluster by
\begin{align}
\f{\widehat{\mu}}_{t_1} \triangleq \frac{\sum_{i \in \ca{S}} \s{U}_{ij}^{\alpha} \f{x}_i}{\sum_{i \in \ca{S}} \s{U}_{ij}^{\alpha}}. \label{eq:firstmean}
\end{align}
We have the following lemma, which is similar to Lemma~\ref{lem:mean1}.
\begin{lemma}[Guarantees on the largest cluster]{\label{lem:largest}}
Let $(\calX,\calP)$ be a consistent center-based clustering instance, and let $\delta\in(0,1)$. With probability at least $1-\delta/k$, the estimator in \eqref{eq:firstmean} satisfies
\begin{align}
\norm{\widehat{\f{\mu}}_{t_1}-\f{\mu}_{t_1}}_2 \le  \frac{2\s{R}\p{\frac{c}{m}\log\frac{2k}{\delta}}^{1/4}}{\frac{1}{k^{\alpha}}-\sqrt{\frac{c}{2m}\log\frac{2k}{\delta}}},
\end{align}
where $c>0$ is an absolute constant.
\end{lemma}
\begin{proof}[Proof of Lemma~\ref{lem:largest}]
Recall that Lemma~\ref{lem:mean1} tells us that with probability at least $1-\delta/(2k)$,
\begin{align}
\norm{\widehat{\f{\mu}}_{t_1}-\f{\mu}_{t_1}}_2 \le \frac{2\s{R}}{\s{Y}}\p{\frac{c}{m}\log\frac{2k}{\delta}}^{1/4},
\end{align}
where $\s{Y} = (1/n)\sum_{i \in [n]} \s{U}_{it_1}^{\alpha}$. Now, since $t_1$ is chosen as the cluster with the maximum membership in the subset $\ca{S}$, we will first bound $\widehat{\s{Y}} \triangleq (1/m)\sum_{i \in \ca{S}} \s{U}_{it_1}^{\alpha}$. Notice that $\sum_{i \in \ca{S}} \sum_{j=1}^{k} \s{U}_{ij} =m$, and therefore, using H\"{o}lder's inequality we have that for $\alpha>1$,
\begin{align}
& \p{\sum_{i \in \ca{S}} \sum_{j=1}^{k} \s{U}_{ij}^{\alpha}}^{1/\alpha}\p{\sum_{i \in \ca{S}} \sum_{j=1}^{k} 1^{\alpha/(\alpha-1)}}^{(\alpha-1)/\alpha} \ge \sum_{i \in \ca{S}} \sum_{j=1}^{k} \s{U}_{ij} \ge m,
\end{align}
which implies that
\begin{align}
\p{\sum_{i \in \ca{S}} \sum_{j=1}^{k} \s{U}_{ij}^{\alpha}}^{1/\alpha} \ge \frac{m}{(km)^{(\alpha-1)/\alpha}}=\frac{m^{1/\alpha}}{k^{(\alpha-1)/\alpha}},
\end{align}
and therefore,
\begin{align}
\sum_{i \in \ca{S}} \sum_{j=1}^{k} \s{U}_{ij}^{\alpha} \ge \frac{m}{k^{\alpha-1}}.
\end{align}
Accordingly, we must have $\sum_{i \in \ca{S}} \s{U}_{it_1}  \ge \frac{m}{k^{\alpha}}$ which in turn implies that $\widehat{\s{Y}} \ge \frac{1}{k^{\alpha}}$. Next, using Hoeffding's inequality in Lemma~\ref{lem:hoeffding} we obtain that $|\s{Y} -\widehat{\s{Y}}|\leq\sqrt{\frac{c}{2m}\log\frac{2k}{\delta}}$, with probability at least $1-\delta/(2k)$, and therefore $\s{Y} \ge \widehat{\s{Y}}-\sqrt{\frac{c}{2m}\log\frac{2k}{\delta}}$, which concludes the proof.
\end{proof}
Using the above result and Lemma~\ref{lem:membership} we obtain the following corollaries.
\begin{corollary}{\label{coro:base1}}
Let $(\calX,\calP)$ be a consistent center-based clustering instance, and let $\delta\in(0,1)$. Then, with probability at least $1-\delta/k$, the estimator in \eqref{eq:firstmean} satisfies $\norm{\widehat{\f{\mu}}_{t_1}-\f{\mu}_{t_1}}_2 \le \epsilon$, 
if $m \ge \frac{c\s{R}^4k^{4\alpha}}{\epsilon^4}\log\frac{2k}{\delta}$. Also, this estimate requires $O\p{\frac{\s{R}^4k^{4\alpha+1}}{\epsilon^4}\log\frac{2k}{\delta}}$ membership-oracle queries, and a time-complexity of $O\p{\frac{d\s{R}^4k^{4\alpha}}{\epsilon^4}\log\frac{2k}{\delta}}$.
\end{corollary}
\begin{proof}[Proof of Corollary~\ref{coro:base1}]
The proof follows from rearranging terms of Lemma \ref{lem:largest}. The query complexity follows from the fact that we query the membership values $\s{U}_{ij}$ for all the $i \in \ca{S},j \in [k]$ and the time-complexity follows from the fact that we take the mean of $m$ $d$-dimensional vectors in order to return the estimator $\widehat{\f{\mu}}_{t_1}$.
\end{proof}
\begin{corollary}{\label{coro:base2}}
Let $(\calX,\calP)$ be a consistent center-based clustering instance, and recall the definition of $\gamma\in\mathbb{R}_+$ in \eqref{eqn:equiv}. Assume that there exists an estimator $\widehat{\f{\mu}}_{t_1}$ such that $\norm{\f{\mu}_{t_1}-\widehat{\f{\mu}}_{t_1}}_2 \le \epsilon$ with $\epsilon\leq\gamma$. Then, Algorithm~\ref{algo:estimate_sequentially} outputs $\widehat{\s{U}}_{it_1}$, for $i\in[n]$, such that
\begin{align}
&0 \le \s{U}_{it_1}-\widehat{\s{U}}_{it_1} \le \eta_1,\quad \widehat{\s{U}}_{it_1} \in \{0,\eta_1,2\eta_1,\dots,1\}, \quad \forall i \in [n],\label{eqn:membershipEstLemma2}
\end{align}
for some $\eta\in\mathbb{R}_+$, using $\s{Q} = O\p{\log n/\eta_1}$ queries to the membership-oracle. 
\end{corollary}
\begin{proof}[Proof of Corollary~\ref{coro:base2}]
The proof of this lemma follows the same steps as in the proof of Lemma~\ref{lem:membership}.
\end{proof}
Corollaries~\ref{coro:base1} and \ref{coro:base2} show that the base case of our induction is correct. 

\subsubsection{Induction Hypothesis} We condition on the event that we have been able to estimate $\f{\mu}_{t_1},\f{\mu}_{t_{2}},\dots,\f{\mu}_{t_{\ell}}$ by their corresponding estimators $\widehat{\f{\mu}}_{t_1},\widehat{\f{\mu}}_{t_{2}},\dots,\widehat{\f{\mu}}_{t_{\ell}}$, respectively, such that 
\begin{align}
\norm{\widehat{\f{\mu}}_{t_j}-\f{\mu}_{t_j}}_2 \le \epsilon, \quad \forall \; j \in [\ell],
\end{align}
and further, we have been able to recover ${\s{U}}_{it_j}$, for all $i\in [n]$ and $j\in[\ell]$, in the sense that
\begin{align}
0 \le \s{U}_{it_j}-\widehat{\s{U}}_{it_j}  \le \eta_1, \quad \widehat{\s{U}}_{it_j} \in \{0,\eta_1,2\eta_1,\dots,1\} \quad, \forall \; i \in [n], j \in [\ell].\label{eqn:inducHypothesisMem}
\end{align}
The induction hypothesis states that we have been able to estimate the means of $\ell$ clusters up to an error of $\epsilon$ and subsequently also estimated the memberships of every element in $\ca{X}$ to those $\ell$ clusters such that the estimated memberships are an integral multiple of $\eta_1$ and also have a precision error of at most $\eta_1$. Given the induction hypothesis, we characterize next the sufficient query complexity and time-complexity required in order to estimate the mean of the $(\ell+1)^{\s{th}}$ cluster and its membership weights. 

\subsubsection{Inductive Step} Let $\s{Z}_{\ell} \triangleq \sum_{i \in [n]}\sum_{j \in [\ell]} \s{U}_{it_j}$, and define $\ca{X}_s \triangleq \{i \in [n]: \sum_{j \in [\ell]} \widehat{\s{U}}_{it_j} = s\eta_1 \}$, for $s\in\{0,1,\ldots,1/\eta_1\}$. In step 10 of Algorithm~\ref{algo:estimate_sequentially}, we sub-sample $r$ indices uniformly at random with replacement from each of the sets $\ca{X}_s$, for $s \in \{0,1,2,\dots,1/\eta_1\}$ . Let us denote the multi-set of indices chosen from $\ca{X}_s$ by $\ca{Y}_s$. Subsequently, in the step 12  of Algorithm~\ref{algo:estimate_sequentially}, for every $s \in \{0,1,2,\dots,1/\eta_1\}$ and for every element in $\ca{Y}_s$, we query the memberships to all the clusters except $t_1,t_2,\dots,t_{\ell}$ from the oracle, and set
\begin{align}
t_{\ell+1} = \underset{j \in [k]\setminus \{t_1,t_{2},\dots,t_{\ell}\}}{\s{argmax}}\; \sum_s\frac{|\ca{X}_s|}{r}\sum_{i \in \ca{Y}_s} \s{U}_{ij}^{\alpha}.
\end{align}
Step 13 of Algorithm~\ref{algo:estimate_sequentially} computes
\begin{align}
\widehat{\f{\mu}}^{t_{\ell+1}} \triangleq \frac{\sum_s\frac{|\ca{X}_s|}{r}\sum_{i \in \ca{Y}_s} \s{U}_{it_{\ell+1}}^{\alpha} \f{x}_i}{\sum_s \frac{|\ca{X}_s|}{r}\sum_{i \in \ca{Y}_s} \s{U}_{it_{\ell+1}}^{\alpha}}.
\label{eq:othermean}
\end{align}
The following analysis the performance of the estimator in \eqref{eq:othermean}. We relegate the proof of this result to the end of this section.
\begin{lemma}[Performance of \eqref{eq:othermean}]{\label{lem:inductive}}
Let $(\calX,\calP)$ be a consistent center-based clustering instance, and let $\delta\in(0,1)$. With probability at least $1-\delta/k$, the estimator in \eqref{eq:othermean} satisfies
\begin{align}
\norm{\widehat{\f{\mu}}_{t_{\ell+1}}-\f{\mu}_{t_{\ell+1}}}_2 
\le \frac{2\s{R}\p{\frac{c}{r}\log\frac{4k}{\eta_1\delta}}^{1/4}\Big(n-\s{Z}_{\ell}+n \ell \eta_1\Big)}{\frac{\Big(n- \s{Z}_{\ell}-n \ell \eta_1 \sqrt{\frac{c}{2r}\log\frac{4k}{\eta_1\delta}}\Big)^{\alpha}}{n^{\alpha-1}(k-\ell)^{\alpha}}-(n-\s{Z}_{\ell}-n \ell \eta_1)\sqrt{\frac{c}{2r}\log\frac{4k}{\eta_1\delta}}},
\end{align}
where $c>0$ is an absolute constant. The query and time-complexity required for evaluating this estimator are of order $O(kr/\eta_1)$ and $O(rd/\eta_1)$, respectively.
\end{lemma}
Using the above result and Lemma~\ref{lem:membership} we obtain the following corollaries.
\begin{corollary}{\label{coro:ind1}}
Let $(\calX,\calP)$ be a consistent center-based clustering instance, recall the definition of $\beta\in(0,1)$ in \eqref{eqn:beta}, and let $\delta\in(0,1)$. Then, with probability at least $1-\delta/k$, the estimator in \eqref{eq:othermean} satisfies $\norm{\widehat{\f{\mu}}_{t_{\ell+1}}-\f{\mu}_{t_{\ell+1}}}_2 \le \epsilon$, if $r \ge \frac{c'\s{R}^4k^{4\alpha}}{\epsilon^4\beta^{4\alpha-4}}\log\frac{4k}{\eta_1\delta}$ and $\eta_1 \le \frac{1}{k}\p{1-\frac{\beta}{k}}$. Also, this estimate requires $O\p{\frac{\s{R}^4k^{4\alpha+1}}{\epsilon^4\beta^{4\alpha-4}}\log\frac{4k}{\eta_1\delta}}$ membership-oracle queries, and a time-complexity of $O\p{\frac{\s{R}^4k^{4\alpha}}{\epsilon^4\beta^{4\alpha-4}}\log\frac{4k}{\eta_1\delta}}$.
\end{corollary}
\begin{proof}[Proof of Corollary~\ref{coro:ind1}]
Using \eqref{eqn:beta} and the fact that $n-\s{Z}_{\ell} = \sum_{i \in [n]}\sum_{j \in [k]\setminus[\ell]} \s{U}_{it_j}$, we have the following upper and lower bound
\begin{align}
\frac{n-\s{Z}_{\ell}}{k-\ell} \ge \frac{\beta n}{k} \quad \text{and} \quad n-\s{Z}_{\ell} \le n-\frac{\beta n}{k}, 
\end{align}
which follows from the fact that the average \textit{membership size} of the any $k-\ell$ clusters must be larger than the membership size of the smallest cluster. Thus, if $\eta_1 \le \frac{1}{k}\p{1-\frac{\beta}{k}}$, as claimed in the statement of the lemma, we have $n\ell \eta_1 \le n-\s{Z}_{\ell}$. With the chosen values of $\eta_1$ and $r$, and the fact that $\ell<k$, we get $n \ell \eta_1\sqrt{\frac{c}{2r}\log\frac{4k}{\eta_1\delta}} = o(n-\s{Z}_{\ell})$. Therefore,  
\begin{align}
\norm{\widehat{\f{\mu}}_{t_{\ell+1}}-\f{\mu}_{t_{\ell+1}}}_2 
&\le \frac{4\s{R}\p{\frac{c}{r}\log\frac{4k}{\eta_1\delta}}^{1/4}(n-\s{Z}_{\ell})}{\frac{(n- \s{Z}_{\ell})^{\alpha}}{n^{\alpha-1}(k-\ell)^{\alpha}}-2\sqrt{\frac{c}{2r}\log\frac{4k}{\eta_1\delta}}(n-\s{Z}_{\ell})} \\
& = \frac{4\s{R}\p{\frac{c}{r}\log\frac{4k}{\eta_1\delta}}^{1/4}}{\frac{(n- \s{Z}_{\ell})^{\alpha-1}}{n^{\alpha-1}(k-\ell)^{\alpha}}-2\sqrt{\frac{c}{2r}\log\frac{4k}{\eta_1\delta}}}\\
&\le \frac{4\s{R}\p{\frac{c}{r}\log\frac{4k}{\eta_1\delta}}^{1/4}}{\frac{\beta^{\alpha-1}}{k^{\alpha}(k-\ell)}-2\sqrt{\frac{c}{2r}\log\frac{4k}{\eta_1\delta}}}.
\end{align}
Again, for the chosen value of $r$, it is clear that $\sqrt{\frac{c}{2r}\log\frac{4k}{\eta_1\delta}}=o\Big(\frac{\beta^{\alpha-1}}{k^{\alpha}(k-\ell)}\Big)$, and thus we get that $\norm{\f{\widehat{\mu}}_{t_{\ell+1}}-\f{\mu}_{t_{\ell+1}}}_2 \le \epsilon$, with probability at least $1-\delta/k$.
\end{proof}

\begin{corollary}{\label{coro:ind2}}
Let $(\calX,\calP)$ be a consistent center-based clustering instance, and recall the definition of $\gamma\in\mathbb{R}_+$ in \eqref{eqn:equiv}. Assume that there exists an estimator $\widehat{\f{\mu}}_{t_{\ell+1}}$ such that $\norm{\f{\mu}_{t_{\ell+1}}-\widehat{\f{\mu}}_{t_{\ell+1}}}_2 \le \epsilon$ with $\epsilon\leq\gamma$. Then, Algorithm~\ref{algo:estimate_sequentially} outputs $\widehat{\s{U}}_{it_{\ell+1}}$, for $i\in[n]$, such that
\begin{align}
&0 \le \s{U}_{it_{\ell+1}}-\widehat{\s{U}}_{it_{\ell+1}} \le \eta_1,\quad \widehat{\s{U}}_{it_{\ell+1}} \in \{0,\eta_1,2\eta_1,\dots,1\}, \quad \forall i \in [n],\label{eqn:membershipEstLemma3}
\end{align}
for some $\eta\in\mathbb{R}_+$, using $\s{Q} = O\p{\log n/\eta_1}$ queries to the membership-oracle. 
\end{corollary}
\begin{proof}[Proof of Corollary~\ref{coro:ind2}]
The proof of this lemma follows the same steps as in the proof of Lemma~\ref{lem:membership}.
\end{proof}

\subsubsection{Proof of Theorem~\ref{thm:2}} We are now in a position to proof Theorem~\ref{thm:2}. To that end, we use our induction mechanism. Specifically, Corollaries \ref{coro:base1} and \ref{coro:base2} prove the base case for the first cluster $t_1$. Subsequently, Corollaries \ref{coro:ind1} and \ref{coro:ind2} prove the induction step after taking a union bound over all clusters and using $\eta_1=\frac{1}{k}\Big(1-\frac{\beta}{k}\Big)$. Finally, we can use Lemma~\ref{lem:membership} in order to estimate the memberships $\s{U}_{ij} \; \forall \; i\in [n], j \in [k]$ up to a precision of $\eta_2$ using an addition query and time-complexity of $O(k\log n/\eta_2)$. 

It is left to prove Lemma~\ref{lem:inductive}.
\begin{proof}[Proof of Lemma~\ref{lem:inductive}]
Let $\Sigma\triangleq\{0,1,\ldots,1/\eta_1\}$. We have
\begin{align}
\f{\mu}_{t_{\ell+1}} = \frac{\sum_{i \in [n]} \s{U}_{it_{\ell+1}}^{\alpha} \f{x}_i}{\sum_{i \in [n]} \s{U}_{it_{\ell+1}}^{\alpha}} = \frac{\sum_{s\in\Sigma}\sum_{i \in \ca{X}_s} \s{U}_{it_{\ell+1}}^{\alpha} \f{x}_i}{\sum_{s\in\Sigma} \sum_{i \in \ca{X}_s} \s{U}_{it_{\ell+1}}^{\alpha}}  \triangleq \frac{\sum_{s\in\Sigma} \f{\lambda}_s}{\sum_{s\in\Sigma} \s{Y}_s},
\end{align}
and that
\begin{align}
\widehat{\f{\mu}}_{t_{\ell+1}} = \frac{\sum_{s\in\Sigma}\frac{|\ca{X}_s|}{r}\sum_{i \in \ca{Y}_s} \s{U}_{it_{\ell+1}}^{\alpha} \f{x}_i}{\sum_{s\in\Sigma} \frac{|\ca{X}_s|}{r}\sum_{i \in \ca{Y}_s} \s{U}_{it_{\ell+1}}^{\alpha}} \triangleq  \frac{\sum_{s\in\Sigma} \bar{\f{\lambda}}_s}{\sum_{s\in\Sigma} \bar{\s{Y}}_s}.\label{eqn:esttlplu}
\end{align} 
Now, note that we can write $\bar{\f{\lambda}}_s$ as an average of $r$ i.i.d random variables $\bar{\f{x}}_{s,i_p}\triangleq \left|\ca{X}_s\right|\s{U}_{i_pt_{\ell+1}}^{\alpha} \f{x}_{i_p}$, where $i_p$ is sampled uniformly at random from $[n]$, and included to the set $\ca{Y}_s$ in the step 9 of Algorithm~\ref{algo:estimate_sequentially} as the $p^{\s{th}}$ sample. Similarly, $\bar{\s{Y}}_s$ can also be written as the average of $r$ i.i.d random variables $\bar{\s{Y}}_{s,i_{p}}=\left|\ca{X}_s \right | \s{U}_{i_pt_{\ell+1}}^{\alpha}$. Therefore, it is evident that $\bb{E} \bar{\f{x}}_{s,i_p} = \f{\lambda}_s$ and $\bb{E} \bar{\s{Y}}_{s,i_{p}} = \s{Y}_s$ for all $p \in[r]$. This implies that the numerator and denominator of \eqref{eqn:esttlplu} are both unbiased. Now, note that
\begin{align}
\widehat{\f{\mu}}_{t_{\ell+1}} = \frac{\sum_{s\in\Sigma}\bar{\f{\lambda}}_s}{\sum_{s\in\Sigma}\bar{\s{Y}}_s} =  \frac{\sum_{s\in\Sigma} \f{\lambda}_s}{\sum_{s\in\Sigma} \s{Y}_s}+\frac{\sum_{s\in\Sigma} \bar{\f{\lambda}}_s-\f{\lambda}_s} {\sum_{s\in\Sigma} \s{Y}_s}+\frac{\sum_{s\in\Sigma} \bar{\f{\lambda}}_s}{\sum_{s\in\Sigma} \bar{\s{Y}}_s}\frac{\sum_{s\in\Sigma} \s{Y}_s-\bar{\s{Y}}_s}{\sum_{s\in\Sigma} \s{Y}_s}.
\end{align}
Thus, using the triangle inequality we get
\begin{align}
\norm{\widehat{\f{\mu}}_{t_{\ell+1}}-\f{\mu}_{t_{\ell+1}}}_2 \le \sum_{s\in\Sigma}\norm{\frac{\bar{\f{\lambda}}_s-\f{\lambda}_s} {\sum_{s\in\Sigma} \s{Y}_s}}_2+\s{R}\sum_{s\in\Sigma}\norm{\frac{ \s{Y}_s-\bar{\s{Y}}_s}{\sum_{s\in\Sigma} \s{Y}_s}}_2,\label{eqn:YsGaran}
\end{align}
where the last inequality follows from the fact that $\norm{\widehat{\f{\mu}}_{t_{\ell+1}}}_2 = \norm{\frac{\sum_{s\in\Sigma} \bar{\f{\lambda}}_s}{\sum_{s\in\Sigma} \bar{\s{Y}}_s}}_2\leq\s{R}$.

Next, we note that for $i \in \ca{Y}_s$, we have
\begin{align}
\s{U}_{it_{\ell+1}}^{\alpha} \le (1-\sum_{j \in [\ell]}\s{U}_{it_j})^{\alpha} \le (1-\sum_{j \in [\ell]}\widehat{\s{U}}_{it_j})^{\alpha} =  (1-s\eta_1)^{\alpha},    
\end{align}
where we have used the induction hypothesis in \eqref{eqn:inducHypothesisMem}. 
Also, using Lemmas~\ref{lem:hoeffding} and \ref{lem:gen_hoeffding}, for all $s\in\Sigma$, with probability at least $1-\delta/2k$, 
\begin{align}
\norm{\bar{\f{\lambda}}_s-\f{\lambda}_s}_2 &\le \s{R}|\ca{X}_s|(1-s\eta_1)^{\alpha}  \p{\frac{c}{r}\log\frac{4k}{\eta_1\delta}}^{1/4}\nonumber\\
&\le \s{R}|\ca{X}_s|(1-s\eta_1) \p{\frac{c}{r}\log\frac{4k}{\eta_1\delta}}^{1/4},\label{eqn:ConcentrationLambda}
\end{align}
and
\begin{align}
\abs{\bar{\s{Y}}_s-\s{Y}_s} &\le |\ca{X}_s|(1-s\eta_1)^{\alpha}\sqrt{\frac{c}{r}\log\frac{4k}{\eta_1\delta}}\nonumber\\
&\le |\ca{X}_s|(1-s\eta_1) \sqrt{\frac{c}{2r}\log\frac{4k}{\eta_1\delta}}.\label{eqn:ConcentrationY}
\end{align}
Using the induction hypothesis in \eqref{eqn:inducHypothesisMem} once again, we have
\begin{align}
\s{Z}_{\ell} = \sum_{s\in\Sigma} \sum_{i \in \ca{X}_s} \sum_{j \in [\ell]}\s{U}_{it_j} \le \sum_{s\in\Sigma} \sum_{i \in \ca{X}_s}\Big( \ell \eta_1+\sum_{j \in [\ell]}\widehat{\s{U}}_{it_j}\Big) \le \sum_{s\in\Sigma}\left|\ca{X}_s\right|s \eta_1 + n \ell \eta_1,
\end{align}
and thus
\begin{align}
\sum_{s\in\Sigma} |\ca{X}_s|(1-s\eta_1) \le n-\s{Z}_{\ell}+n \ell \eta_1.\label{eq:imp}
\end{align}
Next, we lower bound $\sum_{s\in\Sigma} \s{Y}_s$. To that end, in light of \eqref{eqn:YsGaran}, it is suffice to bound $\sum_{s\in\Sigma} \bar{\s{Y}}_s$. Using H\"{o}lder's inequality we have
\begin{align}
&\Big(\sum_{i \in \ca{Y}_s} \sum_{j \in [k]\setminus \{t_1,t_2,\dots,t_{\ell}\}} \s{U}_{ij}^{\alpha}\Big)^{1/\alpha}\Big(\sum_{i \in \ca{Y}_s} \sum_{j \in [k]\setminus \{t_1,t_2,\dots,t_{\ell}\}} 1^{\alpha/(\alpha-1)}\Big)^{(\alpha-1)/\alpha}\nonumber\\
&\hspace{3cm}\ge \sum_{i \in \ca{Y}_s} \sum_{j \in [k]\setminus \{t_1,t_2,\dots,t_{\ell}\}} \s{U}_{ij},
\end{align}
which implies that
\begin{align}
\Big(\sum_{i \in \ca{Y}_s} \sum_{j \in [k]\setminus \{t_1,t_2,\dots,t_{\ell}\}} \s{U}_{ij}^{\alpha}\Big)^{1/\alpha} \ge \frac{\sum_{i \in \ca{Y}_s} \sum_{j \in [k]\setminus\{t_1,t_2,\dots,t_{\ell}\}} \s{U}_{ij}}{(r(k-\ell))^{(\alpha-1)/\alpha}},
\end{align}
and therefore,
\begin{align}
\sum_{i \in \ca{Y}_s} \sum_{j \in [k]\setminus \{t_1,t_2,\dots,t_{\ell}\}} \s{U}_{ij}^{\alpha} \ge \frac{\Big(\sum_{i \in \ca{Y}_s} \sum_{j \in [k]\setminus\{t_1,t_2,\dots,t_{\ell}\}} \s{U}_{ij}\Big)^{\alpha}} {(r(k-\ell))^{(\alpha-1)}}.
\label{eq:holder}
\end{align}
To further lower bound the r.h.s. of \eqref{eq:holder}, we use the power means inequality and ger
\begin{align}
\frac{\sum_{s\in\Sigma}\frac{|\ca{X}_s|}{r} \Big(\sum_{i \in \ca{Y}_s} \sum_{j \in [k]\setminus \{t_1,t_2,\dots,t_{\ell}\}} \s{U}_{ij}\Big)^{\alpha}}{\sum_{s'\in\Sigma}\frac{|\ca{X}_{s'}|}{r}} \ge \Bigg(\frac{\sum_{s\in\Sigma}\frac{|\ca{X}_{s}|}{r} \sum_{i \in \ca{Y}_s} \sum_{j \in [k]\setminus \{t_1,t_2,\dots,t_{\ell}\}} \s{U}_{ij}}{\sum_{s'\in\Sigma}\frac{|\ca{X}_{s'}|}{r}} \Bigg)^{\alpha}.
\label{eq:power_means}
\end{align}
Thus, the fact that $\sum_{s\in\Sigma}\frac{|\ca{X}_s|}{r} =\frac{n}{r}$, combined with \eqref{eq:holder} and \eqref{eq:power_means}, imply that
\begin{align}
\sum_{s\in\Sigma} \frac{\left|\ca{X}_s \right|}{r}\sum_{i \in \ca{Y}_s} \sum_{j \in [k]\setminus \{t_1,t_2,\dots,t_{\ell}\}} \s{U}_{ij}^{\alpha} &\ge \frac{\Bigg(\sum_{s\in\Sigma}\frac{|\ca{X}_s|}{r} \sum_{i \in \ca{Y}_s} \sum_{j \in [k]\setminus \{t_1,t_2,\dots,t_{\ell}\}} \s{U}_{ij} \Bigg)^{\alpha}}{(n(k-\ell))^{(\alpha-1)}} \\
& = \frac{\Bigg(n-\sum_{s\in\Sigma}\frac{|\ca{X}_s|}{r} \sum_{i \in \ca{Y}_s} \sum_{j \in \{t_1,t_2,\dots,t_{\ell}\}} \s{U}_{ij} \Bigg)^{\alpha}}{(n(k-\ell))^{(\alpha-1)}}.\label{eqn:lowerHold}
\end{align}
We next upper bound the term inside the brackets at the r.h.s. of \eqref{eqn:lowerHold}. To that end, for a given $s$, we define the random variables
\begin{align}
\s{H}_{s,i_p}\triangleq \left|\ca{X}_s \right| \left(\sum_{j \in \{t_1,t_2,\dots,t_{\ell}\}} \s{U}_{i_pj}-s\eta_1 \right),
\end{align}
where where $i_p$ is sampled uniformly at random from $[n]$, and included to the set $\ca{Y}_s$ in the step 9 of Algorithm~\ref{algo:estimate_sequentially} as the $p^{\s{th}}$ sample.  With this definition, it is evident that $\frac{|\ca{X}_s|}{r} \sum_{i \in \ca{Y}_s} (\sum_{j \in \{t_1,t_2,\dots,t_{\ell}\}} \s{U}_{ij}-s\eta_1)$ can be written as the average of these $r$ i.i.d random variables, namely,
\begin{align}
\frac{|\ca{X}_s|}{r} \sum_{i \in \ca{Y}_s} \left(\sum_{j \in \{t_1,t_2,\dots,t_{\ell}\}} \s{U}_{ij}-s\eta_1 \right) =\frac{1}{r}\sum_{i\in\calY_s}\s{H}_{s,i}.
\end{align}
Note that $\bb{E} \s{H}_{s,i_p} = \sum_{i \in \ca{X}_s}\sum_{j \in \{t_1,t_2,\dots,t_{\ell}\}} \s{U}_{ij}-|\ca{X}_s|s \eta_1$ and $\left|\s{H}_{s,i_p} \right| \le \ell \eta_1$ for all $i_p\in [r]$. For simplicity of notation we define $\s{Z}_{s\ell}\triangleq \sum_{i \in \ca{X}_s}\sum_{j \in \{t_1,t_2,\dots,t_{\ell}\}} \s{U}_{ij}$. Then, using Hoeffding's inequality in Lemma~\ref{lem:hoeffding}, we have with probability at least $1-\delta/4k$, 
\begin{align}
&\left|\frac{|\ca{X}_s|}{r} \sum_{i \in \ca{Y}_s}\left(\sum_{j \in \{t_1,t_2,\dots,t_{\ell}\}} \s{U}_{ij}-s\eta_1 \right)-\s{Z}_{s\ell} +|\ca{X}_s|s\eta_1 \right|\nonumber\\
&\hspace{4cm}\le \left| \ca{X}_s \right| \ell \eta_1 \sqrt{\frac{c}{2r}\log\frac{4k}{\eta_1\delta}}, 
\end{align}
and thus, 
\begin{align}
\frac{|\ca{X}_s|}{r}\sum_{i \in \ca{Y}_s}\sum_{j \in \{t_1,t_2,\dots,t_{\ell}\}} \s{U}_{ij} &\leq \max\p{\left|\ca{X}_s \right|,\s{Z}_{s\ell}+\left| \ca{X}_s \right| \ell \eta_1\sqrt{\frac{c}{2r}\log\frac{4k}{\eta_1\delta}}}\\
&=\s{Z}_{s\ell}+\max\p{\left|\ca{X}_s \right|-\s{Z}_{s\ell}, \left| \ca{X}_s \right| \ell \eta_1\sqrt{\frac{c}{2r}\log\frac{4k}{\eta_1\delta}}}.\label{eqn:middleBound}
\end{align}
Summing \eqref{eqn:middleBound} over $s\in\Sigma$ and using the fact that $\sum_{s\in\Sigma}\s{Z}_{s\ell} = \s{Z}_{\ell}$, we obtain
\begin{align}
&n-\sum_{s\in\Sigma}\frac{|\ca{X}_s|}{r} \sum_{i \in \ca{Y}_s} \sum_{j \in \{t_1,t_2,\dots,t_{\ell}\}} \s{U}_{ij}\geq n- \s{Z}_{\ell}-\min\p{n-\s{Z}_{\ell}, n \ell \eta_1 \sqrt{\frac{c}{2r}\log\frac{4k}{\eta_1\delta}}}.
\end{align}
Substituting the last inequality in \eqref{eqn:lowerHold}, we finally get
\begin{align}
\sum_{s\in\Sigma} \frac{\left|\ca{X}_s \right|}{r}\sum_{i \in \ca{Y}_s} \sum_{j \in [k]\setminus \{t_1,t_2,\dots,t_{\ell}\}} \s{U}_{ij}^{\alpha} \ge \frac{\Big(n- \s{Z}_{\ell}-n \ell \eta_1 \sqrt{\frac{c}{2r}\log\frac{4k}{\eta_1\delta}}\Big)^{\alpha}}{(n(k-\ell))^{(\alpha-1)}}.
\end{align}
Next, recall that the index of the $(\ell+1)^{\s{th}}$ cluster is chosen as 
\begin{align}
t_{\ell+1} = \underset{j \in [k]\setminus \{t_1,t_{2},\dots,t_{\ell}\}}{\s{argmax}}\; \sum_{s\in\Sigma}\frac{|\ca{X}_s|}{r}\sum_{i \in \ca{Y}_s} \s{U}_{ij}^{\alpha},
\end{align}
and therefore,
\begin{align}
\sum_{s\in\Sigma} \bar{\s{Y}}_s &= \sum_{s\in\Sigma}\frac{|\ca{X}_s|}{r} \sum_{i \in \ca{Y}_s}\s{U}_{it_{\ell+1}}^{\alpha}\\
&\geq \frac{\Big(n- \s{Z}_{\ell}-n \ell \eta_1 \sqrt{\frac{c}{2r}\log\frac{4k}{\eta_1\delta}}\Big)^{\alpha}}{n^{\alpha-1}(k-\ell)^{\alpha}}.\label{eqn:lowerBoundBarY}
\end{align}
Combining \eqref{eqn:ConcentrationY}, \eqref{eq:imp}, and \eqref{eqn:lowerBoundBarY}, we get a lower bound on  $\sum_{s\in\Sigma}\s{Y}_s$ as follows
\begin{align}
\sum_{s\in\Sigma} \s{Y}_s \ge  \frac{\Big(n- \s{Z}_{\ell}-n \ell \eta_1 \sqrt{\frac{c}{2r}\log\frac{4k}{\eta_1\delta}}\Big)^{\alpha}}{n^{\alpha-1}(k-\ell)^{\alpha}}-(n-\s{Z}_{\ell}-n \ell \eta_1)\sqrt{\frac{c}{2r}\log\frac{4k}{\eta_1\delta}}.\label{eqn:finalLowerBoundY}
\end{align}
Finally, combining \eqref{eqn:YsGaran}, \eqref{eqn:ConcentrationLambda}, and \eqref{eqn:finalLowerBoundY}, we get 
\begin{align}
\norm{\widehat{\f{\mu}}_{t_{\ell+1}}-\f{\mu}_{t_{\ell+1}}}_2 &\le \sum_{s\in\Sigma}\norm{\frac{\bar{\f{\lambda}}_s-\f{\lambda}_s} {\sum_{s\in\Sigma} \s{Y}_s}}_2+\s{R}\sum_{s\in\Sigma}\norm{\frac{ \s{Y}_s-\bar{\s{Y}}_s}{\sum_{s\in\Sigma} \s{Y}_s}}_2 \\
&\le \frac{2\s{R}\p{\frac{c}{r}\log\frac{4k}{\eta_1\delta}}^{1/4}\Big(n-\s{Z}_{\ell}+n \ell \eta_1\Big)}{\frac{\Big(n- \s{Z}_{\ell}-n \ell \eta_1 \sqrt{\frac{c}{2r}\log\frac{4k}{\eta_1\delta}}\Big)^{\alpha}}{n^{\alpha-1}(k-\ell)^{\alpha}}-(n-\s{Z}_{\ell}-n \ell \eta_1)\sqrt{\frac{c}{2r}\log\frac{4k}{\eta_1\delta}}},
\end{align}
with probability at least $1-\delta/k$, which concludes the proof.
\end{proof}

\subsection{Proof of Theorem \ref{thm:3state}}\label{app:algoTh3}

To prove Theorem~\ref{thm:3state} we will establish some preliminary results. 

\subsubsection{Auxiliary Lemmata}

We start with the following result which shows that given a good estimate for the larger cluster among the two, we can approximate the membership weights of the smaller cluster reliably. This is done in Algorithm~\ref{algo:membership2}.
\begin{lemma}{\label{lem:membership2}}
Let $(\calX,\calP)$ be a consistent center-based clustering instance, recall the definition of $\gamma\in\mathbb{R}_+$ in \eqref{eqn:equiv}, and let $k=2$. Assume that there exists an estimator $\widehat{\f{\mu}}_{t_1}$ such that $\norm{\f{\mu}_{t_1}-\widehat{\f{\mu}}_{t_1}}_2 \le \epsilon$ with $\epsilon\leq\gamma$. Then, Algorithm~\ref{algo:membership2} outputs $\widehat{\s{U}}_{it_2}$, for $i\in[n]$, such that
\begin{align}
&\widehat{\s{U}}_{it_2} \in \ca{A} \subset \{\s{U}_{it_2}: i \in [n]\},
\end{align}
with $\left|\ca{A}\right| = O(\log^{2}n)$, and
\begin{align}
&\max_{i \in [n]:\widehat{\s{U}}_{it_2} = x} \s{U}_{it_2} \le \max\left( \frac{\max_{i \in [n]} \s{U}_{it_2}}{n^3},\min_{i \in [n]:\widehat{\s{U}}_{it_2} = x} 2\s{U}_{it_2} \right),
\end{align}
for all $x\in\ca{A}$, and for some $\eta\in\mathbb{R}_+$, using $\s{Q} = O(\log^{2}n)$ queries to the membership-oracle, and time-complexity $O(\log^{2}n)$.
\end{lemma}

\begin{proof}[Proof of Lemma~\ref{lem:membership2}]
First, note that since $\calP$ is a consistent center-based clustering, we have 
\begin{align}
\s{U}_{\sigma_{\f{\mu}_{t_1}}(r_1)t_2} \leq \s{U}_{\sigma_{\f{\mu}_{t_1}}(r_2)t_2} \quad \text{if} \quad r_1<r_2,\; r_1,r_2 \in [n].
\end{align}
Indeed, when the elements of $\calX$ are sorted in ascending order according to their distance from $\f{\mu}_{t_1}$, if $\f{x}_{i}$ is closer to $\f{\mu}_{t_1}$ than it is to $\f{x}_{j}$, then $\s{U}_{it_1}\geq\s{U}_{jt_1}$ and thus $\s{U}_{it_2}\leq\s{U}_{jt_2}$. Also, since $\norm{\f{\mu}_{t_1}- \widehat{\f{\mu}}_{t_1}}_2 \le \epsilon\le\gamma$, using \eqref{eqn:equiv}, this ordering remains the same. Therefore, sorting the elements in $\ca{X}$ in ascending order from $\widehat{\f{\mu}}_{t_1}$ as in the first step of Algorithm~\ref{algo:membership2} gives the same ordering with respect to the true mean.
Now, given $\eta\in\mathbb{R}_+$, we search for the index 
\begin{align}
p'_{\eta} \triangleq \s{argmin}_{j \in [n]} \mathds{1} \left[ \s{U}_{\pi_{\widehat{\f{\mu}}_{t_1}}(j)t_2} \ge \eta \right],
\end{align}
which can be done by using the binary search routine in Algorithm~\ref{algo:binary_search2}, which ask for at most $O(\log n)$ membership-oracle queries. We will do this step for $\eta_1,\eta_2,\dots, \eta_{3\log n}$, as described in Algorithm~\ref{algo:membership2}. The values of $\{\eta_i\}$ are chosen as follows. We initialize $\eta_1=\s{U}_{\pi_{\widehat{\f{\mu}}_{t_1}}(n)t_2}$ and $p_{\eta_1}\triangleq n$, and update the other values of $\eta_i$'s recursively as follows. Let $\calV\triangleq\{1,2,\dots,3\log n\}$. For each $q \in \calV\setminus1$, we first set $\eta_q=\eta_{q-1}/2$ and subsequently, if $|p'_{\eta_q}-p_{\eta_{q-1}}| \ge \log n$, then $\eta_q$ remains unchanged and we set $p_{\eta_q} = p'_{\eta_q}$. Otherwise, if $|p'_{\eta_q}-p_{\eta_{q-1}}|<\log n$, then we update both $\eta_q$ and $p_{\eta_q}$ as follows:
\begin{align}
p_{\eta_q} &= \min(0,p_{\eta_{q-1}}-1-\log n) \\
\eta_q &=  \s{U}_{\pi_{\widehat{\f{\mu}}_{t_1}}(p_{\eta_q})t_2}
\end{align}
For each value of $q \in \calV$, we initialize two sets $\ca{X}_q, \ca{X}'_q = \phi$. If  $|p'_{\eta_q}-p_{\eta_{q-1}}|\ge\log n$, then we update $\ca{X}_q = \{\pi_{\widehat{\f{\mu}}_{t_1}}(i):   p_{\eta_{q}}  \le  i  \le p_{\eta_{q-1}}-1, i\in [n] \}$ and if $|p'_{\eta_q}-p_{\eta_{q-1}}|<\log n$, we update $\ca{X}'_q = \{\pi_{\widehat{\f{\mu}}_{t_1}}(i): p_{\eta_{q}}  \le i \le p_{\eta_{q-1}}-1, i \in [n] \}$. It is clear that $\eta_{q} \le \eta_{q-1}/2$ and therefore, we must have
$\eta_{3\log n} \le \eta_1/n^3$. We now define the following sets: 
\begin{align}
& \ca{P}_1 \triangleq \{i \in [n]: \s{U}_{it_2} \le \eta_{3\log n}\}, \\
& \ca{P}_2 \triangleq \bigcup_{q} \ca{X}'_q.
\end{align}
For each $i \in \ca{P}_1$, we estimate $\widehat{\s{U}}_{it_2}=0$ and since $\s{U}_{it_2} \le \s{U}_{\pi_{\widehat{\f{\mu}}_{t_1}}(n)t_2}/n^3$ and $\s{U}_{\pi_{\widehat{\f{\mu}}_{t_1}}(n)t_2} = \max_{i \in [n]} \s{U}_{it_2}$, we must have 
\begin{align}
\left| \widehat{\s{U}}_{it_2} - \s{U}_{it_2} \right| \le \frac{\max_{i \in [n]} \s{U}_{it_2}}{n^3} \quad \text{for all} \; i \in \ca{P}_1.
\end{align}  
For each $i \in \ca{P}_2$, we query $\s{U}_{it_2}$ and estimate $\widehat{\s{U}}_{it_2}=\s{U}_{it_2}$. Notice that we have 
\begin{align}
[n] \setminus \{\ca{P}_1 \cup \ca{P}_2\} = \bigcup_q \ca{X}_q,
\end{align}
and therefore for each $\ca{X}_q$ such that $\ca{X}_q \neq \phi$, we estimate $\widehat{\s{U}}_{it_2}= \s{U}_{\pi_{\widehat{\f{\mu}}_{t_1}}(p_{\eta_q})t_2}$ for all $i \in \ca{X}_q$. Now, since 
\begin{align}
\s{U}_{\pi_{\widehat{\f{\mu}}_{t_1}}(p_{\eta_{q-1}}-1)t_2} \le \eta_{q-1} \quad \text{and} \quad \s{U}_{\pi_{\widehat{\f{\mu}}_{t_1}}(p_{\eta_{q}})t_2} \ge \eta_q = \frac{\eta_{q-1}}{2} ,
\end{align} 
we must have that for all $i \in \ca{X}_q$ such that $\ca{X}_q \neq \phi$,
\begin{align}
\max_{i \in \ca{X}_q} \s{U}_{it_2} \le 2\min_{i \in \ca{X}_q} \s{U}_{it_2},
\end{align}
which proves the lemma. Note that each binary search step in Algorithm~\ref{algo:binary_search2} requires $O(\log n)$ queries, and may require an additional $O(\log n)$ queries if $\ca{X}'_q \neq \phi$. Similarly the time-complexity of each binary step is $O(\log n)$ as well. Since we are making at most $3\log n$ binary search steps, we get the desired query and time-complexity results.
\end{proof}
Lemma~\ref{lem:membership2} implies the following corollary.
\begin{corollary}
Consider the setting of Lemma~\ref{lem:membership2}. Then,
\begin{align}
\sum_{i \in [n]} \left| \widehat{\s{U}}_{it_2}- \s{U}_{it_2} \right| \le \sum_{i \in [n]} \s{U}_{it_2} + \frac{\max_{i \in [n]} \s{U}_{it_2}}{n^{2}}.
\end{align}
\end{corollary}
\begin{proof}
Recall that $\calV\triangleq\{1,2,\dots,3\log n\}$. Then, note that
\begin{align}
\sum_{i \in [n]} \left| \widehat{\s{U}}_{it_2}- \s{U}_{it_2} \right| &= \sum_{i \in \ca{P}_1} \left| \widehat{\s{U}}_{it_2}- \s{U}_{it_2} \right|+\sum_{i \in \ca{P}_2} \left| \widehat{\s{U}}_{it_2}- \s{U}_{it_2} \right|+\sum_{q\in\calV} \sum_{i \in \ca{X}_q:|\ca{X}_q| \neq \phi} \left| \widehat{\s{U}}_{it_2}- \s{U}_{it_2} \right|.
\end{align}
Next, we bound each of the terms on the r.h.s. of the above inequality. We have,
\begin{align}
&\sum_{i \in \ca{P}_2} \left| \widehat{\s{U}}_{it_2}- \s{U}_{it_2}\right| = 0, \\
&\sum_{i \in \ca{P}_1} \left| \widehat{\s{U}}_{it_2}- \s{U}_{it_2} \right| \le \frac{|\ca{P}_2|\max_{i \in [n]} \s{U}_{it_2}}{n^3} \le \frac{\max_{i \in [n]} \s{U}_{it_2}}{n^2}.
\end{align}
Finally, for each $q\in\calV$ such that $\ca{X}_q \neq \phi$, recall that $\widehat{\s{U}}_{it_2}=\min_{j \in \ca{X}_q} \s{U}_{it_2}$, for all $i \in \ca{X}_q$, and therefore,
\begin{align*}
&\max_{i \in \ca{X}_q} \s{U}_{it_2} \le 2\min_{i \in \ca{X}_q} \s{U}_{it_2}\implies\left| \widehat{\s{U}}_{it_2}- \s{U}_{it_2} \right| \le \frac{\max_{i \in \ca{X}_q} \s{U}_{it_2}}{2} \le \frac{\sum_{i \in \ca{X}_q} \s{U}_{it_2}}{|\ca{X}_q|}, \quad\forall i \in \ca{X}_q. 
\end{align*}
Thus,
\begin{align*}
\sum_{q\in\calV} \sum_{i \in \ca{X}_q} \left| \widehat{\s{U}}_{it_2}- \s{U}_{it_2} \right| \le \sum_{i \in [n]} \s{U}_{it_2},
\end{align*}
which proves the desired result.
\end{proof}

It is left to estimate the center of the smaller cluster among the two. This is done in Steps~7-13 of Algorithm~\ref{algo:estimate_sequentially2}. Specifically, for each $q \in \calV$, we randomly sample with replacement $r$ elements from $\ca{X}_q$ where $\ca{X}_q \neq \phi$. We denote this sampled multi-set by $\ca{Y}_q$, and query $\s{U}_{it_2}$, for each $i\in\ca{Y}_q$. We also sample $r$ elements from $\ca{P}_1$, and denote this sampled multi-set by $\ca{Q}$, and query $\s{U}_{it_2}$ for each $i\in\ca{Q}$. Note that we have already queried $\s{U}_{it_2}$ for every $i \in \ca{P}_2$ in Step 19 of Algorithm~\ref{algo:membership2}. Subsequently, we propose the following estimate for the center of the smaller cluster,
\begin{align}
\widehat{\f{\mu}}_{t_{2}} = \frac{\sum_{q\in\calV:\ca{X}_q \neq \phi} \frac{|\ca{X}_q|}{r}\sum_{i \in \ca{Y}_q} \s{U}_{it_{2}}^{\alpha} \f{x}_i+\sum_{i \in \ca{P}_2} \s{U}_{it_{2}}^{\alpha} \f{x}_i+\sum_{i \in \ca{Q}} \frac{|\ca{P}_1|}{r} \s{U}_{it_{2}}^{\alpha} \f{x}_i}{\sum_{q\in\calV:\ca{X}_q \neq \phi} \frac{|\ca{X}_q|}{r}\sum_{i \in \ca{Y}_q} \s{U}_{it_{2}}^{\alpha}+\sum_{i \in \ca{P}_2} \s{U}_{it_{2}}^{\alpha}+\sum_{i \in \ca{Q}}\frac{|\ca{P}_1|}{r} \s{U}_{it_{2}}^{\alpha}}.\label{eq:secondmean}
\end{align}
The following result gives guarantees on the estimation error associated with the smaller cluster among the two.
\begin{lemma}{\label{lem:smallest}}
Let $(\calX,\calP)$ be a consistent center-based clustering instance, and let $\delta\in(0,1)$. Then, with probability at least $1-\delta/2$, the estimator in \eqref{eq:secondmean} satisfies $\norm{\widehat{\f{\mu}}_{t_2}-\f{\mu}_{t_2}}_2 \le \epsilon$, 
if $r \ge \frac{c'\s{R}^4}{\epsilon^4}\log\frac{2}{\eta\delta}$, where $c'>0$ is an absolute constant. Also, this estimate requires $O\p{r\log n}$ membership-oracle queries, and a time-complexity of $O\p{dr \log n}$.
\end{lemma}

\begin{proof}[Proof of Lemma~\ref{lem:smallest}]
First, note that $[n] = \cup_{q:\ca{X}_q \neq \phi} \ca{X}_q \cup \ca{P}_1 \cup \ca{P}_2$. Therefore,
\begin{align}
\f{\mu}_{t_{2}} &=  \frac{ \sum_{q: \ca{X}_q \neq \phi} \f{\lambda}_q+\sum_{i \in \ca{P}_2} \s{U}_{it_{2}}^{\alpha} \f{x}_i+\f{\rho}}{ \sum_{q: \ca{X}_q \neq \phi} \s{Y}_q+\sum_{i \in \ca{P}_2} \s{U}_{it_{2}}^{\alpha}+\s{B}},
\end{align}
where $\f{\lambda}_q\triangleq\sum_{i \in \ca{X}_q} \s{U}_{it_{2}}^{\alpha} \f{x}_i$, $\f{\rho}\triangleq\sum_{i \in \ca{P}_1} \s{U}_{it_{2}}^{\alpha} \f{x}_i$, $\s{Y}_q\triangleq\sum_{i \in \ca{X}_q} \s{U}_{it_{2}}^{\alpha}$, and $\s{B}\triangleq\sum_{i \in \ca{P}_1} \s{U}_{it_{2}}^{\alpha}$. Similarly, using \eqref{eq:secondmean}, we have
\begin{align}
\widehat{\f{\mu}}_{t_{2}} &= \frac{\sum_{q\in\calV:\ca{X}_q \neq \phi}\bar{\f{\lambda}}_q+\sum_{i \in \ca{P}_2} \s{U}_{it_{2}}^{\alpha}\f{x}_i+\bar{\f{\rho}}}{\sum_{q\in\calV:\ca{X}_q \neq \phi}\bar{\s{Y}}_q+\sum_{i \in \ca{P}_2} \s{U}_{it_{2}}^{\alpha}+\bar{\s{B}}},
\end{align}
where $\bar{\f{\lambda}}_q\triangleq\frac{|\ca{X}_q|}{r}\sum_{i \in \ca{Y}_q} \s{U}_{it_{2}}^{\alpha} \f{x}_i$, $\bar{\f{\rho}}\triangleq\sum_{i \in \ca{Q}} \frac{|\ca{P}_1|}{r} \s{U}_{it_{2}}^{\alpha} \f{x}_i$, $\bar{\s{Y}}_q\triangleq\frac{|\ca{X}_q|}{r}\sum_{i \in \ca{Y}_q} \s{U}_{it_{2}}^{\alpha}$, and $\bar{\s{B}}\triangleq\sum_{i \in \ca{Q}}\frac{|\ca{P}_1|}{r} \s{U}_{it_{2}}^{\alpha}$. Notice that for each $q\in\calV$, the random variable $\bar{\f{\lambda}}_q$ can be written as a sum of $r$ i.i.d. random variables $\bar{\f{\lambda}}_{q,i_p}\triangleq\left|\ca{X}_q \right | \s{U}_{i_pt_{2}}^{\alpha} \f{x}_{i_p}$, where $i_p$ is sampled uniformly over $[n]$. Similarly, $\bar{\s{Y}}_q$ written as a sum of $r$ i.i.d random variables  $\bar{\s{Y}}_{q,i_p}\triangleq\left|\ca{X}_{q} \right | \s{U}_{i_pt_{2}}^{\alpha}$, where again $i_p$ is sampled uniformly over $[n]$. Finally, both $\bar{\f{\rho}}$ and $\bar{\s{B}}$ can also be written as a sum of $r$ i.i.d random variables $\bar{\f{\rho}}_{i_p}\triangleq \left|\ca{P}_1 \right | \s{U}_{i_pt_{2}}^{\alpha} \f{x}_{i_p}$ and $\bar{\s{B}}_{i_p}\triangleq\left|\ca{P}_1 \right | \s{U}_{i_pt_{2}}^{\alpha}$, respectively, where $i_p$ is sampled uniformly over $[n]$. Thus, it is evident that $\bb{E} \bar{\f{\lambda}}_{q,i_p} = \sum_{i \in \ca{X}_q} \s{U}_{it_{2}}^{\alpha} \f{x}_i$ and $\bb{E} \bar{\s{Y}}_{q,i_p} = \sum_{i \in \ca{X}_q} \s{U}_{it_{2}}^{\alpha}$, for all $p\in[r]$. Similarly, $\bb{E} \bar{\f{\rho}}_{i_p} = \sum_{i \in \ca{P}_1} \s{U}_{it_{2}}^{\alpha} \f{x}_i$ and $\bb{E} \bar{\s{B}}_{i_p} = \sum_{i \in \ca{P}_1} \s{U}_{it_{2}}^{\alpha}$, for all $p\in[r]$. Next, we note that
\begin{align}
\widehat{\f{\mu}}_{t_{2}} &= \f{\mu}_{t_{2}} + \sum_{q\in\calV} \frac{\bar{\f{\lambda}}_q-\f{\lambda}_q}{\sum_{q\in\calV:\ca{X}_{q} \neq \phi} \s{Y}_{q}+\sum_{i \in \ca{P}_2} \s{U}_{it_{2}}^{\alpha}+\s{B}}+\frac{\bar{\f{\rho}}-\f{\rho}}{\sum_{q\in\calV:\ca{X}_{q} \neq \phi} \s{Y}_{q}+\sum_{i \in \ca{P}_2} \s{U}_{it_{2}}^{\alpha}+\s{B}} \nonumber\\
&\quad+\widehat{\f{\mu}}_{t_{2}}\cdot\frac{\sum_{q\in\calV:\ca{X}_{q} \neq \phi} \s{Y}_{q}-\sum_{q\in\calV:\ca{X}_{q} \neq \phi} \bar{\s{Y}}_{q}+\s{B}-\bar{\s{B}}}{\sum_{q\in\calV:\ca{X}_{q} \neq \phi} \s{Y}_{q}+\sum_{i \in \ca{P}_2} \s{U}_{it_{2}}^{\alpha}+\s{B}}.
\end{align}
Therefore, using the triangle inequality and the fact that $\norm{\widehat{\f{\mu}}_{t_2}}_2\leq\s{R}$, we get
\begin{align}
&\norm{\widehat{\f{\mu}}_{t_{2}}-\f{\mu}_{t_{2}}}_2 \le \sum_{q\in\calV} \norm{\frac{\bar{\f{\lambda}}_q-\f{\lambda}_q}{\sum_{q\in\calQ:\ca{X}_{q} \neq \phi} \s{Y}_{q}+\sum_{i \in \ca{P}_2} \s{U}_{it_{2}}^{\alpha}+\s{B}}}_2+\norm{\frac{\bar{\f{\rho}}-\f{\rho}}{\sum_{q\in\calV:\ca{X}_{q} \neq \phi} \s{Y}_{q}+\sum_{i \in \ca{P}_2} \s{U}_{it_{2}}^{\alpha}+\s{B}}}_2\nonumber\\
&\hspace{1cm}+\s{R}\sum_{q\in\calV}\norm{\frac{\bar{\s{Y}}_{q}-\bar{Y}_{q}}{\sum_{q\in\calV:\ca{X}_{q} \neq \phi} \s{Y}_{q}+\sum_{i \in \ca{P}_2} \s{U}_{it_{2}}^{\alpha}+\s{B}}}_2+\s{R}\norm{\frac{\bar{\s{B}}-\s{B}}{\sum_{q\in\calV:\ca{X}_{q} \neq \phi} \s{Y}_{q}+\sum_{i \in \ca{P}_2} \s{U}_{it_{2}}^{\alpha}+\s{B}}}_2.\label{eqn:normDevi}
\end{align}
Now, for any $i \in \ca{Y}_q$, we have by definition,
\begin{align}
\s{U}_{it_{2}}^{\alpha} \leq\p{2\s{U}_{\pi_{\widehat{\f{\mu}}_{t_1}}(p_{\eta_q})t_2}}^\alpha,
\end{align}
and for any $i \in \ca{P}_1$,
\begin{align}
\s{U}_{it_{2}}^{\alpha} \le \frac{\max_{j \in [n]}\s{U}_{jt_2}^{\alpha}}{n^3}.
\end{align}
Next, using Lemmas~\ref{lem:hoeffding} and \ref{lem:gen_hoeffding}, we have for all $q\in\calV$, with probability at least $1-\delta$, 
\begin{align}
&\norm{\bar{\f{\lambda}}_q-\f{\lambda}_q}_2 \le \s{R}|\ca{X}_q| (2\s{U}_{\pi_{\widehat{\f{\mu}}_{t_1}}(p_{\eta_q})t_2})^{\alpha}\p{\frac{c}{r}\log\frac{2}{\eta\delta}}^{1/4}, \\ 
&\abs{\bar{\s{Y}}_q-\s{Y}_q} \le |\ca{X}_q| (2\s{U}_{\pi_{\widehat{\f{\mu}}_{t_1}}(p_{\eta_q})t_2})^{\alpha} \sqrt{\frac{c}{2r}\log\frac{2}{\eta\delta}},\\
&\norm{\bar{\f{\rho}}-\f{\rho}}_2 \le  \frac{\s{R}|\ca{P}_1| \max_{j \in [n]}\s{U}_{jt_2}^{\alpha}}{n^3}\p{\frac{c}{r} \log\frac{2}{\eta\delta}}^{1/4},\\
&\abs{\bar{\s{B}}-\s{B}} \le  \frac{|\ca{P}_1| \max_{j \in [n]}\s{U}_{jt_2}^{\alpha}}{n^3}\sqrt{\frac{c}{r}\log\frac{2}{\eta\delta}},
\end{align}
for some $c>0$. Substituting the above results in \eqref{eqn:normDevi}, we get 
\begin{align}
\norm{\widehat{\f{\mu}}_{t_{2}}-\f{\mu}_{t_{2}}}_2 &\le 2\s{R}\p{\frac{c}{r}\log\frac{2}{\eta\delta}}^{1/4}\cdot\frac{ \sum_{q\in\calV:\ca{X}_q \neq \phi} |\ca{X}_q| (2\s{U}_{\pi_{\widehat{\f{\mu}}_{t_1}}(p_{\eta_q})t_1})^{\alpha} }{\sum_{q\in\calV:\ca{X}_{q} \neq \phi} \s{Y}_{q}+\sum_{i \in \ca{P}_2} \s{U}_{it_{2}}^{\alpha}+\s{B}} \nonumber\\
&\quad+2\s{R}\p{\frac{c}{r}\log\frac{2}{\eta\delta}}^{1/4}\cdot\frac{ \frac{|\ca{P}_1| \max_{j \in [n]}\s{U}_{jt_1}^{\alpha}}{n^3} }{\sum_{q\in\calV:\ca{X}_{q} \neq \phi} \s{Y}_{q}+\sum_{i \in \ca{P}_2} \s{U}_{it_{2}}^{\alpha}+\s{B}}.
\end{align}
Noting to the facts that 
\begin{align}
\sum_{q\in\calV:\ca{X}_q \neq \phi} \s{Y}_q = \sum_{i \in \ca{X}_q:\ca{X}_q \neq \phi} \s{U}_{it_2}^{\alpha}  \ge |\ca{X}_q|(\s{U}_{\pi_{\widehat{\f{\mu}}_{t_1}}(p_{\eta_q})t_1})^{\alpha}
\end{align}
and 
\begin{align}
\sum_{q\in\calV:\ca{X}_{q} \neq \phi} \s{Y}_{q}+\sum_{i \in \ca{P}_2} \s{U}_{it_{2}}^{\alpha}+\s{B} &= \sum_{i \in [n]} \s{U}_{it_2}^{\alpha}\sum_{i \in [n]} \s{U}_{it_1}^{\alpha}\ge \max_{i \in [n]} \s{U}_{it_1}^{\alpha}, 
\end{align} 
and finally that $|\ca{P}_1| \le n$, we obtain
\begin{align}
\norm{\widehat{\f{\mu}}_{t_{2}}-\f{\mu}_{t_{2}}}_2 & 
\le 2^{\alpha+1}\s{R}\p{\frac{c}{r}\log\frac{2}{\eta\delta}}^{1/4}.
\end{align}
Therefore, for any $\epsilon>0$, with $r\ge\frac{c\s{R}^4}{\epsilon^4}\log\frac{2}{\eta\delta}$, we obtain $\norm{\widehat{\f{\mu}}_{t_{2}}-\f{\mu}_{t_{2}}}_2\leq\epsilon$, which proves the lemma. 
\end{proof}

\subsubsection{Proof of Theorem~\ref{thm:3state}}
First, Corollary~\ref{coro:base1} implies that a query complexity of $O\p{\frac{\s{R}^4}{\epsilon^4}\log\frac{1}{\delta}}$, and time-complexity of $O\p{\frac{d\s{R}^4}{\epsilon^4}\log\frac{1}{\delta}}$, suffice to approximate the center of the first cluster $t_1$ with probability at least $1-\delta/2$. Then, Lemma~\ref{lem:membership2} allows us to estimate $\s{U}_{it_2}$ for all $i\in[n]$ using a query complexity of $O(\log^2 n)$ and time-complexity of $O(\log^2 n)$. Also, Lemma \ref{lem:smallest} shows that a query complexity of $O\p{\frac{\s{R}^4\log n}{\epsilon^4} \log\frac{1}{\eta\delta}}$, and time-complexity of $O\p{\frac{d\s{R}^4\log n}{\epsilon^4}\log\frac{1}{\eta\delta}}$, suffice to approximate $\widehat{\f{\mu}}_{t_2}$ up to an error of $\epsilon$. Finally, we can use Lemma~\ref{lem:membership} to approximate $\s{U}_{ij}$ up to an error of $\eta$ using query complexity of $O(\log n/\eta)$, and a time-complexity of $O(n\log n+\log n/\eta)$, for all $i\in [n]$ and $j\in\{1,2\}$.

\section{Experiments}\label{app:exper}

\begin{figure*}[tb]
  \begin{subfigure}[t]{0.32\textwidth}
    \centering 
    \includegraphics[scale = 0.3]{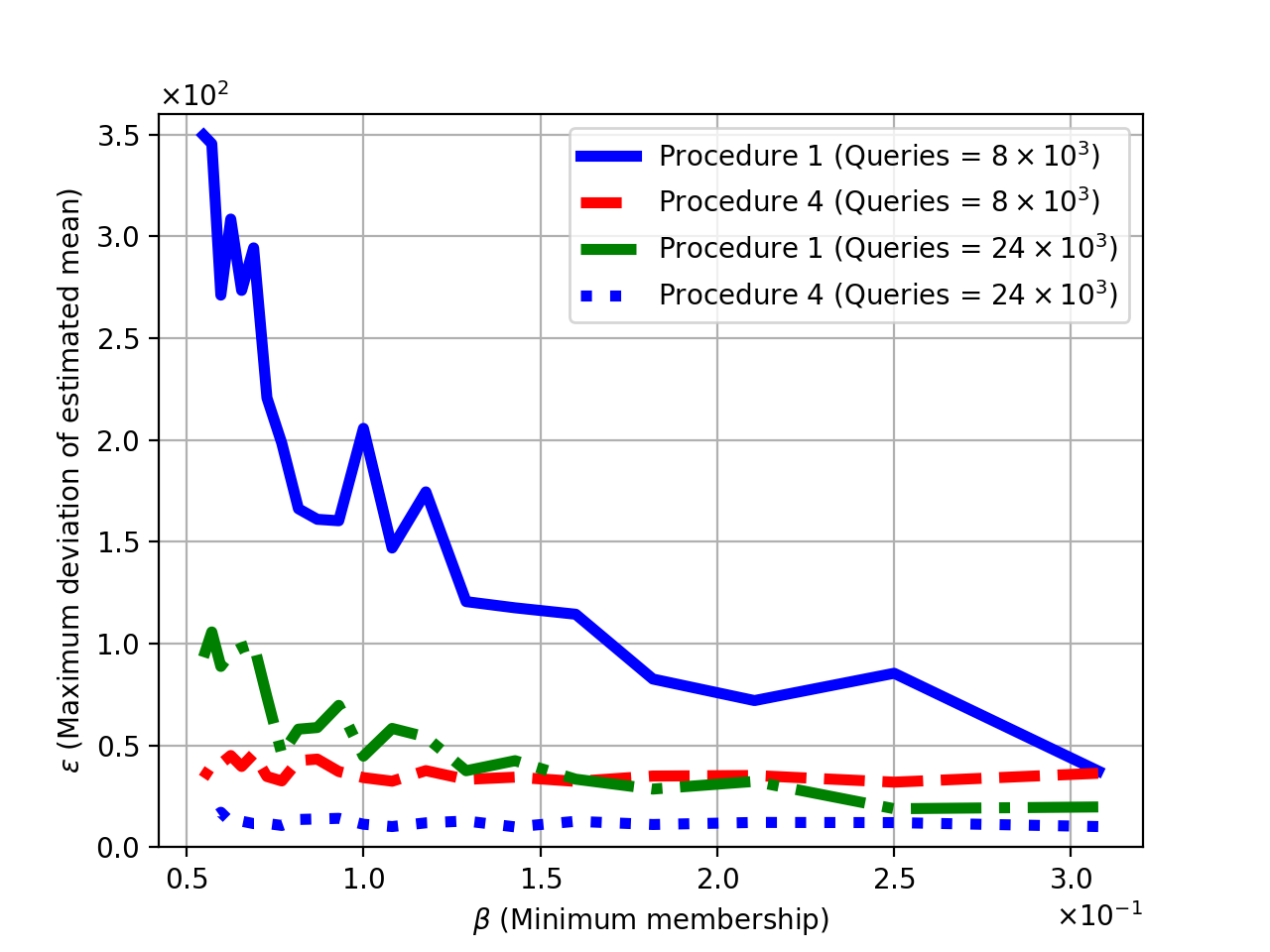}
    \caption{Comparison of two-phase algorithm and the sequential algorithm (see Algorithm \ref{algo:estimate_together} and \ref{algo:estimate_sequentially}. The error in recovery of means is plotted with varying $\beta$.}
          ~\label{fig:varybeta}
  \end{subfigure}
  \hfill
  \begin{subfigure}[t]{0.32\textwidth}
     \includegraphics[scale =  0.3]{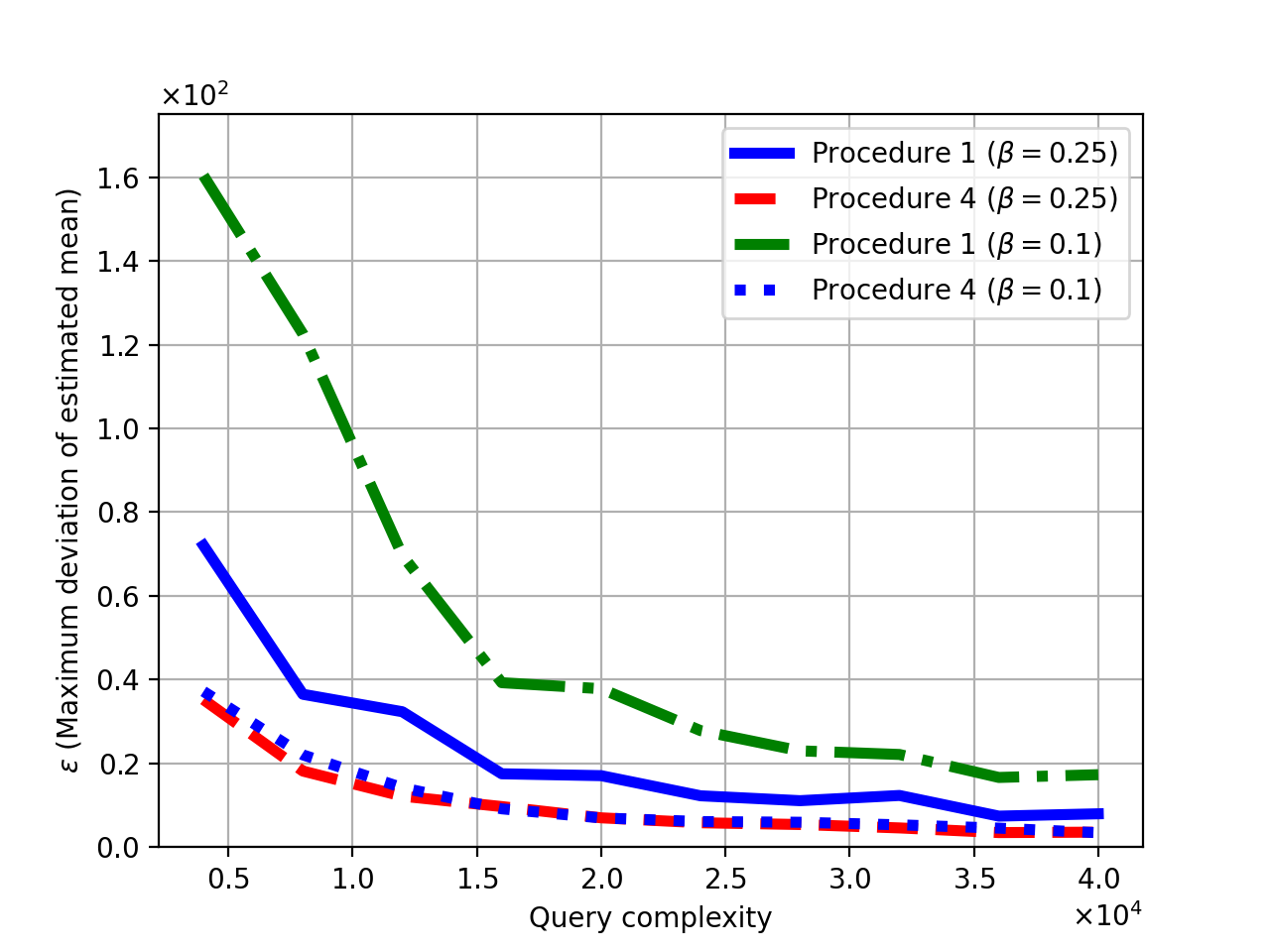}
     \caption{\small The error in recovery of means using the two-phase algorithm and sequential algorithm (see Algorithm \ref{algo:estimate_together} and \ref{algo:estimate_sequentially} with increasing queries keeping $\beta$ fixed.}
          ~\label{fig:varyquery}
  \end{subfigure}\hfill 
\begin{subfigure}[t]{0.32 \textwidth}
   \includegraphics[scale = 0.3]{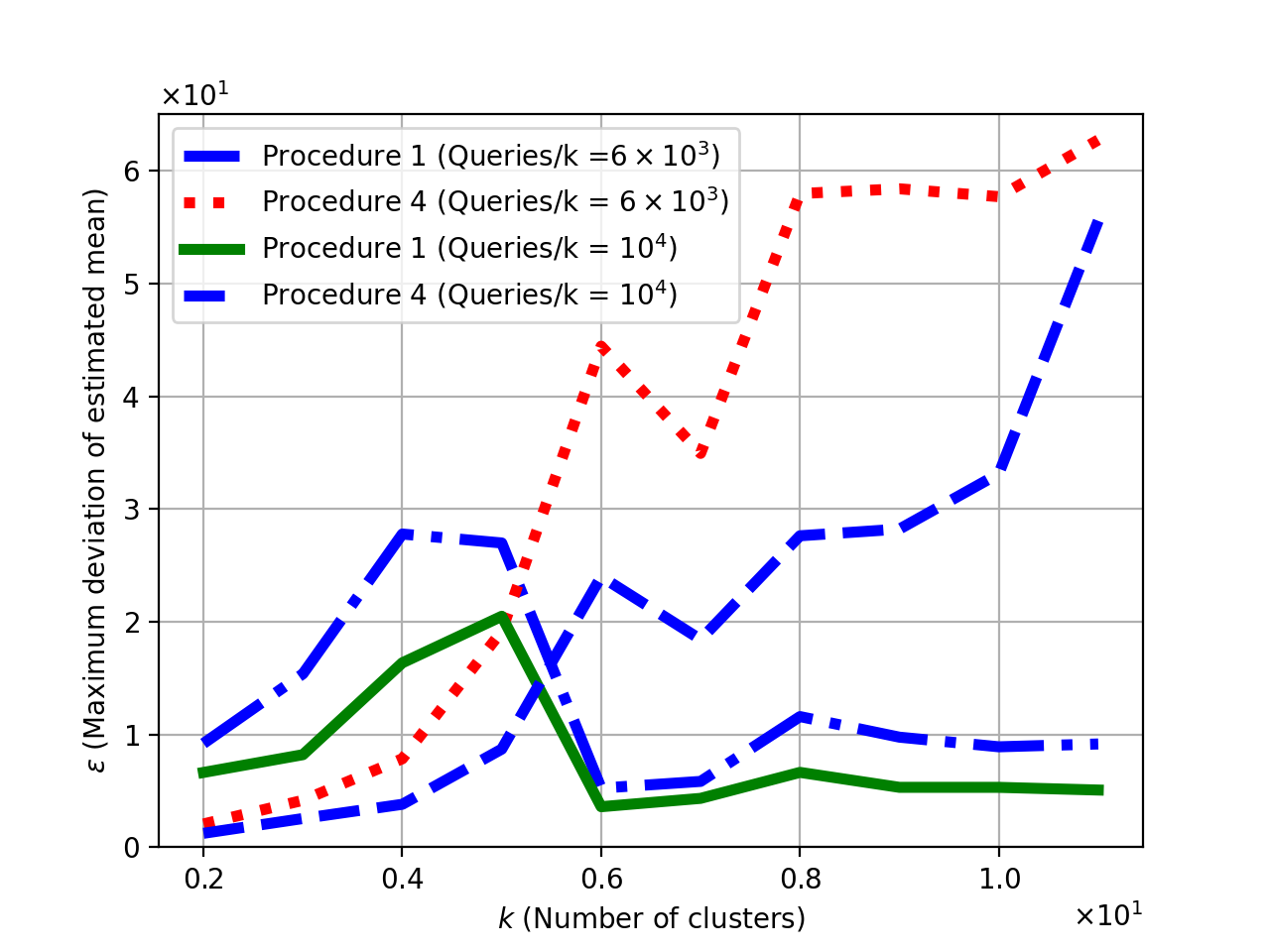}
   \caption{\small Comparison of the two-phase algorithm and the sequential algorithm (see Algorithm \ref{algo:estimate_together} and \ref{algo:estimate_sequentially}. The error in recovery of means is plotted with varying number of clusters ($k$).}
       ~\label{fig:varyk}
 \end{subfigure}%
  
\hfill

 \caption{\small Testing algorithms over synthetic datasets.\label{fig:1}}
\end{figure*}

\begin{figure*}[tb]
  \begin{subfigure}[t]{0.32\textwidth}
    \centering 
    \includegraphics[scale = 0.3]{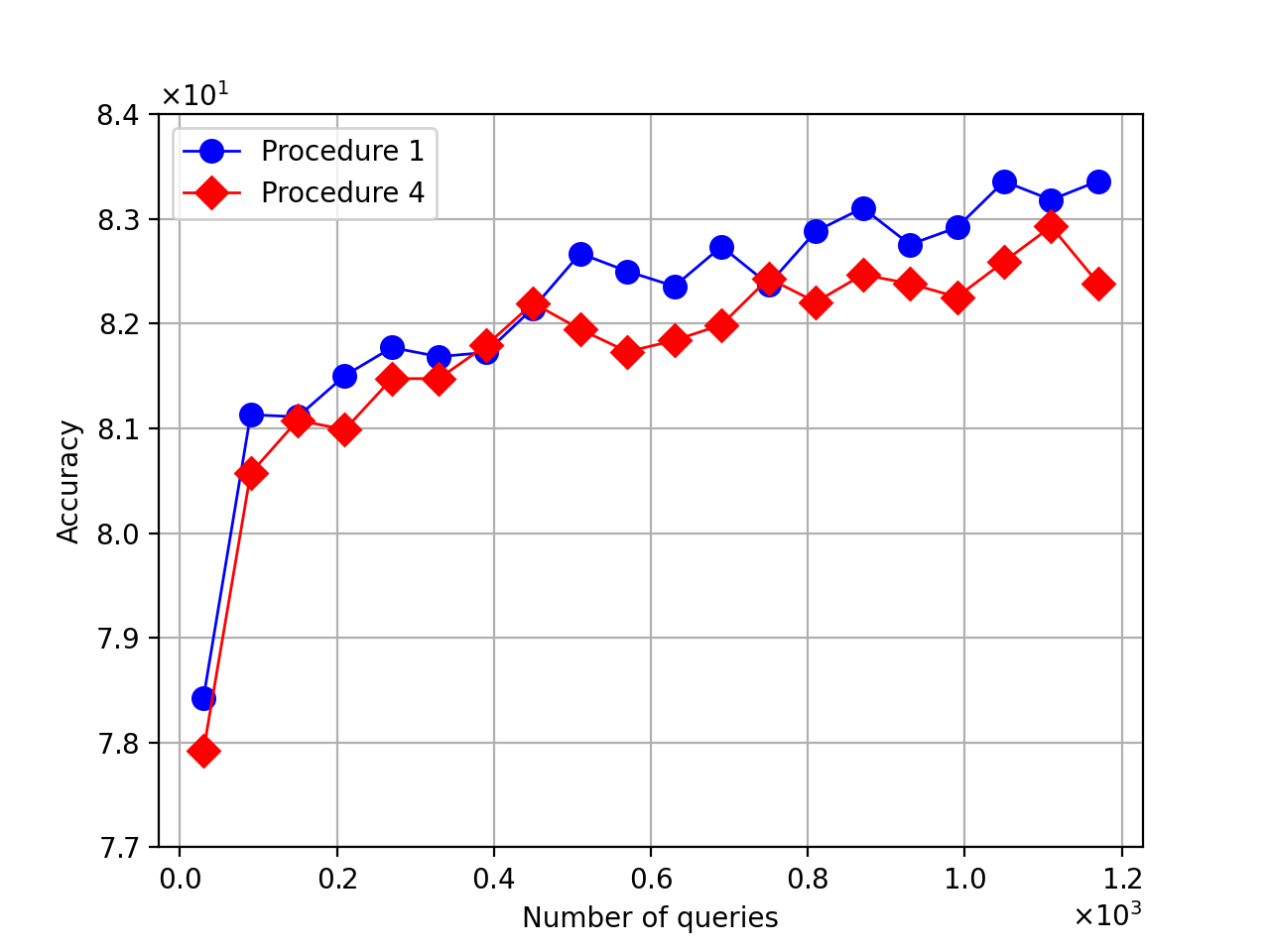}
    \caption{Iris}
          ~\label{fig:irispp}
  \end{subfigure}
  \hfill
  \begin{subfigure}[t]{0.32\textwidth}
     \includegraphics[scale =  0.3]{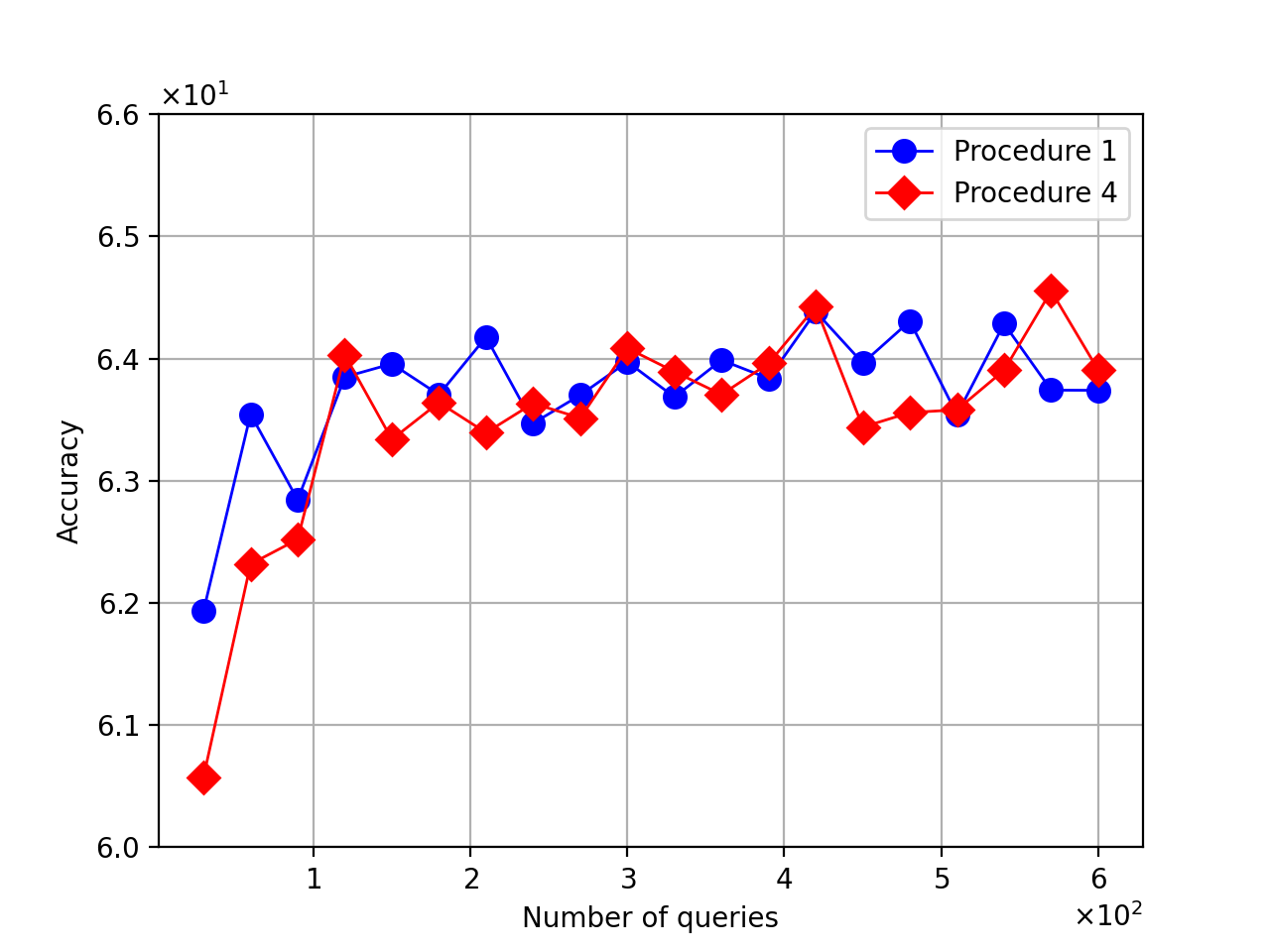}
     \caption{Wine}
          ~\label{fig:wineapp}
  \end{subfigure}\hfill 
\begin{subfigure}[t]{0.32 \textwidth}
   \includegraphics[scale = 0.3]{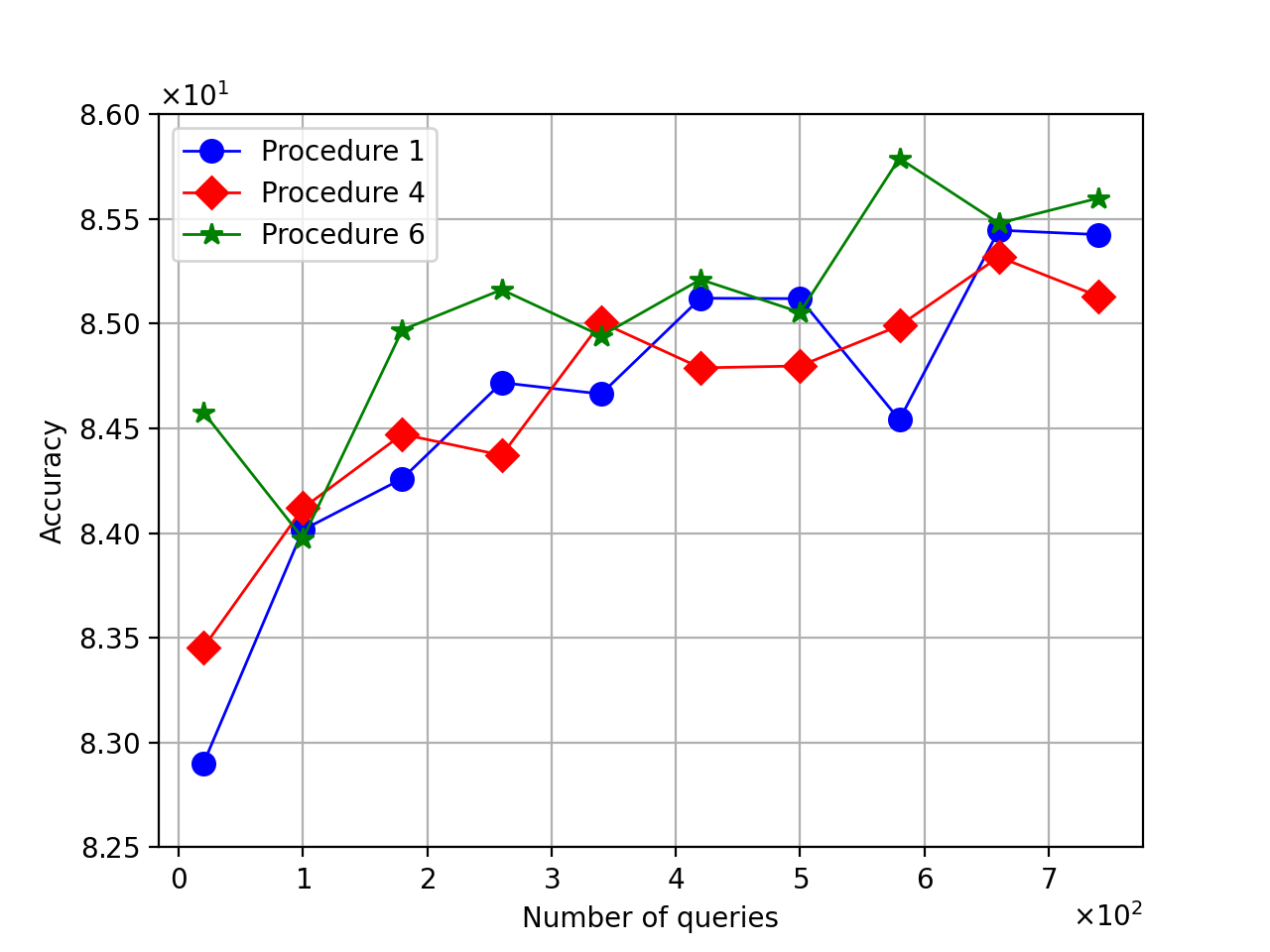}
   \caption{Breast Cancer}
       ~\label{fig:breast_cancerapp}
 \end{subfigure}%

\hfill

 \caption{\small Classification accuracy of algorithms for the Iris, Wine and Breast Cancer datasets.\label{fig:2app}}
\end{figure*}

\subsection{Synthetic Datasets} We conduct in-depth simulations of the proposed techniques over the following synthetic dataset. Specifically, we generate the dataset $\ca{X}$ by choosing $k=4$ centers with dimension $d=10$, such that $\f{\mu}_1$ is significantly separated from the other centers; the distance from each coordinate of $\f{\mu}_1$ to the coordinates of the other means is at least $1000$. Subsequently, for each $i \in \{1,2,3,4\}$ we randomly generate $\s{L}_i$ vectors from a spherical Gaussian distribution, with mean $\f{\mu}_i$, and a standard deviation of $20$ per coordinate. We then run the Fuzzy C-means algorithm\footnote{https://github.com/omadson/fuzzy-c-means}, and obtain a target solution $\ca{P}$ to be used by the oracle for responses. In order to understand the effect of $\beta$, we fix $\s{L}_1=5000$, and vary $\s{L}_2,\s{L}_3,\s{L}_4\in5000\cdot \zeta$, where $\zeta \in \{1,2,\dots,24\}$. It can be checked that $\beta =4/(1+3\zeta)$. We run Algorithms~\ref{algo:estimate_together} and \ref{algo:estimate_sequentially}. For the two-phase algorithm we take $\alpha=2$, $m=\nu$, and $\eta=0.1$, and $\alpha=2$, $m=\nu/2.5$, $\eta_1=0.1$, and $\eta_2=0.1$, for the sequential algorithm, where $\nu \in \{2000,6000\}$. Setting the parameters in this way keeps the same query complexity for both algorithms, so as to keep a fair comparison. We run each algorithm $20$ times. For each algorithm, we evaluate the maximal error in estimating the centers. The results are shown in Fig.~\ref{fig:1}. Specifically, Fig.~\ref{fig:varybeta} presents the estimation error as a function of $\beta$. It can be seen that for small $\beta$'s, the sequential algorithm is significantly better compared to the two-phase algorithm, whereas for larger $\beta$'s, they are comparable. Then, for $\beta=0.25,0.1$, Fig.~\ref{fig:varyquery} shows the estimation error as a function of the number of queries. Finally, to understand the effect of the number of clusters, we generate $k$ clusters using a similar method as above. We take $\s{L}_1=1000$, and $\s{L}_i=12000$, for all $2\leq i\leq k$. We vary $k \in \{2,3,\dots,11\}$. For the two-phase algorithm, we take $\alpha=2$, $m=\nu$, and $\eta=0.1$, and $\alpha=2$, $m=\nu/2.5$, $\eta_1=0.1$, and $\eta_2=0.1$, for the sequential algorithm, where $\nu=\{2000,6000\}$. Fig.~\ref{fig:varyk} shows the estimation error as a function of $k$. We can clearly observe that the two-phase algorithm performs significantly better as $k$ increases but the sequential algorithm works better for small $k$.

\subsection{Real-World Datasets} 

 In our experiments, we use three well-known real-world datasets available in scikit-learn \cite{scikit-learn}: the Iris dataset ($150$ elements, $4$ features, and $3$ classes), the Wine dataset ($178$ elements, $13$ features, and $3$ classes), and the Breast Cancer dataset ($569$ elements, $30$ features, and $2$ classes). For the Iris and Wine datasets, we run the two-phase and sequential algorithms. We take $\alpha=2$, $m=\nu$, and $\eta=0.1$, for the two-phase algorithm, and $\alpha=2$, $m=2\nu/3$, $r = m/\eta_1$, $\eta_1=0.1$, and $\eta_2=0.1$, for the sequential algorithm, where $\nu \in \{10,20,\dots,410\}$, keeping the same query complexity for both algorithms. These values do not necessarily satisfy what is needed by our theoretical results. We run both algorithms with each set of parameters $500$ times to account for the randomness. 
In our experiments, we use a hard cluster assignment as ground truth (or rather the target clustering $\ca{P}$ to be used by the oracle for responses), and use our algorithms to return a fuzzy assignment. We must point out over here that our fuzzy algorithms can be used to solve hard clustering problems as well and therefore, it is not unreasonable to have hard clusters as the target solution.

Subsequently, we estimate the membership weights for all elements, and for each element, we predict the class the element belongs to as the one to which the element has the highest membership weight (i.e., $\s{argmax}_j \widehat{\s{U}}_{ij}$, for element $i$)\footnote{This is similar to rounding in Linear Programming}. Once we have classified all the data-points using our algorithms, we can check the classification accuracy since we possess the ground-truth labels. Note that the ground truth labels can be inconsistent with the best clustering or $\ca{P}^{\star}$, the solution that minimizes the objective in the Fuzzy $k$-means problems (Defintion \ref{def:fuzzy_prim}) but we assume that the ``ground truth'' labels that are given by humans are a good proxy for the best clustering. 

We then plot the classification accuracy as a function of the number of queries. Fig.~\ref{fig:2app} shows the average classification accuracy for the above three data-sets by comparing the predicted classes and the ground truth. For the Breast Cancer dataset, since the number of clusters is two, we additionally compare the two-phase and sequential algorithms to Algorithm~\ref{algo:estimate_sequentially2} with $\alpha=2$, $m=2\nu/3$, $r = m/\eta_1$, and $\eta=0.1$. It turns out that for these real-world datasets, the performance of all algorithms are comparable. It can be seen that the accuracy increases as a function of the number of queries, as expected. \textbf{Further, by using the well-known Lloyd's style iterative Fuzzy C-means algorithm with random initialization~\cite{dias2019fuzzy}, we get an average classification accuracy (over 20 trials) of only $31.33\%$, $35.96\%$ and $14.58\%$ on the Iris, Wine, and Breast Cancer datasets, respectively}. This experiment shows that using a few membership queries increases the accuracy of a poly-time algorithm drastically, corroborating the results of our paper.

\begin{figure*}[t!]
\centering
\includegraphics[width=0.6\textwidth]{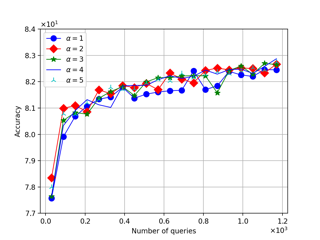}
\caption{Classification accuracy as a function of the fuzzifier $\alpha$ and the number of queries for the Iris dataset.}
\label{fig:compalpha}
\end{figure*}

\subsection{Accuracy as a Function of $\alpha$} As discussed in right after Theorem~\ref{thm:2}, the ``fuzzifier" $\alpha$ is not subject to an optimization. Nonetheless, if we assume the existence of a ground truth, we can compare the clustering accuracy  for different values of $\alpha$. Accordingly, in Fig.~\ref{fig:compalpha} we test the performance of our algorithms on the Iris dataset for a few values of $\alpha$. We calculate the average accuracy over $500$ trials for each set of parameters. We conclude this section by discussing the issue of comparing our semi-supervised fuzzy approach to the semi-supervised hard objective \cite{Ashtiani16}. Generally speaking, in the absence of a ground truth, comparing both approaches is \emph{meaningless}. When the ground truth represents a disjoint clustering, then it is reasonable that following a hard approach (essentially $\alpha=1$) will capture this ground truth better. However, the whole point of using fuzzy clustering in the first place is when the clusters, in some sense, overlap. Indeed, the initial main motivation for studying fuzzy clustering is that it is applicable to datasets where datapoints show affinity to multiple labels, the clustering criteria are vague and data features are unavailable. Nonetheless, in Fig.~\ref{fig:compalpha} we compare the performance of both the fuzzy and hard approaches (essentially $\alpha=1$) on the Iris dataset.

\section{Conclusion and Outlook}\label{sec:conc}
In this paper, we studied the fuzzy $k$-means problem, and proposed a semi-supervised active clustering framework, where the learner is allowed to interact with a membership-oracle, asking for the memberships of a certain set of chosen items. We studied both the query and computational complexities of clustering in this framework. In particular, we provided two probabilistic algorithms (two-phase and sequential) for fuzzy clustering that ask $O(\mathsf{poly}(k)\log n)$ membership queries and run with polynomial-time-complexity. The main difference between these two algorithms is the dependency of their query complexities on the size of the smallest cluster $\beta$. The sequential algorithm exhibits more graceful dependency on $\beta$. Finally, for $k=2$ we were able to remove completely the dependency on $\beta$ (see, Appendix~\ref{app:algoTh3}). We hope our work has opened more doors than it closes. Apart from tightening the obtained query complexities, there are several exciting directions for future work:
\begin{itemize}
\item It is important to understand completely the dependency of the query complexity on $\beta$. Indeed, we showed that for $k=2$ there exists an algorithm whose query complexity is independent of $\beta$, but what happens for $k>2$?
\item It would be interesting to understand to what extent the algorithms and analysis in this paper, can be applied to other clustering problems which depend on different metrics other than the Euclidean one. 
\item Our paper presents upper bounds (sufficient conditions) on the query complexity. It is interesting and challenging to derive algorithm-independent information-theoretic lower bounds on the query complexity.
\item As mentioned in the introduction it is not known yet whether the fuzzy $k$-means problem lies in NP like the hard $k$-means problem. Answering this question will give a solid motivation to the semi-supervised setting considered in this paper. Furthermore, just as the information-theoretic lower bounds, it would be interesting to derive computational lower bounds as well.
\item In this paper we focused on the simplest form of oracle responses. However, there are many other interesting and important settings, e.g., the noisy setting (Appendix~\ref{app:NoisyOracle}). Another interesting problem would be to consider adversarial oracles who intentionally provide corrupted responses. 
\end{itemize}

\paragraph{Acknowledgements.} This work is supported in part by NSF awards 2133484, 2127929, and 1934846.
\bibliographystyle{abbrv}

\newpage
\appendix
\appendixpage

\section{Proof of Lemma~\ref{lem:1}}\label{app:XBmeasure}

To prove Lemma~\ref{lem:1} we start by proving a few inequalities. Since $\calA$ is an $(\epsilon_1,\epsilon_2,\s{Q})$-solver, using Definition~\ref{def:query} and Taylor's expansion, we get for any $i\in[n]$ and $j\in[k]$,
\begin{align}
\widehat{\s{U}}_{ij}^\alpha &\leq (\s{U}_{ij}+\epsilon_2)^\alpha\\
& = \s{U}_{ij}^\alpha+\alpha\s{U}_{ij}^{\alpha-1}\epsilon_2+o(\epsilon_2^2)\\
&\leq \s{U}_{ij}^\alpha+\alpha\epsilon_2+o(\epsilon_2^2),
\end{align}
where the last inequality follows from the fact that $\s{U}_{ij}\leq1$. Also, for any $i\in[n]$ and $j\in[k]$, we have
\begin{align}
\norm{\f{x}_i-\widehat{\f{\mu}}_j}_2^2 &= \norm{\f{x}_i-\f{\mu}_j}_2^2+2\cdot(\f{x}_i-\f{\mu}_j)^T(\f{\mu}_j-\widehat{\f{\mu}}_j)+\norm{\f{\mu}_j-\widehat{\f{\mu}}_j}_2^2\\
&\leq \norm{\f{x}_i-\f{\mu}_j}_2^2+4\cdot\s{R}\cdot\epsilon_1+\epsilon_1^2,\label{eqn:upperBound2norm}
\end{align}
where the inequality follows from the fact that $\calA$ is an $(\epsilon_1,\epsilon_2,\s{Q})$-solver, the Cauchy-Schwarz inequality, and the fact that since $\f{x}_i\in\calB(0,\s{R})$ so as $\bf{\mu}_j\in\calB(0,\s{R})$, and thus $\norm{\f{x}_i-\f{\mu}_j}_2^2\leq 4\cdot\s{R}^2$. Finally, we have for any $i\neq j\in[k]$,
\begin{align}
\norm{\widehat{\f{\mu}}_i-\widehat{\f{\mu}}_j}_2^2&\geq \norm{\f{\mu}_i-\f{\mu}_j}_2^2+\norm{\widehat{\f{\mu}}_i-\f{\mu}_i}_2^2+\norm{\widehat{\f{\mu}}_j-\f{\mu}_j}_2^2+2(\f{\mu}_i-\f{\mu}_j)^T(\widehat{\f{\mu}}_i-\f{\mu}_i)\nonumber\\
&+2(\f{\mu}_i-\f{\mu}_j)^T(\widehat{\f{\mu}}_j-\f{\mu}_j)+2(\widehat{\f{\mu}}_i-\f{\mu}_i)^T(\widehat{\f{\mu}}_j-\f{\mu}_j)\\
&\geq \norm{\f{\mu}_i-\f{\mu}_j}_2^2-2\epsilon_1^2-8\s{R}\epsilon_1-2\epsilon_1^2\\
& = \norm{\f{\mu}_i-\f{\mu}_j}_2^2-8\s{R}\epsilon_1-4\epsilon_1^2.\label{eqn:lowerBoundOnMini}
\end{align}
Now, with the above results, we note that
\begin{align}
\mathsf{J_{fm}}(\calX,\widehat{\calP}) &= \sum_{i=1}^n\sum_{j=1}^k\widehat{\s{U}}^\alpha_{ij}\norm{\f{x}_i-\widehat{\f{\mu}}_j}^2_2\\
&\leq \sum_{i=1}^n\sum_{j=1}^k\s{U}^\alpha_{ij}\norm{\f{x}_i-\widehat{\f{\mu}}_j}^2_2+\alpha\epsilon_2\sum_{i=1}^n\sum_{j=1}^k\norm{\f{x}_i-\widehat{\f{\mu}}_j}^2_2+o(\epsilon_2^2).
\end{align}
Now since $\f{\mu}_j\in\calB(0,\s{R})$, we can say that $\widehat{\f{\mu}}_j\in\calB(0,\s{R}+\epsilon_1)$. Therefore, $\norm{\f{x}_i-\widehat{\f{\mu}}_j}^2_2\leq 2[\s{R^2}+(\s{R}+\epsilon_1)^2]$. Hence,
\begin{align}
\mathsf{J_{fm}}(\calX,\widehat{\calP}) &\leq \sum_{i=1}^n\sum_{j=1}^k\s{U}^\alpha_{ij}\norm{\f{x}_i-\widehat{\f{\mu}}_j}^2_2+2nk\alpha\epsilon_2[\s{R^2}+(\s{R}+\epsilon_1)^2]+o(\epsilon_2^2).
\end{align}
Next, using \eqref{eqn:upperBound2norm}, we get
\begin{align}
\sum_{i=1}^n\sum_{j=1}^k\s{U}^\alpha_{ij}\norm{\f{x}_i-\widehat{\f{\mu}}_j}^2&\leq \sum_{i=1}^n\sum_{j=1}^k\s{U}^\alpha_{ij}\norm{\f{x}_i-\f{\mu}_j}_2^2+\pp{4\cdot\s{R}\cdot\epsilon_1+\epsilon_1^2}\sum_{i=1}^n\sum_{j=1}^k\s{U}^\alpha_{ij}\\
&\leq \sum_{i=1}^n\sum_{j=1}^k\s{U}^\alpha_{ij}\norm{\f{x}_i-\f{\mu}_j}_2^2+n\pp{4\cdot\s{R}\cdot\epsilon_1+\epsilon_1^2}\\
& = \mathsf{J_{fm}}(\calX,\calP)+n\pp{4\cdot\s{R}\cdot\epsilon_1+\epsilon_1^2},
\end{align}
where the second inequality follows from the fact that $\s{U}_{ij}\in[0,1]$, and thus $\sum_{i=1}^n\sum_{j=1}^k\s{U}^\alpha_{ij}\leq \sum_{i=1}^n\sum_{j=1}^k\s{U}_{ij} = n$. Therefore, we obtain
\begin{align}
\mathsf{J_{fm}}(\calX,\widehat{\calP})&\leq  \mathsf{J_{fm}}(\calX,\calP)+n\pp{4\cdot\s{R}\cdot\epsilon_1+\epsilon_1^2}+2nk\alpha\epsilon_2[\s{R^2}+(\s{R}+\epsilon_1)^2]+o(\epsilon_2^2)\\
& \leq \mathsf{J_{fm}}(\calX,\calP)+n\cdot O(\epsilon_1)+nk\cdot O(\epsilon_2)+n\cdot o(\epsilon_1^2)+nk\cdot o(\epsilon_2^2).\label{eqn:JfmEst}
\end{align}
We are now in a position to bound $\s{XB}(\calX,\widehat{\calP})$. Using \eqref{eqn:lowerBoundOnMini} and \eqref{eqn:JfmEst}, we have
\begin{align}
\mathsf{XB}(\calX,\widehat{\calP})&= \frac{\mathsf{J_{fm}}(\calX,\widehat{\calP})}{nk\cdot\min_{i\neq j}\norm{\widehat{\f{\mu}}_i-\widehat{\f{\mu}}_j}_2^2}\\
&\leq \frac{\mathsf{J_{fm}}(\calX,\calP)+n\cdot O(\epsilon_1)+nk\cdot O(\epsilon_2)+n\cdot o(\epsilon_1^2)+nk\cdot o(\epsilon_2^2)}{nk\cdot\pp{\min_{i\neq j}\norm{\f{\mu}_i-\f{\mu}_j}_2^2-8\s{R}\epsilon_1-4\epsilon_1^2}}\\
& = \frac{\mathsf{J_{fm}}(\calX,\calP)}{nk\cdot\min_{i\neq j}\norm{\f{\mu}_i-\f{\mu}_j}_2^2}+\frac{n\cdot O(\epsilon_1)+nk\cdot O(\epsilon_2)+n\cdot o(\epsilon_1^2)+nk\cdot o(\epsilon_2^2)}{nk\cdot\min_{i\neq j}\norm{\f{\mu}_i-\f{\mu}_j}_2^2}\nonumber\\
&+\frac{\mathsf{J_{fm}}(\calX,\calP)}{nk\cdot\min_{i\neq j}\norm{\f{\mu}_i-\f{\mu}_j}_2^2}\frac{O(\epsilon_1)}{\min_{i\neq j}\norm{\f{\mu}_i-\f{\mu}_j}_2^2}+o(\epsilon_1^2+\epsilon_2^2)\\
& = \mathsf{XB}(\calX,\calP)+\mathsf{XB}(\calX,\calP)\cdot\frac{O(\epsilon_1)}{\min_{i\neq j}\norm{\f{\mu}_i-\f{\mu}_j}_2^2}+\frac{O(\epsilon_2)}{\min_{i\neq j}\norm{\f{\mu}_i-\f{\mu}_j}_2^2}\nonumber\\
&\ \ +o\p{\frac{\epsilon_1^2+\epsilon_2^2}{\min_{i\neq j}\norm{\f{\mu}_i-\f{\mu}_j}_2^2}}.
\end{align}
Using the same steps a similar lower bound can be obtained, which concludes the proof.

\section{Auxiliary Lemmata}

In this section we present and prove a few auxiliary results which will be used in the proofs our main results. We start with the following standard concentration inequalities.
\begin{lemma}[Hoeffding's inequality]{\label{lem:hoeffding}}
Let $\s{X}_1,\s{X}_2,\dots,\s{X}_n$ be i.i.d random variables, such that $|\s{X}_i| \le \s{R}$ a.s., and $\bb{E}\s{X}_i =\mu$, for all $i \in [n]$. Then, with probability at least $1-\delta$,
\begin{align}
\left|\frac{1}{n} \sum_{i=1}^n \s{X}_i -\mu \right| \le \s{R}\epsilon,
\end{align}
if $n \ge \frac{c\log (1/\delta)}{2\epsilon^2}$, where $c>0$ is some absolute constant.
\end{lemma}

\begin{lemma}[Generalized Hoeffding's inequality]{\label{lem:gen_hoeffding}}
Let $\mathbf{X}_1,\mathbf{X}_2,\dots,\mathbf{X}_n$ be i.i.d random vectors, such that $\norm{\mathbf{X}_i}_2 \le \s{R}$ a.s., and $\bb{E}\mathbf{X}_i =\f{\mu}$, for all $i \in [n]$. Then, with probability at least $1-\delta$,
\begin{align}
\norm{\frac{1}{n} \sum_{i=1}^n \mathbf{X}_i -\f{\mu} }_2^2 \le \s{R}^2 \epsilon,
\end{align}
if $n \ge \frac{c \log (1/\delta)}{\epsilon^2}$, where $c>0$ is some absolute constant.
\end{lemma}
The following locality lemma states that the fuzzy $k$-means function is strictly increasing.
\begin{lemma}\label{lem:fuzyy_prop}
Let $(\calX,\calP^\star)$ be a clustering instance, where $\calP^\star$ refers to the optimal solution for the fuzzy $k$-mean problem (namely, minimizes the objective in \eqref{eqn:FuzzyObjective}). Then, for any $i,j\in[n]$ and $\ell\in[k]$ with $   \norm{\f{x}_i-\f{\mu}^\star_{\ell}}_2^2\leq\norm{\f{x}_j-\f{\mu}^\star_{\ell}}_2^2$, we have $\s{U}_{i\ell}\geq\s{U}_{j\ell}$.
\end{lemma}
\begin{proof}[Proof of Lemma~\ref{lem:fuzyy_prop}]
Consider some $i,j\in[n]$ and $\ell\in[k]$ with $\norm{\f{x}_i-\f{\mu}^\star_{\ell}}_2^2\leq\norm{\f{x}_j-\f{\mu}^\star_{\ell}}_2^2$. By definition $\{\s{U}^\star_{i\ell}\}_{i=1}^n$ minimizes the cost $\sum_{i=1}^n\s{U}_{i\ell}^\alpha\norm{\f{x}_i-\f{\mu}_\ell}_2^2$. This implies that,
\begin{align}
\s{U}_{i\ell}^\alpha\norm{\f{x}_i-\f{\mu}_\ell}_2^2+\s{U}_{j\ell}^\alpha\norm{\f{x}_j-\f{\mu}_\ell}_2^2\leq \s{U}_{j\ell}^\alpha\norm{\f{x}_i-\f{\mu}_\ell}_2^2+\s{U}_{i\ell}^\alpha\norm{\f{x}_j-\f{\mu}_\ell}_2^2, 
\end{align}
which is equivalent to,
\begin{align}
\s{U}_{i\ell}^\alpha\pp{\norm{\f{x}_i-\f{\mu}_\ell}_2^2-\norm{\f{x}_j-\f{\mu}_\ell}_2^2}\leq \s{U}_{j\ell}^\alpha\pp{\norm{\f{x}_i-\f{\mu}_\ell}_2^2-\norm{\f{x}_j-\f{\mu}_\ell}_2^2}.
\end{align}
Since $\norm{\f{x}_i-\f{\mu}_\ell}_2^2-\norm{\f{x}_j-\f{\mu}_\ell}_2^2\leq0$, we get $\s{U}_{i\ell}^\alpha\geq\s{U}_{j\ell}^\alpha$, which concludes the proof.
\end{proof}

\section{Noisy Oracle Responses}\label{app:NoisyOracle}

In this section, we discuss briefly the effect of a noisy membership-oracle, defined as follows.
\begin{definition}[Noisy Membership-Oracle]\label{def:noisy_oracle_fuzzy}
A fuzzy query asks the membership weight of an instance $\xb_i$ to a cluster $j$ and obtains in response a noisy answer $\calO_{\mathrm{noisy}}(i,j) = \s{U}_{ij}+\zeta_{ij}$
where $\zeta_{ij}$ is a zero mean random variable with variance $\sigma^2$.
\end{definition}
To present our main result, let $\rho\in\mathbb{R}_+$ be defined as
\begin{align}
\min_{j \in [k]} \sum_{i \in [n]} \s{U}_{ij}^{\alpha} = \rho n.\label{eqn:rhodef}
\end{align} 
The result below handles the situation where the oracle responses are noisy. 
\begin{theorem}\ref{thm:10}
Let $\kappa>0$, and assume that there exists an $(\epsilon_1,\epsilon_2,\s{Q})$-solver for a clustering instance $(\ca{X},\ca{P})$ using the membership-oracle responses $O_{\s{fuzzy}}$. Then, there exist a $$\p{\frac{2\s{R}\alpha  (\epsilon_2+\kappa)}{\rho} ,\epsilon_2+\kappa,\frac{8\s{Q}\sigma^2\log n}{\kappa^2}}-\s{solver},$$
for $(\ca{X},\ca{P})$ using the noisy oracle $O_{\s{noisy}}$.
\end{theorem}
\begin{proof}
Assume that algorithm $\calA_{\s{noiseless}}$ is an $(\epsilon_1,\epsilon_2,\s{Q})$-solver for a clustering instance $(\ca{X},\ca{P})$ using queries to a noiseless oracle $O_{\s{fuzzy}}$. In order to handle noisy responses we propose the following algorithm $\calA_{\s{noisy}}$: apply algorithm $\calA_{\s{noiseless}}$ for $\s{T}$ steps using noisy queries to $O_{\s{noisy}}$. We will show that this algorithm obtains the guarantees in Theorem~\label{thm:10}. To that end, in each such step, we obtain noisy estimates for the memberships and the centers. Then, we use these local estimates to obtain a clean final estimate for the memberships and the centers. Specifically, consider $\s{T}$ independent and noisy clustering instances $\{ \ca{P}^t \equiv (\f{\mu}^t,\s{V}^t)\}_{t=1}^{\s{T}}$, such that 
\begin{align}
    \s{V}_{ij}^t = \s{U}_{ij}+\zeta_{ij} \quad \text{and} \quad 
     \f{\mu}_j^t = \frac{\sum_{i=1}^n\s{V}_{ij}^\alpha\xb_i}{\sum_{i=1}^n\s{V}_{ij}^\alpha},
\end{align}
for $t\in[\s{T}]$. Note that the randomness in the definition of the aforementioned clustering instances lies in the realization of the independent random variables $\zeta_{ij}^{1},\zeta_{ij}^{2},\dots,\zeta_{ij}^{\s{T}}$. 
For each such instance we apply one of the algorithms we developed for the noiseless oracle. Accordingly, for all $t \in [\s{T}]$, suppose we have a $(\epsilon_1,\epsilon_2,Q)$-solver that makes $Q$ queries to the $\ca{P}^t$-oracle to compute $\widehat{\s{V}}_{ij}^t$. Then, we know that,
\begin{align}
    \left|\widehat{\s{V}}_{ij}^t-\s{V}_{ij}^t\right| \le \epsilon_2,
\end{align}
for all $t<[\s{T}]$. Now, all we have to do is to use these local estimates to calculate our final estimates for the underlying memberships and centers. Specifically, for $\s{T}'<\s{T}$, we must have 
\begin{align}
    \left|\frac{1}{\s{T}'}\sum_{j=1}^{\s{T}'}\widehat{\s{V}}_{ij}^t-\s{U}_{ij}\right| \le \epsilon_2+\left|\frac{1}{\s{T}'}\sum_{t=1}^{\s{T}'} \zeta_{ij}^t\right|.
\end{align}
By Chebychev's inequality, for any $\kappa>0$, we get 
\begin{align}
    \Pr\pp{\left|\frac{1}{\s{T}'}\sum_{t=1}^{\s{T}'} \zeta_{ij}^t\right| \ge \kappa } \le \frac{\sigma^2}{\kappa^2 \s{T}'}.
\end{align}
Next, we partition the $\s{T}$ responses from the oracle into $\s{B}$ batches of size $\s{T}'$ each. For batch $b\in[\s{B}]$, define the random variable 
$\s{Y}^b \triangleq \mathds{1}\left[\left|\frac{1}{\s{T}'}\sum_{t \in \text{Batch $b$}} \zeta_{ij}^t\right| \ge \kappa\right]$. Clearly, $\Pr(\s{Y}^b=1)\le \frac{\sigma^2}{\kappa^2 \s{T}' }$ and further, $\s{Y}^1,\s{Y}^2,\dots,\s{Y}^{\s{B}}$ are independent random variables. Therefore, Chernoff bound implies that
\begin{align}
    \Pr\p{\sum_{b=1}^{\s{B}} \s{Y}^b \ge \s{B}/2} \le \exp\pp{-2\s{B}\p{\frac{1}{2}-\frac{\sigma^2}{\kappa^2 \s{T}' }}^2}.
\end{align}
Our final membership estimate is evaluated as follows: 
\begin{align}
    \widehat{\s{U}}_{ij} \triangleq \s{median}\p{\frac{1}{\s{T}'}\sum_{t \in \text{Batch $1$}} \s{V}_{ij}^t, \frac{1}{\s{T}'}\sum_{t \in \text{Batch $2$}} \s{V}_{ij}^t,\dots,\frac{1}{\s{T}'}\sum_{t \in \text{Batch $\s{B}$}} \s{V}_{ij}^t},
\end{align}
namely, $\widehat{\s{U}}_{ij}$ is the median of the mean of $ \widehat{\s{V}}_{ij}^t$ in each batch. Therefore, for $\s{B}=6\log n$ and $\s{T}'= 4\sigma^2/\kappa^2$ (hence $\s{T}=8\sigma^2\log n/\kappa^2$), we must have that 
\begin{align}
    \Pr\p{\left|\widehat{\s{U}}_{ij}-\s{U}_{ij}\right| \ge \epsilon_2+\kappa} \le \frac{2}{n^3}. 
\end{align}
Therefore, by taking a union bound over all $i \in [n],j \in [k]$, we can compute $\widehat{\s{U}}_{ij}$, an estimate of $\s{U}_{ij}$, such that 
\begin{align}
    \left|\widehat{\s{U}}_{ij}-\s{U}_{ij}\right| \le \epsilon_2+\kappa,\label{eqn:noisyUguran}
\end{align}
for all $i\in[n],j\in[k]$ with probability at least $1-1/n$. Finally, we estimate the means $\f{\mu}_j$'s using the already computed $\widehat{\s{U}}_{ij}$'s as follows. Note that,
\begin{align}
\widehat{\f{\mu}}_j = \frac{\sum_i \widehat{\s{U}}_{ij}^{\alpha}\f{x}_i}{\sum_i \widehat{\s{U}}_{ij}^{\alpha}} \triangleq \frac{\widehat{\f{\lambda}}_{\f{x}}}{\widehat{\s{Y}}},
\end{align}
and 
\begin{align}
\f{\mu}_{j} = \frac{\sum_{i \in [n]} \s{U}_{ij}^{\alpha} \f{x}_i}{\sum_{i \in [n]} \s{U}_{ij}^{\alpha}}  \triangleq\frac{\f{\lambda}_{\f{x}}}{\s{Y}}.
\end{align}
Therefore, we get
\begin{align}
\norm{\widehat{\f{\mu}}_{j} - \f{\mu}_{j}}_2 \le \norm{\frac{\widehat{\f{\lambda}}_{\f{x}}-\f{\lambda}_{\f{x}}}{\s{Y}}}_2+\norm{\frac{\widehat{\f{\lambda}}_{\f{x}}}{\widehat{\s{Y}}}\frac{\s{Y}-\widehat{\s{Y}}}{\s{Y}}}_2 \le \norm{\frac{\widehat{\f{\lambda}}_{\f{x}}-\f{\lambda}_{\f{x}}}{\s{Y}}}_2+\s{R}\left|\frac{\widehat{\s{Y}}-\s{Y}}{\s{Y}}\right|.\label{eqn:dd1}
\end{align}
Using \eqref{eqn:noisyUguran} it is evident that
\begin{align}
    &\left|\widehat{\s{Y}}-\s{Y}\right| \le \alpha n (\epsilon_2+\kappa)(1+o(1)),\label{eqn:dd2}\\
    &\norm{\widehat{\f{\lambda}}_{\f{x}}-\f{\lambda}_{\f{x}}}_2 \le \s{R}\alpha n (\epsilon_2+\kappa)(1+o(1)).\label{eqn:dd3}
\end{align}
Combining \eqref{eqn:dd1}--\eqref{eqn:dd3} together with \eqref{eqn:rhodef}, we finally obtain that
\begin{align}
    \norm{\widehat{\f{\mu}}_{j} - \f{\mu}_{j}}_2 \le \frac{2\s{R}\alpha  (\epsilon_2+\kappa)}{\rho},
\end{align}
for all $j \in [k]$, which concludes the proof.
\end{proof}

\section{Membership Queries From Similarity Queries}\label{app:equivalence2}

Recall that $\ca{X}\subset \bb{R}^d, \left|\ca{X}\right|=n$ is the set of points provided as input along with their corresponding $d$-dimensional vector assignments denoted by $\{\f{x}_i\}_{i=1}^{n}$. Recall that the membership-oracle $\calO_{\mathrm{fuzzy}}(i,j) = \s{U}_{ij}$ returns the membership weight of the instance $\f{x}_i$ to a cluster $j$. However, such oracle queries are often impractical in real-world settings since it requires knowledge of the relevant clusters. Instead, a popular query model that takes a few elements (two or three) as input and is easy to implement in practice is the following similarity query "How similar are these  elements?" \cite{Ashtiani16} showed that for the hard clustering setting, a membership query can be simulated by $k$ pairwise similarity queries since a pairwise similarity query reveals whether two items belong to the same cluster or not in the hard clustering setting. In the fuzzy problem we model the oracle response to the similarity query by the inner product of their membership weight vectors. More formally, we have 
\begin{definition}[Restatement of Definition \ref{def:oracle_similarity_body}]
A fuzzy pairwise similarity query asks the similarity of two distinct instances $\f{x}_i$ and $\f{x}_j$ i.e., $\calO_{\mathrm{sim}}(i,j) = \langle \s{U}_i,\s{U}_j \rangle$. A fuzzy triplet similarity query asks the similarity of three distinct instances $\f{x}_p,\f{x}_q,\f{x}_r$ i.e. $\calO_{\mathrm{triplet}}(p,q,r) = \sum_{t \in [k]}\s{U}_{pt}\s{U}_{qt}\s{U}_{rt}$.
\end{definition}

Now, we show that fuzzy pairwise similarity queries can often be used to simulate $\ca{O}_{\s{fuzzy}}(i,j)$. 
Note that if we possess the membership weight vectors of $k$ elements that are linearly independent, then, for a new element, responses to fuzzy pairwise similarity queries with the aforementioned $k$ elements reveals all the membership weights of the new element. Now, the question becomes ``How can we obtain the membership weights of the $k$ elements in the first place?". Suppose we sub-sample a set of elements $\ca{Y}\subseteq \ca{X}$ such that $\left|\ca{Y}\right|=m>k$ and we make all fuzzy pairwise similarity queries among the elements present in $\ca{Y}$. Let us denote by $\s{V}$ the membership weight matrix $\s{U}$ constrained to the rows corresponding to the elements in $\ca{Y}$. Clearly, the fuzzy pairwise similarity queries between all pairs of elements in $\ca{Y}$ reveals $\s{V}\s{V}^{T}$, the gram matrix of $\s{V}$. If we can recover $\s{V}$ uniquely from $\s{V}\s{V}^T$, and $\s{V}$ is full rank, then we are done. If we assume almost any continuous distribution according to which the membership weight vectors are generated, then with probability $1$, the matrix $\s{V}$ is full rank. On the other hand, the question of uniquely recovering $\s{V}$ from $\s{V}\s{V}^T$ is trickier. In general it is not possible to recover $\s{V}$ uniquely from $\s{V}\s{V}^T$ since $\s{V}\s{R}$, for any orthonormal matrix $\s{R}$, also has the gram matrix $\s{V}\s{V}^{T}$. However, recall that in our case, the entries of $\s{V}$ are non-negative and furthermore, the rows of $\s{V}$ add up-to 1 leading to additional constraints. This leads to the problem 

\begin{align*}
    \text{ Find } \s{M} \text{ such that } \s{M}\s{M}^{T} = \s{V}\s{V}^T \text{ subject to } \s{M} \in \bb{R}_{\ge 0}^{m \times k}, \sum_{j\in [k]} \s{M}_{ij}=1 \; \forall \; i\in [m].
\end{align*}

As a matter of fact, this is a relatively well-studied problem known as the Symmetric Non-Negative matrix factorization or SNMF. We will say that the solution to the SNMF problem is unique if $\s{VP}$ is the only solution to the problem for any permutation matrix $\s{P}$.
Below, we state the following sufficient condition  that guarantees the uniqueness of the solution to the SNMF problem.

\begin{lemma}[Lemma 4 in \cite{huang2013non}]\label{lem:terse}
If $\s{rank}(\s{V})=k$, then the solution to the SNMF problem is unique if and only if the non- negative orthant is the only self-dual simplicial cone $\ca{A}$ with $k$ extreme rays that satisfies $\s{cone}(\s{V}^{T})\subseteq \ca{A}=\ca{A}^{\star}$ where $\ca{A}^{\star}$ is the dual cone of $\ca{A}$, defined as $\ca{A}^{\star}=\{\f{y}\mid \f{x}^{T}\f{y}\ge 0 \; \forall \f{x} \in \ca{A}\}$.
\end{lemma}

More recently, Lemma \ref{lem:terse} was used in \cite{mao2017mixed} to show the following result that is directly applicable to our setting:

\begin{lemma}\label{lem:pairw}
If $\s{V}$ contains any permutation matrix of dimensions $k \times k$, then the solution of the SNMF problem is unique.
\end{lemma}

Suppose we have the guarantee that for each cluster $j\in [k]$, there exists a set $\ca{Z}_j$ of at least $\rho n$ elements belonging purely to the $j^{\s{th}}$ cluster i.e. $\s{U}_{ij}=1$ for all $i \in \ca{Z}_j$. Then, for $m\ge \rho^{-1}\log (nk)$, the matrix $\s{V}$ will contain a permutation matrix with probability at least $1-n^{-1}$. As a results, this will lead to an overhead of $O(m^2)=O(\rho^{-2}\log^2 (nk))$ queries. 

If it is possible to make more complex similarity queries such as the fuzzy triplet similarity query, 
we can significantly generalize and improve the previous guarantees. Before proceeding further, let us provide some  background on tensors beginning with the following definition:





\begin{definition}[Kruskal rank]
The Kruskal rank of a matrix $\f A$ is defined as the maximum number $r$ such that any $r$ columns of $\f A$ are linearly independent. 
\end{definition}

Consider a tensor $\ca{A}$ of order $w \in \bb{N}$ for $w>2$ on $\bb{R}^n$, denoted by $\ca{A} \in \bb{R}^n \otimes \bb{R}^n \otimes \dots \otimes \bb{R}^n \; (w \; \text{times})$.
Let $\ca{A}_{i_1,i_2,\dots,i_{w}}$ where $i_1,i_2,\dots,i_{w} \in \{0,1, \dots, n-1\}$, denote the element in $\ca{A}$ whose location along the $j^{\s{th}}$ dimension is $i_j+1$, i.e., there are $i_j$ elements along the $j^{\s{th}}$ dimension before $\ca{A}_{i_1,i_2,...,i_w}$. Notice that this indexing protocol uniquely determines the element within the tensor. For a detailed review of tensors, we defer the reader to \cite{kolda2009tensor}. In this work, we are interested in low-rank decomposition of tensors. 
A tensor $\ca{A}$ can be described as a rank-\texttt{1} tensor if it can be expressed as\footnote{In this work, we focus on the special case of rank-\texttt{1} tensors where every component is identical. In general, rank-\texttt{1} tensors can be described as $\ca{A} = \f{z}^1\otimes \f{z}^2 \otimes \dots \otimes \f{z}^w$ for $\f z^1, \ldots \f z^w \in \bb{R}^n$.} 
\begin{align*}
    \ca{A} =\underbrace{\f{z}\otimes \f{z} \otimes \dots \otimes \f{z}}_{w \text{\;\; times}}
\end{align*}
for some $\f{z}\in \bb{R}^n$, i.e., $\ca{A}_{i_1,i_2,\dots,i_{w}} = \prod_{j=1}^{w}\f{z}_{i_j}$. For a given tensor $\ca{A}$, we are concerned with the problem of uniquely decomposing $\ca{A}$ into a sum of $R$ rank-\texttt{1} tensors. 
A tensor $\ca{A}$ that can be expressed in this form is denoted as a rank-$R$ tensor, and such a decomposition is also known as the Canonical Polyadic (CP) decomposition. Below, we state a result due to \cite{sidiropoulos2000uniqueness} describing the sufficient conditions for the unique CP decomposition of a rank-$R$ tensor $\ca{A}$.
\begin{lemma}[Unique CP decomposition \cite{sidiropoulos2000uniqueness}]\label{lem:unique_cp}
Suppose $\ca{A}$ is the sum of $R$ rank-\texttt{1} tensors, i.e., 
\begin{align*}
    \ca{A}=\sum_{r=1}^{R}\underbrace{\f{z}^r\otimes \f{z}^r \otimes \dots \otimes \f{z}^r}_{w \text{\;\; times}},
\end{align*}
and further, the Kruskal rank of the $n \times R$ matrix whose columns are formed by $\f{z}^1,\f{z}^2,\dots,\f{z}^R$ is $J$. Then, if 
$
    wJ \ge 2R+(w-1),
$
then the CP decomposition is unique and we can recover the vectors $\f{z}^1,\f{z}^2,\dots,\f{z}^R$ up to permutations.
\end{lemma}

\begin{algorithm}[htbp!]
\caption{\textsc{\textsc{Jennrich's Algorithm}}$(\ca{A})$\label{algo:tensor}}
\begin{algorithmic}[1]
\REQUIRE A symmetric rank-$ R$ tensor $\ca{A} \in \bb{R}^n \otimes \bb{R}^n  \otimes \bb{R}^n$ of order $3$. 

\STATE Choose $\f{a},\f{b}\in \bb{R}^n$ uniformly at random such that it satisfies $\norm{\f{a}}_2=\norm{\f{b}}_2=1$.

\STATE Compute  $\f{T}^{(1)} \triangleq \sum_{i \in [n]}\f{a}_i\ca{A}_{\cdot,\cdot,i},\f{T}^{(2)} \triangleq \sum_{i \in [n]}\f{b}_i\ca{A}_{\cdot,\cdot,i}$.

\IF{$\s{rank}(T^{1})<R$}
\STATE Return Error
\ENDIF

\STATE Solve the general eigenvalue problem $\f{T}^{(1)}\f{v}=\lambda_v \f{T}^{(2)}\f{v}$.

\STATE Return the  eigen-vectors $\f{v}$ corresponding to the non-zero eigen-values.

\end{algorithmic}
\end{algorithm}
Notice that for the special case of $w=3$, the underlying vectors $\f{z}^1,\f{z}^2,\dots,\f{z}^R$ can be recovered uniquely if they are linearly independent.
Now, we are ready to show that $k$ fuzzy triplet similarity queries can be used to recover the memberships weights of $k$  elements uniquely. As before, we can sub-sample a set of  elements $\ca{Y}\subseteq \ca{X}$ such that $\left|\ca{Y}\right|=k$ and we make all possible ${k \choose 3}$ fuzzy triplet similarity queries among the elements present in $\ca{Y}$. Again, let us denote by $\s{V}$ the membership weight matrix $\s{U}$ constrained to the rows corresponding to the elements in $\ca{Y}$. Let us denote by $\f{v}^1,\f{v}^2,\dots,\f{v}^k$ the $k$ columns of the matrix $\s{V}$. Notice that the responses to all the fuzzy triplet similarity queries reveals the following symmetric tensor 
\begin{align*}
    \sum_{r=1}^{k}\f{v}^r\otimes \f{v}^r  \otimes \f{v}^r.
\end{align*}
Suppose the matrix $\ca{V}$ is full rank. This will happen with probability $1$ if the membership weights are assumed to be generated according to any continuous distributions. Algorithmically, Jennrich's algorithm (see Section 3.3, \cite{moitra2014algorithmic}) can be used to efficiently recover the unique CP decomposition of a third order low rank tensor whose underlying vectors are full rank. We have provided the algorithm (see Algorithm \ref{algo:tensor}) for the sake of completeness.

\end{document}